\documentclass[english]{article}

\usepackage{geometry}
\usepackage[T1]{fontenc}
\usepackage[utf8]{inputenc}
\usepackage{bm}
\usepackage{microtype}
\usepackage{amsmath,mathtools}
\usepackage{amssymb,mathrsfs}
\usepackage[unicode=true,
 bookmarks=false,
 breaklinks=false,pdfborder={0 0 1},colorlinks=false]
 {hyperref}
\hypersetup{colorlinks,citecolor=blue,filecolor=blue,linkcolor=blue,urlcolor=blue}
\usepackage{tikz}
  \usetikzlibrary{positioning,fit,backgrounds,matrix,calc}
\usepackage{xcolor}
\usepackage{amsthm}
\usepackage{comment}
\usepackage{authblk}
\usepackage[sort,compress]{natbib}
\usepackage{booktabs}
\usepackage{mathabx}
\usepackage{graphicx}
\usepackage{array}

\usepackage{bbding}
\usepackage{colortbl}
\usepackage{makecell}
\usepackage{algorithm}
\usepackage{algorithmic}
\usepackage{float}
\usepackage{multirow}
\usepackage{dsfont}
\usepackage{tcolorbox}
\usepackage{color}
\definecolor{yxc}{RGB}{255,0,0}
\definecolor{ytw}{RGB}{255,69,0}
\definecolor{gen}{RGB}{0,0,200}
\definecolor{zhh}{RGB}{200,200,0}
\allowdisplaybreaks
\usepackage{enumitem}
\usepackage[nosort,capitalise]{cleveref}
\crefname{enumi}{}{}
\crefname{equation}{Eqn.}{Eqns.}
\crefname{thm}{Theorem}{Theorems}
\crefname{lem}{Lemma}{Lemmas}
\crefname{cor}{Corollary}{Corollaries}
\crefname{proo}{Proposition}{Propositions}
\crefname{deff}{Definition}{Definitions}
\crefname{rem}{Remark}{Remarks}
\crefname{assumption}{Assumption}{Assumptions}
\crefname{exa}{Example}{Examples}
\crefname{property}{Property}{Properties}
\Crefname{thm}{Theorem}{Theorems}
\Crefname{lem}{Lemma}{Lemmas}
\Crefname{cor}{Corollary}{Corollaries}
\Crefname{proo}{Proposition}{Propositions}
\Crefname{deff}{Definition}{Definitions}
\Crefname{rem}{Remark}{Remarks}
\Crefname{assumption}{Assumption}{Assumptions}
\Crefname{exa}{Example}{Examples}
\Crefname{property}{Property}{Properties}

\newtheorem{assumption}{Assumption}
\newtheorem{thm}{Theorem}

\newtheorem{proo}{Proposition}
\newtheorem{lem}{Lemma}
\newtheorem{cor}{Corollary}

\AtBeginEnvironment{proof}{\parindent=0pt\relax}
\newcommand{\x}{\mathbf{x}}

\newcommand{\polylog}{\mathrm{polylog}}
\newcommand{\y}{\mathbf{y}}

\newcommand{\Expect}{\mathbb{E}}

\title{Why Adam Can Beat SGD: Second-Moment Normalization Yields Sharper Tails}

\author{Ruinan Jin\thanks{Department of Electrical and Computer Engineering, The Ohio State University, USA; emails: \texttt{jrnjrnjrnjrnjrn@126.com}, \texttt{liang.889@osu.edu}
},  ~~~Yingbin Liang$^*$, ~~~Shaofeng Zou\thanks{School of Electrical, Computer and Energy Engineering, Arizona State University, USA; email: \texttt{zou@asu.edu}
}
}

\date{\today}

\begin{document}

\maketitle

\begin{abstract}%
Despite Adam’s faster empirical convergence than stochastic gradient descent (SGD) in practice, existing theory often yields convergence guarantees for Adam that are comparable to those for SGD, leaving Adam’s advantage largely unexplained theoretically. In this paper, we establish the first theoretical separation between the high-probability convergence behaviors of Adam and SGD: Adam achieves a $\delta^{-1/2}$ dependence on the confidence parameter $\delta$, whereas the corresponding high-probability guarantee for SGD necessarily incurs at least a $\delta^{-1}$ dependence. Our analysis develops a stopping-time and martingale-based framework to provably characterize this advantage and identifies Adam’s second-moment normalization as the key mechanism underlying its performance gain.
\end{abstract}

\setcounter{tocdepth}{2}
\tableofcontents

\section{Introduction}
Adaptive gradient methods such as Adam (Adaptive Moment Estimation) \citep{kingma2014adam} are widely observed to converge faster and behave more robustly than vanilla stochastic gradient descent (SGD) across a broad range of machine learning tasks \citep{zhou2020towards,pan2023toward,li2025choice}. Yet, from the theoretical perspective, this empirical advantage remains insufficiently explained. 

From a high-probability perspective, existing convergence results can be viewed through the lens of guarantees that hold with probability at least $1-\delta$. These studies mainly fall into two regimes, distinguished by the assumptions imposed on the stochastic gradients. (i) Under standard second-moment conditions (e.g., bounded variance), recent studies on Adam \citep{wang2024closing,jin2024comprehensive} obtained bounds on the gradient norm with a confidence dependence of $\mathcal{O}(\delta^{-2})$, and \citep{zou2019sufficient} obtained bounds on the gradient norm with a confidence dependence of $\mathcal{O}(\delta^{-3/2})$. All these bounds are even worse than the $\mathcal{O}(\delta^{-1})$ dependence achieved by SGD under comparable second-moment assumptions \citep{ghadimi2013stochastic,nguyen2018sgd}. (ii) Under stronger tail assumptions on stochastic gradients (e.g., sub-Gaussian or sub-exponential concentration), prior studies on Adam \citep{DBLP:conf/nips/LiRJ23,DBLP:journals/corr/abs-2402-03982,xiao2024adam,DBLP:journals/tmlr/DefossezBBU22} achieved only polylogarithmic confidence dependence of $\polylog(\delta^{-1})$. However, SGD typically enjoys comparable $\delta$-dependence under the same tail conditions \citep{madden2024high}. As a result, these guarantees do not establish a clear performance separation between Adam and SGD.

%Clearly, existing studies do not provide provable guarantees showing that Adam attains a better convergence rate than SGD. This gap between empirical observations and available theory motivates us to address the following foundational open question:
Given that existing studies do not establish provable performance advantage of Adam over SGD, we are motivated to address the following foundational open question:
\begin{quote}
{\em What key factor underlies Adam’s performance advantage over SGD, and can we develop an analytical framework that rigorously captures this advantage?}
\end{quote}

%We provide an affirmative answer to the above question in this paper.
%Such a problem formulation captures the traditional finite-sum type of objective functions over samples and expected loss function over the distribution of the samples

\begin{table}[t]
\centering
\caption{\em Convergence bound for Adam with a $1-\delta$ high-probability guarantee under second-moment stochastic gradient assumptions.}
\begin{tabular}{p{0.16\linewidth}|p{0.48\linewidth}|p{0.22\linewidth}}
%\hline
Paper & Assumption & Results \\
\hline
\citet{wang2024closing}$^{\dagger}$ 
&
\(\Expect\!\left[\|g_t-\nabla f(\x_t)\|^2 \mid \mathscr{F}_{t-1}\right]\le C\)
&
\(\mathcal{O}\!\left(\delta^{-2}T^{-1/2}\right)\)
\\
\hline
\citet{jin2024comprehensive}  
&
\(\Expect\!\left[\|g_t-\nabla f(\x_t)\|^2 \mid \mathscr{F}_{t-1}\right]
\le A\bigl(f(\x_t)-f^*\bigr)+ B\|\nabla f(\x_t)\|^2 + C\)
&
\(\mathcal{O}\!\left(\delta^{-2}T^{-1/2}\log T\right)\)
\\
\hline
\citet{zou2019sufficient}$^{\dagger}$ 
&
\(\Expect\!\left[\|g_t\|^2 \mid \mathscr{F}_{t-1}\right]\le C\)
&
\(\mathcal{O}\!\left(\delta^{-3/2}T^{-1/2}\right)\)
\\
\hline
\rowcolor{gray!20} \textcolor{blue}{This paper} & \(\textcolor{blue}{\Expect\!\left[\|g_t-\nabla f(\x_t)\|^2 \mid \mathscr{F}_{t-1}\right]\le C}\)
&
\(\textcolor{blue}{\mathcal{O}\!\left(\delta^{-1/2}T^{-1/2}\right)}\)
\end{tabular}

\parbox{\linewidth}{\textit{\bf Comparison:} Our result achieves the best dependence on $\delta$. This is the first result that established the advantage of Adam over SGD, which attains no better bound than $\mathcal{O}\!\left(\frac{1}{\delta}\frac{1}{\sqrt{T}}\right)$ as shown in Theorem \ref{prop:sgd_lower_informal}.

\(\dagger\) In~\citet{wang2024closing} and \citet{zou2019sufficient}, the results were originally derived 
%for the \emph{sub-second-moment bound} of the gradient norm, i.e., 
as the bounds on $\mathbb{E}\left[\|\nabla f(\x_t)\|^{2-\epsilon}\right]$ where $\epsilon>0.$ 
To enable a comparison at the same level of high-probability guarantees, we convert those expectation bounds to high-probability bounds via Markov's inequality.
}
\label{tab:adam_second_moment}
\end{table}

\subsection{Main Contributions} 

We study the stochastic optimization problem
$
\min_{\x\in\mathbb{R}^d} f(\x),
$
where the algorithm has access only to stochastic gradient estimates and must operate under sampling noise.
We run Adam for $T$ iterations. 

(1) {\bf Sharper Bound for Adam.} Under the common classical conditions ($L$-smoothness and bounded variance), we show that with probability at least $1-\delta$, Adam satisfies
\begin{equation*}
\frac{1}{T}\sum_{t=1}^{T}\|\nabla f(\x_t)\|^2
=\widetilde{O}\!\left(\frac{1}{\sqrt{\delta\,T}}\right)
\qquad\text{(Adam)}.
\end{equation*}
The above bound improves the existing bound of $\mathcal{O}(\delta^{-2})$ in \citep{wang2024closing,jin2024comprehensive} and bound of $\delta^{-3/2}$ in \citep{zou2019sufficient} under the same second-moment-type assumptions for Adam.

(2) \textbf{Provable Advantage of Adam.} To calibrate the advantage of Adam, we also present a
high-probability lower bound for SGD under the same assumptions of $L$-smoothness and bounded variance. Namely, for any step size $\gamma>0$ and any horizon $T\ge 1$, there exists a hard example such that, with probability at least $\delta$,
the SGD iterates $\{x_t\}_{t=1}^T$ satisfy
\[
\frac{1}{T}\sum_{t=1}^{T}\|\nabla f(x_t)\|^2 \;\ge\; \widetilde{\Omega}\!\left(\frac{1}{\delta\sqrt{T}}\right). \qquad \text{(SGD)}
\]
%Consequently, under the bounded-variance assumption, any high-probability guarantee for SGD on this stationarity measure must necessarily exhibit (up to logarithmic factors) at least a $\delta^{-1}$ dependence, i.e., it cannot improve upon $\widetilde{\Omega}\!\left(\frac{1}{\delta\sqrt{T}}\right)$ in the worst case.

%reflecting that second-moment normalization suppresses the accumulation of trajectory noise under mere bounded variance condition, whereas SGD without such normalization cannot match this $\delta$-dependence without stronger tail assumptions. Empirically, when $\delta$ is not fixed and we inspect the distribution over multiple independent runs, this predicts a consistent stochastic dominance: Adam’s performance curves concentrate more on the favorable side than SGD’s (e.g., higher empirical CDF / lower high-quantile curve over a broad range of confidence levels).

Comparison between Adam’s upper bound and SGD's lower bound shows that Adam improves the confidence dependence by a factor of $\delta^{-1/2}$ (up to $\polylog(\delta^{-1})$ terms).
% From a distributional perspective, this comparison implies that across multiple independent runs, Adam’s performance, measured by the average gradient norm, concentrates more tightly around smaller values than SGD’s, indicating distributionally faster convergence. 
To the best of our knowledge, this is the first work that provably establishes a performance advantage of Adam over SGD, thereby providing a theoretical explanation for Adam’s empirical acceleration.

(3) {\bf Novel Analysis: Second-Moment Normalization Yields Sharper Tails.} Our sharper analysis identifies second-moment normalization as the key mechanism behind Adam’s high-probability advantage over SGD. Adam’s coordinate-wise normalization turns cumulative stochastic-gradient energy into a self-normalized logarithmic quantity, yielding lighter tails and sharper confidence dependence. In contrast, SGD accumulates raw stochastic-gradient energy through a fixed step size, inheriting heavier tails under the same second-moment assumption.

Our key difference from standard Adam analyses is that we exploit the logarithmic control of the adaptive stochastic-gradient energy induced by Adam’s second-moment normalization. This enables all-order moment bounds that control the two main error terms in the descent inequality: the stochastic-gradient martingale difference and the second-order descent residual. These bounds lead to a stretched-exponential tail bound for the adaptive gradient energy and ultimately yield a sharper high-probability convergence guarantee for Adam. In contrast, standard analyses typically average out the martingale terms and rely on first-moment control, thereby not capturing the lighter-tail behavior created by Adam’s normalization.

\subsection{Related Work}
\textbf{Convergence Guarantee of Adam.} The convergence guarantees have been extensively studied recently for Adam under various assumptions on stochastic gradients. Under standard second-moment type conditions (e.g., bounded variance), recent works \citep{wang2024closing,jin2024comprehensive} obtained bounds on the gradient norm with a confidence dependence of $\mathcal{O}(\delta^{-2})$, and \citep{zou2019sufficient} obtained bounds on the gradient norm with a confidence dependence of $\mathcal{O}(\delta^{-3/2})$.
%All these bounds are even worse than the $\mathcal{O}(\delta^{-1})$ dependence achieved by SGD under comparable second-moment assumptions \citep{ghadimi2013stochastic,nguyen2018sgd}. 
Clearly, these second-moment-based high-probability bounds attain at best a $\delta^{-(1+\epsilon)}\ (\epsilon>0)$-type confidence dependence, which can be even \emph{worse} than the bound of $\mathcal{O}(\delta^{-1})$ achieved by SGD under comparable second moment conditions \citep{ghadimi2013stochastic,nguyen2018sgd}.The work \citep{guo2021novel} also established convergence guarantees for Adam with $\mathcal O(\delta^{-1})$ dependence under second-moment assumptions; however, these results additionally rely on auxiliary conditions, such as $\|v_t\|_{\infty}\sim \mathcal{O}(1)$ (where $v_t$ is defined in Algorithm \ref{alg:adam}), which effectively impose stronger requirements, including the existence of moments of an arbitrarily high order. Further, under stronger tail assumptions on stochastic gradients (e.g., sub-Gaussian or sub-exponential concentration), prior studies \citep{DBLP:conf/nips/LiRJ23,DBLP:journals/corr/abs-2402-03982,xiao2024adam,DBLP:journals/tmlr/DefossezBBU22} showed that Adam achieves polylogarithmic confidence dependence of $\polylog(\delta^{-1})$. However, SGD also enjoys comparable $\delta$-dependence under the same tail conditions \citep{madden2024high}. As a result, these guarantees do not establish a clear performance separation between Adam and SGD.

\textbf{Separation Between Adam-type Methods and SGD.} Several works aimed to characterize advantages of Adam-type adaptive methods over SGD. One line of results showed that adaptive methods can converge under generalized smoothness conditions
(e.g., $L_0$--$L_1$ smoothness), whereas vanilla SGD may fail or even diverge \citep{wang2024provable}.
In contrast, this paper studies a regime where both methods converge and prove a sharp separation in their high-probability convergence rates,
mirroring the common observation that Adam can be faster without relying on divergence examples.
%It is also worth noting that subsequent work \citep{li2023convex} showed that SGD can, in fact, converge under generalized smoothness as well. Under the second moment (bounded variance) condition, their high-probability guarantees still exhibit at best $\mathcal{O}(\delta^{-1})$ dependence on the confidence parameter $\delta$, which is matched by our $\Omega(\delta^{-1})$ lower bound for SGD under shared classical conditions ($L$-smoothness and bounded variance).

A complementary viewpoint concerned diagonal preconditioning and dimension dependence. \citet{jiang2025provable} analyzed AdaGrad and obtained improved $d$-dependence under coordinate-wise smoothness, and studied the first moment of the gradient norm. In contrast, we analyze the original Adam algorithm under standard smoothness assumptions, focusing on the canonical criterion $\frac{1}{T}\sum_{t=1}^T \|\nabla f(\x_t)\|^2$, and establishing the first high-probability separation between Adam and SGD.

%\textbf{Convergence for Adam Variants.} There has also been extensive work on Adam-type algorithms, such as Adam-norm \citep{wang2024convergence}, AdamW \citep{zhou2024towards,li2025d}, AdaGrad-norm \citep{wang2023convergence,faw2023beyond}, Adam+\citep{liu2020}, and several other variants of Adam (e.g.,~\citep{guo2021novel,huang2021super,zhou2024on}). Related studies have further investigated their generalization properties and comparisons with gradient descent (e.g.,~\cite{wilson2017marginal,zhang2020why,zhou2020towards,zou2023understanding}), as well as their connections to SignGD and related methods (e.g.,~\cite{balles2018dissecting,bernstein2019signsgd,kunstner2023noise,xie2024implicit}). 

%While most deep learning algorithms are rooted from stochastic gradient descent (SGD) type of first-order method, practical training often uses advanced and adaptive versions of SGD such as Adam \cite{kingma2014adam} or AdamW \cite{loshchilov2018decoupled} to accelerate the performance. 

\textbf{Lower Bounds for SGD and Stochastic First-Order Methods.} High-probability lower bounds were developed for convex and strongly convex nonsmooth objectives under almost-surely bounded stochastic subgradients by \citet{harvey2019tight,harvey2019simple} with \(\mathcal O(\log(\delta^{-1}))\) confidence dependence that matches the upper bound. In contrast, we develop lower bounds for SGD under a different assumption of bounded variance. Another related line of work by \citet{arjevani2023lower,drori2020complexity,jiang2025provable} studied lower bounds for SGD in terms of the optimization horizon \(T\) or the target stationarity accuracy \(\epsilon\). But these works did not provide high-probability lower bounds with the dependence on the confidence parameter $\delta$, which is the focus of our work.

More related work on {\bf convergence for Adam variants} and {\bf high-probability upper bound for SGD} is provided in Appendix \ref{sec:extendedrelatedwork}.

\section{Preliminaries}
Consider the optimization problem
\begin{align}\label{eq:objective}
    \min_{\x\in\mathbb{R}^d} f(\x),
\end{align}
where $f:\mathbb{R}^d\to\mathbb{R}$ is differentiable. We assume that the function $f$ admits a finite lower bound and is $L$-smooth, as given in Assumptions \ref{ass:nonneg} and \ref{ass:smooth}, respectively. 

To solve the problem, we assume that only gradient estimate of $\nabla f(\x)$ from a {\em stochastic} oracle is available. We study the convergence of first-order iterative stochastic methods (e.g., SGD and Adam) that generate a sequence
of iterates $\{\x_t\}_{t\ge1}\subset\mathbb R^d$.
At each iteration $t\ge1$, the method obtains a stochastic gradient estimate of $\nabla f(\x_t)$
from a stochastic oracle driven by fresh randomness, and then generate an update $\x_{t+1}$ accordingly.

Throughout the paper, we work on a common probability space $(\Omega,\mathscr F,\mathbb P)$.
Let $(\mathcal Z,\mathscr Z)$ be a measurable space, and let
$\{z_t\}_{t\ge1}$ be an \emph{independent} sequence of $\mathcal Z$-valued random variables,
i.e., each $z_t:(\Omega,\mathscr F)\to(\mathcal Z,\mathscr Z)$ is measurable.
Define the canonical filtration $\mathscr F_t:=\sigma(z_1,\ldots,z_t)$.
The iterates $\{\x_t\}_{t\ge1}$ are assumed to be predictable, namely
$\x_t$ is $\mathscr F_{t-1}$-measurable for every $t\ge1$. The expectation is denoted by $\mathbb E[\cdot]$.

At iteration $t\ge1$, after observing $z_t$, the algorithm queries a stochastic first-order oracle
\(
g_t := g(\x_t;z_t):\mathbb R^d\times\mathcal Z\to\mathbb R^d,
\)
and uses $g_t$ as an estimator of the true gradient $\nabla f(\x_t).$ We assume that such a gradient estimator satisfies the bounded variance condition given in Assumption \ref{ass:abc}. We often refer to $g_t$ as the {\em stochastic gradient}.

\subsection{Adaptive Moment Estimation (Adam)}
Adam \citep{DBLP:journals/corr/KingmaB14,DBLP:conf/iclr/ReddiKK18,wang2024closing,jin2024comprehensive}
is a widely used adaptive stochastic first-order method.
As described in Algorithm~\ref{alg:adam}, it obtains a stochastic gradient at each step (Line 3), uses $m_t$ and $v_t$ to maintain exponentially weighted moving averages for tracking the first and second moments of the stochastic gradients (Lines 4 and 5), respectively, updates an adaptive step size $\gamma_t$ normalized by the second moment (Line 6), and then leverages the first moment and adaptive stepsize to form a coordinatewise update of the variable (Line 7). 

\begin{algorithm}
\caption{Adam}\label{alg:adam}
\begin{algorithmic}[1]
\STATE \textbf{Input:} Stochastic oracle $\mathcal Z;$ stepsize $\gamma>0$; initial point $\x_1\in\mathbb{R}^d$;
  $m_0=\mathbf{0}$; $v_0=v\,\mathbf{1}$ with $v>0$;
  $\beta_1\in[0,1)$; $\beta_{2}\in[0,1)$; $\epsilon>0$; number of iterations $T$.
\FOR{$t=1$ \TO $T$}
  \STATE Sample $z_t$ and compute $g_t \gets g(\x_t; z_t)$.
  \STATE Update the first-moment estimate:
  $m_t \gets \beta_1 m_{t-1} + (1-\beta_1) g_t$.
  \STATE Update the second-moment estimate:
  $v_t \gets \beta_{2} v_{t-1} + (1-\beta_{2}) (g_t \odot g_t)$.
  \STATE Form the coordinatewise preconditioning vector:
  $\gamma_t \gets \gamma \cdot (\sqrt{v_t} + \epsilon)^{-1}$.
  \hfill (element-wise)
  \STATE Update the iterate:
  $\x_{t+1} \gets \x_t - \gamma_t \odot m_t$.
\ENDFOR
\STATE \textbf{Output:} final iterate $\x_{T+1}$.
\end{algorithmic}
\end{algorithm}
\noindent All arithmetic operations in Algorithm~\ref{alg:adam} are executed componentwise, the symbol \(\odot\) denotes the Hadamard product, and
\((\sqrt{v_t}+\epsilon)^{-1}\) is taken elementwisely.
For notational convenience, we also define the initial coordinatewise preconditioning vector by
\(\gamma_0:=\gamma(\sqrt v+\epsilon)^{-1}\mathbf 1\), where \(v>0\) is the scalar initialization in \(v_0=v\mathbf 1\).
For a vector \(\alpha\in\mathbb{R}^d\), we write \(\alpha_i\) for its \(i\)-th coordinate; for a time-indexed vector \(\alpha_t\in\mathbb{R}^d\), we write \(\alpha_{t,i}\) for its \(i\)-th coordinate.

\subsection{Technical Assumptions}
We adopt the common conditions on the objective $f$ and the stochastic oracle in our analysis. 
Assumptions~\ref{ass:nonneg}--\ref{ass:smooth} are standard regularity conditions ensuring a well-behaved descent geometry, 
while Assumption~\ref{ass:abc} specifies a second moment (i.e., bounded variance) condition.

\begin{assumption}[Finite lower bound]
\label{ass:nonneg}
The objective satisfies $$f(\x)> f^*>-\infty,\quad\forall  \x\in\mathbb{R}^d.$$
\end{assumption}
This condition is necessary for the minimization problem in \cref{eq:objective} to be well-posed: without a finite lower bound, the objective can decrease without limit, and therefore there is no meaningful notion of (approximate) optimality or convergence.

Next, we impose the smoothness condition, which allows us to relate the one-step progress of the iterates to the gradient norm via the standard descent lemma.
\begin{assumption}[$L$-smoothness]
\label{ass:smooth}
There exists some constant $L>0$ such that
\[
\|\nabla f(\x)-\nabla f(\y)\|\le L\|\x-\y\|,\quad \forall x,y\in\mathbb{R}^d.
\]
\end{assumption}
This $L$-smoothness condition is standard in stochastic optimization \citep{nesterov2013introductory,doi:10.1137/120880811,bottou2018optimization}: it is commonly adopted in the analysis of both SGD and adaptive methods and holds in a broad range of applications.

\begin{assumption}[Bounded variance]\label{ass:abc}
For each $t\ge 1$, the stochastic gradient $g_t$ is conditionally unbiased, i.e.,
\(
\mathbb{E}\!\left[g_t \,\middle|\, \mathscr{F}_{t-1}\right]=\nabla f(\x_t),
\)
and there exists a constant $C\ge 0$ such that
\[
\mathbb{E}\!\left[\left\|g_t-\nabla f(\x_t)\right\|^2 \,\middle|\, \mathscr{F}_{t-1}\right]\le C .
\]
\end{assumption}
Assumption~\ref{ass:abc} is a classical noise model in stochastic optimization and stochastic approximation; see, e.g.,
\citet{doi:10.1137/070704277,doi:10.1137/120880811,bottou2018optimization}. 
We adopt Assumption~\ref{ass:abc} as a representative ``second moment only'' condition: unlike assumptions that directly bound the (conditional) second moment of $g_t,$ which typically impose implicit boundedness requirements on the objective/iterates through $\mathbb{E}[\|g_t\|^2|\mathscr{F}_{t-1}],$ it only controls the noise variance around $\nabla f(\x_t)$. Our analysis can be extended to more general affine-variance models (e.g., ABC-type conditions \citep{jin2024comprehensive}) with additional technical work. This choice keeps the focus on the core mechanism behind our Adam--SGD separation.

{\bf Asymptotic notation.}
Let $a(T)$ and $b(T)$ be nonnegative functions.
We write $a=\mathcal{O}(b)$ if there exists a constant $c>0$ such that $a(T)\le c\,b(T)$ for all sufficiently large $T$, and we use $a=\widetilde{\mathcal{O}}(b)$ to hide logarithmic terms.
We use $a=\Omega(b)$ if there exists a constant $c>0$ such that $a(T)\ge c\,b(T)$ for all sufficiently large $T$, and we use $a=\widetilde{\Omega}(b)$ to hide logarithmic terms.
% We write $a=o(b)$ if $\lim_{T\to\infty} a(T)/b(T)=0$, and $a=\omega(b)$ if $\lim_{T\to\infty} a(T)/b(T)=\infty$.

\section{Main Results}
%In this section, we provide our main results and their implications. 
In this section, we first highlight how the second-moment normalization mechanism benefits Adam, and then present our characterization of the convergence separation between Adam and SGD.
\subsection{How Second-Moment Normalization Benefits Adam}
\label{subsec:self_normalization}

%\noindent\textit{Expository reduction (not a restriction).}
For illustration, we explain the role of second-moment normalization through an RMSProp-style update \citep{hinton2012lecture6e}, where the normalization effect is more explicit.\footnote{This is only a conceptual simplification used to isolate the effect of second-moment normalization; our main results and analysis are for the original Adam algorithm in Algorithm~\ref{alg:adam}.}
Namely, consider
%\[\x_{t+1}=\x_t-\frac{\gamma}{\sqrt{v_t}+\epsilon}\odot g_t,\qquad
%v_t=\beta_2 v_{t-1}+(1-\beta_2)(g_t\odot g_t).\]
\begin{align}\label{eq:qv_rmsprop}
\x_{t+1}=\x_t-\gamma_t\odot g_t,\qquad
\gamma_t:=\gamma(\sqrt{v_t}+\epsilon)^{-1},
\qquad
v_t=\beta_2v_{t-1}+(1-\beta_2)(g_t\odot g_t),
\end{align}
where all operations are componentwise.

%Stochastic-gradient iterates typically admit an (implicit) semimartingale structure, and a key quantity that governs their trajectory behavior is the \emph{quadratic variation} of the iterate increments,
The \emph{adaptive stochastic gradient energy} $\sum_{t=1}^{T} \sum_{i=1}^{d} \gamma_{t,i}^2g_{t,i}^2$ plays a central role in the convergence analysis because it is closely tied to the \emph{quadratic variation} of the iterate increments:
\begin{align}\label{eq:qv_def}
[\x]_T
:=
\sum_{t=1}^{T}
\|\x_{t+1}-\x_t\|^2
=
\sum_{t=1}^{T}
\|\gamma_t\odot g_t\|^2
=
\sum_{t=1}^{T}
\sum_{i=1}^{d}
\gamma_{t,i}^2g_{t,i}^2 .
%\label{eq:rmsprop_qv_adaptive_stoch_energy}
\end{align}
\begin{comment}
\begin{equation}\label{eq:qv_def}
[\x]_T \;\coloneqq\; \sum_{t=1}^{T-1}\|\x_{t+1}-\x_t\|^2 .
\end{equation}
For the RMSProp-style update above,
\begin{equation}\label{eq:qv_rmsprop}
[\x]_T
\;=\;
\sum_{t=1}^{T-1}\left\|\frac{\gamma}{\sqrt{v_t}+\epsilon}\odot g_t\right\|^2
\;\approx\;
\gamma^2\sum_{t=1}^{T-1}\sum_{i=1}^d \frac{g_{t,i}^2}{v_{t,i}} ,
\end{equation}
where the approximation hides only the smoothing by $\epsilon$.
\end{comment}
More directly, it can be shown to serve as the leading term controlling the \emph{adaptive gradient energy} $\sum_{t=1}^{T}
\sum_{i=1}^{d}\gamma_{t,i}(\nabla f(\x_t))_i^2$, which directly governs the convergence of Adam.
\begin{comment}
\begin{align}
\overbrace{\sum_{t=1}^{T}
\sum_{i=1}^{d}
\gamma_{t,i}(\nabla f(\x_t))_i^2}^{\text{adaptive gradient energy}} &\leq \mathcal O(1)+\overbrace{\sum_{t=1}^{T}
\sum_{i=1}^{d}
\gamma_{t,i}^2g_{t,i}^2}^{\text{adaptive stochastic gradient energy}}\notag\\&\quad+\underbrace{\left|\sum_{t=1}^{T}\sum_{i=1}^{d}\left(\Expect\left[\gamma_{t,i}(\nabla f(\x_t))_ig_{t,i}\big|\mathscr{F}_{t-1}\right]-\gamma_{t,i}(\nabla f(\x_t))_ig_{t,i}\right)\right|}_{\text{Martingale term}}.
\end{align}
The martingale term in the above bound can be controlled by the same second-order quantity through martingale inequalities. 
\end{comment}
Therefore, we focus below on the adaptive stochastic gradient energy $\sum_{t=1}^{T} \sum_{i=1}^{d} \gamma_{t,i}^2g_{t,i}^2$. We show that the second-moment normalization embedded in $\gamma_t$ yields a trajectory-wise \emph{logarithmic} bound on this quantity, which is crucial for deriving a sharper convergence rate. To see this, recall $\gamma_t=\gamma(\sqrt{v_t}+\epsilon)^{-1}$, where $v_{t,i}
=\beta_2^{t}v+(1-\beta_2)\sum_{k=1}^{t}\beta_2^{t-k}g_{k,i}^2$ for $t\ge1.$
When \(\beta_2\) is close to \(1\), for instance
\(\beta_2=1-\Theta(1/T)\), and
%the denominator \(v_{t,i}\) behaves as a smoothed running average of the past coordinatewise squared stochastic gradients. With
\(\gamma=\eta/\sqrt T\), we obtain the characteristic self-normalized form
\begin{align}\label{eq:qv_log_heuristic}
\sum_{t=1}^{T}
\sum_{i=1}^{d}
\gamma_{t,i}^{2}g_{t,i}^{2}
\approx
\eta^2
\sum_{i=1}^{d}
\sum_{t=1}^{T}
\frac{g_{t,i}^{2}}
{v+\sum_{s=1}^{t}g_{s,i}^{2}}
\lesssim
\eta^2 d
\log\!\left(
1+
\frac{\sum_{t=1}^{T}\|g_t\|^2}{vdT}
\right),
\end{align}
where the second inequality is up to universal constants coming from the exponential moving average and follows from the inequality
$\sum_{t=1}^{T}
\frac{a_t}{a_0+\sum_{s=1}^{t}a_s}
\le
\log\!\left(
1+\frac{\sum_{t=1}^{T}a_t}{a_0}
\right)$ for
$a_0>0$ and $a_t\ge0$. 

Crucially, \cref{eq:qv_log_heuristic} requires no moment or tail assumptions on the stochastic gradient $g_t$. The normalization by $v_t$ transforms the raw stochastic-gradient energy into a logarithmic functional of $\sum_t \|g_t\|^2$. As a result, such {\em logarithmic} dependence gives the adaptive stochastic gradient energy a {\em much lighter tail}, meaning that large values occur with much smaller probability. This enables all-order moment control of the adaptive stochastic-gradient energy along Adam trajectories. These moment bounds then control the martingale and smoothness-residual terms in the descent inequality, yielding stretched-exponential moment bounds and ultimately a high-probability bound on the gradient energy
$\sum_{t=1}^{T}\sum_{i=1}^{d}\gamma_{t,i}(\nabla f(\x_t))_i^2$ with only $\polylog(\delta^{-1})$ dependence on the confidence parameter. 

\textbf{Comparison with SGD.}
In comparison, SGD updates the variable as follows
\begin{equation}\label{eq:sgd_update}
\x_{t+1}=\x_t-\gamma\, g_t
\end{equation}
and accumulates \emph{raw} stochastic gradients with a constantstepsize. Its quadratic variation is given by
\begin{equation}\label{eq:qv_sgd}
[\x]_T^{\mathrm{SGD}}
\;=\;
\sum_{t=1}^{T-1}\|\x_{t+1}-\x_t\|^2
\;=\;
\gamma^2\sum_{t=1}^{T-1}\|g_t\|^2,
\end{equation}
whose tail behavior is governed directly by the tail of $\|g_t\|^2$. Under a mere second-moment assumption on the stochastic gradients,
the sum $\sum_{t}\|g_t\|^2$ does not admit exponential-type concentration, and hence rare but large realizations of $g_t$ can dominate \cref{eq:qv_sgd}. As a consequence, standard martingale arguments can at best yield polynomial tail bounds for SGD's trajectory fluctuations, resulting no better dependence on the confidence parameter than $\mathcal{O}(\delta^{-1})$. 

Thus, the second-moment normalization fundamentally benefits the accumulation of trajectory quadratic variation for Adam, making it light-tailed. Capturing this statistical advantage requires a refined high-order analysis under high-probability guarantees. Our analysis identifies precisely this mechanism, yielding an improved convergence rate over SGD.
%that upgrades the confidence dependence from the $\mathcal O(\delta^{-1})$ regime of SGD under bounded second-moment noise to a much sharper $\polylog(\delta^{-1})$ regime.

\textbf{Comparison with existing analysis of Adam:} Existing analyses of Adam primarily control only the first moment, namely the expectation, and therefore do not capture the benefit arising from the light logarithmic tail effect. Exploiting higher-order moments, however, is highly nontrivial and requires a more sophisticated stopping-time and martingale-based analysis, which is the focus of the paper.

\subsection{Convergence Rate of Adam}
\label{subsec:main_convergence}

We now formally establish how second-moment normalization benefits Adam over SGD by proving a separation between their high-probability convergence guarantees. 

% Throughout this subsection, we work under $L$-smoothness (Assumption~\ref{ass:smooth}) and the second moment bounded variance condition
% (Assumption~\ref{ass:abc}). 

%\paragraph{A trajectory-wise intermediate bound.}
The second-moment normalization first manifests itself through a high-probability control of the \emph{adaptive gradient energy},
namely the descent-dominating quantity
$\sum_{t=1}^{T}\sum_{i=1}^{d}\gamma_{t,i}\,\bigl(\nabla f(\x_t)\bigr)_i^2$, 
which directly inherits the normalization induced by the second-moment accumulator.
In particular, this trajectory functional admits only $\mathrm{polylog}(1/\delta)$ confidence dependence as shown in the following proposition. The formal statement of this result is presented as Proposition~\ref{thm:precond_energy_bound} in Appendix~\ref{fasdfsadneaiuncfaleiruerqegtw}.
\begin{proo}
\label{thm:precond_energy_informal}
Under Assumptions~\ref{ass:nonneg}--\ref{ass:abc}, consider Adam, as specified in Algorithm~\ref{alg:adam}, run for $T \ge 10$ iterations with $\beta_1\in[0,1),$ $\beta_2=1-1/T$ and step size $\gamma=\eta/\sqrt{T}$.\footnote{
The theorem already allows any fixed $\beta_1\in[0,1)$. We use the finite-horizon calibration $\beta_2=1-1/T$ and $\gamma=\eta/\sqrt T$ for the cleanest vanishing rate; such horizon-dependent parameter calibrations are common in theoretical analyses of Adam-type methods \citep{shi2021rmsprop,zhang2022adam,wang2024closing,jin2024comprehensive}. The argument can be adapted to fixed general $\beta_2$ by tracking the corresponding parameter dependence, since the logarithmic pathwise control remains valid, but it leaves a fixed-bias term in the final bound; see Appendix~\ref{sec:general_beta_discussion} for details.}
There exists an $\eta>0$ (depending only on fixed problem parameters) such that for any $0<\delta<1$, with probability at least $1-\delta$,
\begin{equation}\label{eq:energy}
\sum_{t=1}^{T}\sum_{i=1}^{d}\gamma_{t,i}\,\bigl(\nabla f(\x_t)\bigr)_i^2
\;=\;
\mathcal{O}\!\left(\log^6\!\left(\frac{1}{\delta}\right)\right).
\end{equation}
\end{proo}

%\paragraph{From preconditioned energy to convergence rate.}
To establish the bound on $\frac{1}{T}\sum_{t=1}^{T}\|\nabla f(\x_t)\|^2$ which directly implies the convergence,
we need to remove the coordinatewise adaptive weights $\gamma_{t,i}$ in \cref{eq:energy}.
Such a ``de-preconditioning'' step is inherently lossy, because converting a bound on
$\sum_{t,i}\gamma_{t,i}(\nabla f(\x_t))_i^2$ into a bound on $\sum_t\|\nabla f(\x_t)\|^2$
requires a lower control on the effective preconditioner, which in turn introduces a
\emph{sublinear} confidence degradation.
Consequently, although the second-moment normalization yields $\mathrm{polylog}(\delta^{-1})$ control at the preconditioned level, the final convergence rate has a $\delta^{-1/2}$ factor as shown in the following theorem. The formal statement of this result is presented as Theorem~\ref{lemma_0_0} in Appendix~\ref{formal_thm_1}.

\begin{thm}
\label{thm:main_informal}
Under Assumptions~\ref{ass:nonneg}--\ref{ass:abc}, for Adam, as specified in Algorithm~\ref{alg:adam}, with $\beta_2=1-1/T$ and $\gamma=\eta/\sqrt{T}$,
there exists an $\eta>0$ (depending only on fixed problem parameters) such that for any $0<\delta<1$, with probability at least $1-\delta$,
\[
\frac{1}{T}\sum_{t=1}^{T}\|\nabla f(\x_t)\|^2
\;=\;
\widetilde{\mathcal{O}}\!\left(\frac{1}{\sqrt{\delta\,T}}\right).
\]
\end{thm}
%See Theorem~\ref{lemma_0_0} in Appendix~\ref{formal_thm_1} for the full statement. 
The above guarantee permits $\beta_1=0$, in which case Algorithm~\ref{alg:adam} reduces to a RMSProp-type update. This indicates that the improved confidence dependence is primarily driven by the second-moment normalization (the $v_t$-accumulator), rather than the first-moment term (momentum). 
%\yingbin{explicitly talk about first-moment term will lead to what?}
\footnote{At present, the precise role of the first-moment (momentum) term in improving convergence rates for stochastic optimization remains theoretically unclear \citep{kidambi2018insufficiency,gadat2018stochastic}. While momentum is known to influence the transient behavior of stochastic algorithms \citep{pan2023accelerated,yuan2016influence}, there is no consensus or definitive result establishing an improvement in convergence rates or confidence dependence attributable solely to the first-moment term.}

Theorem \ref{thm:main_informal} improves the best-known bound of $\delta^{-3/2}$ for Adam in \citep{zou2019sufficient} under the same second-moment-type assumptions by a factor of $\delta^{-1}$. As explained in \Cref{subsec:self_normalization}, the central idea behind our sharper analysis of Adam is to leverage the logarithmic control induced by second-moment normalization, which yields sharper, lighter tails for the adaptive stochastic-gradient energy.

We further note that when converting the control of the adaptive gradient energy in Proposition \ref{thm:precond_energy_informal} into a convergence rate on the averaged gradient energy in Theorem \ref{thm:main_informal}, part of the gain is unavoidably lost. Such a converting step requires lower control on the adaptive weights and, in our argument,
introduces a $\delta^{-1/2}$ degradation. This loss is specific to adaptive methods, as it
comes from handling the random, data-dependent $\{v_{t,i}\}$. 
%In contrast, SGD uses a deterministic stepsize $\gamma$ and does not incur an extra confidence penalty in this conversion. 
Overall, even with such a converting loss, Adam still enjoys a strictly better confidence dependence than SGD as we will show in \Cref{sec:adamsgd}.

%\subsection{High-Probability Lower Bound for SGD}\label{subsec:main_sgd}

\subsection{Advantage of Adam over SGD}\label{sec:adamsgd}

%\paragraph{A high-probability lower bound for SGD.}
To show that the improved confidence dependence is intrinsic to Adam, 
%under the same assumptions, 
we provide a high-probability lower bound for SGD as a reference benchmark rather than as a main result. 

In order to contrast with Proposition \ref{thm:precond_energy_informal} for Adam on adaptive gradient energy, we first present the following universal lower bound for SGD for any step size $\gamma>0$.
The formal statement of this result is presented as Proposition \ref{prop:hp-lower-bound-same-cuts} in Appendix~\ref{Teacher_1}.
\begin{proo}\label{prop:sgd_peak_lower_informal}
Fix any step size $\gamma>0$. There exists a \emph{fixed} class $\mathcal C$ of instances $(f,\mathscr O),$ where $\mathscr O$ denotes a stochastic first-order oracle, satisfying the standard
assumptions (Assumptions~\ref{ass:nonneg}--\ref{ass:abc}) with common constants $L,f^*,C$, such that for every $T\ge 10$ and every sufficiently small $\delta$, one can choose $(f,\mathscr O)\in\mathcal C$ such that, with probability at least $\delta$, the SGD iterates $\{\x_t\}_{t=1}^T$ given in \cref{eq:sgd_update} satisfy
\[
\sum_{t=1}^{T}\gamma\|\nabla f(\x_t)\|^2
\;\ge\;
\widetilde{\Omega}\!\left(\frac{1}{\delta}\right),
\]
where $\widetilde{\Omega}(\cdot)$ hides quantities independent of $\gamma$.
%\[
%\mathbb{P}\!\left(
%\sum_{t=1}^{T}\gamma\|\nabla f(\x_t)\|^2
%\;\ge\;
%\widetilde{\Omega}\!\left(\delta^{-1}\right)
%\right)\;\ge\;\delta .
%\]
\end{proo}
{\bf Where the advantage of Adam comes from:} The comparison of Propositions \ref{prop:sgd_peak_lower_informal} with \ref{thm:precond_energy_informal} makes Adam’s advantage over SGD apparent at the adaptive gradient energy level, where the effect of its coordinate-wise second-moment normalization is most directly reflected. Specifically, Adam’s scaling $(\sqrt{v_t}+\epsilon)^{-1}$ directly controls the oracle noise and improves the high-probability dependence from the $\mathcal{O}(\delta^{-1})$ behavior of SGD to $\polylog(\delta^{-1})$ control.

%See Proposition \ref{prop:sgd_peak_lower_informal} in Appendix~\ref{Teacher_1} for the full statement. 
The following theorem provides a lower bound for SGD class that can take any step size. The formal statement of this result is presented as Theorem~\ref{prop:grad-lower-bound} in Appendix~\ref{Teacher_1}.
\begin{thm}\label{prop:sgd_lower_informal}
Under the same conditions as those in Proposition \ref{prop:sgd_peak_lower_informal}, for every $T\ge 10$ and every sufficiently small $\delta$,
there exists $(f,\mathscr O)\in\mathcal C$ such that, with probability at least $\delta$, the SGD iterates $\{\x_t\}_{t=1}^T$ given in \cref{eq:sgd_update} satisfy
\[
\frac{1}{T}\sum_{t=1}^{T}\|\nabla f(\x_t)\|^2
\;\ge\;
\widetilde{\Omega}\!\left(\frac{1}{\delta\sqrt{T}}\right).
\]
%\[\mathbb{P}\!\left(\frac{1}{T}\sum_{t=1}^{T}\|\nabla f(\x_t)\|^2\;\ge\;\widetilde{\Omega}\!\left(\frac{1{\delta\sqrt{T}}\right)\right)\;\ge\;\delta .\]
\end{thm}
%See Theorem \ref{prop:grad-lower-bound} in Appendix~\ref{Teacher_1} for the full statement. 
%\subsection{Comparison between Adam and SGD}

{\bf Convergence advantage of Adam over SGD:} Comparison between Theorem \ref{thm:main_informal} and Theorem \ref{prop:sgd_lower_informal} implies that Adam enjoys a convergence rate of $\widetilde{\mathcal{O}}\!\left(\frac{1}{\sqrt{\delta\,T}}\right)$, whereas SGD cannot be faster than $\widetilde{\Omega}\!\left(\frac{1}{\delta\sqrt{T}}\right)$. This reveals a clear separation between the two algorithms. To our knowledge, this is the first result that rigorously characterizes a convergence-performance gap between Adam and SGD. From a distributional viewpoint (i.e., without fixing $\delta$ a priori), this improved tail behavior predicts a provable shift over repeated runs. Namely, the performance curve of Adam concentrates on smaller values than that of SGD,
so in typical empirical plots of $\frac{1}{T}\sum_{t=1}^{T}\|\nabla f(\x_t)\|^2$ across many trials, Adam’s curves lie overall below SGD’s. 

\textbf{Insights from lower bounds for SGD under time-varying step sizes:} In Appendix~\ref{sec:sgd_time_varying}, we provide two companion lower bounds showing that the limitation of SGD is not merely due to constant step sizes: it persists even under time-varying schedules. The key limitation is that such schedules are fixed before observing the stochastic gradients and therefore cannot selectively shrink the update when a rare but destructive realization occurs. Avoiding this obstruction requires data-dependent normalization, precisely the mechanism provided by Adam’s second-moment accumulator.

\section{Proof Sketch of Proposition \ref{thm:precond_energy_informal} and Theorem \ref{thm:main_informal}}
\label{sec:proof_sketch}

In this section, we provide a proof sketch of Proposition \ref{thm:precond_energy_informal} and Theorem~\ref{thm:main_informal}, highlighting the key proof mechanism behind our high-probability convergence bound.

We first apply the standard momentum-removal transformation \citep{liu2020improved}, and define
\begin{equation}\label{eq:y_def}
\y_1:=\x_1,
\qquad
\y_t
:=
\frac{\x_t-\beta_1\x_{t-1}}{1-\beta_1},
\qquad t\ge 2.
\end{equation}
A direct calculation using the Adam update gives
\begin{equation}\label{eq:y_increment_sketch}
\y_{t+1}-\y_t
=
-\gamma_t\odot g_t
+
\frac{\beta_1}{1-\beta_1}
(\gamma_{t-1}-\gamma_t)\odot m_{t-1}.
\end{equation}
Thus, in contrast to the constant-stepsize momentum transformation, one must retain and control the adaptive-stepsize correction in \cref{eq:y_increment_sketch}.

\noindent\textbf{Step 1 (Martingale-Preserving Descent Inequality).}
Using \(L\)-smoothness and the exact increment \cref{eq:y_increment_sketch}, we establish the following one-step descent inequality.

\begin{lem}[Descent Lemma, informal]\label{descent_lemma'}
Let \(\{\x_t\}_{t=1}^{T}\), \(\{v_t\}_{t=1}^{T}\) be generated by Adam, and \(\{\y_t\}_{t=1}^{T}\) be defined in \cref{eq:y_def}. 
Suppose Assumptions~\ref{ass:nonneg}-\ref{ass:abc} hold, and set \(\beta_2=1-\frac1T\). 
Then, for every \(1\le t\le T-1\),
\begin{align}\label{descent_ineq'}
f(\y_{t+1})
+&
\sum_{i=1}^{d}
\gamma_{t,i}
\bigl(\nabla f(\x_{t+1})\bigr)_i^2
-
\left(
f(\y_t)
+
\sum_{i=1}^{d}
\gamma_{t-1,i}
\bigl(\nabla f(\x_t)\bigr)_i^2
\right)
\notag\\
&\le
-\frac{1}{8}
\sum_{i=1}^{d}
\gamma_{t-1,i}
\bigl(\nabla f(\x_t)\bigr)_i^2
+
D_t
+
P_t .
\end{align}
where \(\{D_t\}_{t\ge1}\) and \(\{P_t\}_{t\ge1}\) are defined in the formal statement of this lemma in Lemma~\ref{descent_lemma}.
\end{lem}

Lemma~\ref{descent_lemma'} identifies two main error terms: \(\{D_t\}_{t\ge1}\) (referred to as {\em stochastic gradient martingale difference}) captures the martingale difference noise induced by stochastic gradients and the adaptive-stepsize fluctuation, and \(\{P_t\}_{t\ge1}\) (referred to as {\em second-order descent residual}) captures the second-order residual terms arising from Taylor series expansion of the descent inequality. 
A key difference of \cref{descent_ineq'} from standard expected-descent analyses is that we retain the martingale term $D_t$ explicitly instead of averaging it out. This preserves the stochastic fluctuation needed for higher-order tail analysis, allowing us to later control the accumulated error $\sum_t D_t$ via martingale inequalities and obtain high-probability bounds.

%The cumulative effects of \(\{D_t\}_{t\ge1}\) and \(\{P_t\}_{t\ge1}\) determine the convergence rate. The following steps provide tight bounds on these two cumulative error terms.

%A key difference of \cref{descent_ineq'} from standard expected-descent analyses is that we do not average out the martingale term at this stage. Instead, we keep the stochastic fluctuation explicitly in the pathwise descent inequality. This is important because expectation-based arguments only control first moments, whereas our goal is to understand the tail behavior of Adam via higher-order analysis. By retaining the martingale term \(D_t\), we can later apply martingale inequalities to the accumulated error \(\sum_t D_t\) and obtain high-probability control.

%\textcolor{blue}{A key feature of \cref{descent_ineq'} is that the martingale difference term \(D_t\) is kept explicitly in the pathwise descent inequality. This differs from the traditional expected-descent approach, where one takes conditional expectation at each step and removes the martingale fluctuation immediately. Such an expected-descent inequality is sufficient for first-moment analysis, but it loses the information needed to study high-order fluctuations. By retaining \(D_t\) before taking expectations, we can later apply martingale inequalities to the accumulated error \(\sum_t D_t\), which is essential for obtaining high-probability, rather than merely in-expectation, bounds.}

\noindent\textbf{Step 2 (Introducing a Stopped Process).}
Define the shifted objective as
\[
\bar f(\x):=f(\x)-f^\star+1,
\quad \text{where }\;
f^\star:=\inf_{\x\in\mathbb R^d}f(\x).
\]
For a threshold \(G>0\), introduce the stopping time
\begin{equation}\label{eq:tauG_def}
\tau_G
:=
\min\{t:\bar f(\x_t)\ge G\}.
\end{equation}
On the stopped trajectory \(t<\tau_G\), we have \(\bar f(\x_t)\le G\), which allows all stochastic-gradient moments and drift terms to be controlled in terms of \(G\).

%The stopping time creates a deterministic scale $G$ against which all high-order terms can be measured. Before the process exits the controlled region, the gradient-dependent terms appearing in the martingale quadratic variation and in the residual terms can be bounded by powers of \(G\). The proof is constructed so that these powers remain strictly below the \(G^{p/6}\) scale corresponding to the \(p/6\)-th moment of the stopped Lyapunov quantity. This separation of scales is what makes it possible to close the stopping-time argument.

Summing \cref{descent_ineq'} up to \(\tau_G\wedge T\) and absorbing part of the descent into the residual terms yield the following stopped control
\begin{align}\label{eq:traj_bound_k'}
W_{\tau_G\wedge T}
:=
\bar f(\y_{\tau_G\wedge T})
+
\frac{1}{16}
\overbrace{
\sum_{t=1}^{\tau_G\wedge T}
\sum_{i=1}^{d}
\gamma_{t,i}
\bigl(\nabla f(\x_t)\bigr)_i^2
}^{\mathcal E_{\tau_G\wedge T}}
\le
\mathcal O(1)
+
\left|
\sum_{t=1}^{\tau_G\wedge T-1}
D_t
\right|
+
\sum_{t=1}^{\tau_G\wedge T-1}
P_t.
\end{align}
The above equation defines the adaptive gradient energy $\mathcal E_{\tau_G\wedge T}$, which is the main quantity we bound later and which ultimately determines the convergence rate.
%The factor \(1/16\) reflects the descent coefficient after accounting for the predictable drift generated by the adaptive-stepsize correction.

\noindent\textbf{Step 3 (All-Order Bounds on Stopped Martingale and Residual Terms).} 
This step establishes all-order moment bounds for the stopped martingale $\left|\sum_{t=1}^{\tau_G\wedge T-1}D_t\right|$ and the stopped residual term $\sum_{t=1}^{\tau_G\wedge T-1}
P_t.$ To this end, both error terms can be controlled by two key cumulative quantities: the \emph{adaptive stochastic gradient energy}
$\sum_{t=1}^{\tau_G \wedge T-1}
\sum_{i=1}^{d}
\gamma_{t,i}^2 g_{t,i}^2 $ 
and the \emph{adaptive momentum energy} $\sum_{t=1}^{\tau_G \wedge T-1}
\|\gamma_t \odot m_t\|^2.$
Bounds for these two quantities are provided as follows.

\begin{lem}\label{lemma_sum_two'}
Let \(S_t:=dv+\sum_{s=1}^{t}\|g_s\|^2\). Then, for any stopping time \(\mu\),
\begin{align}\label{eq:pathwise_log_kept}
\sum_{t=1}^{\mu\wedge T}
\sum_{i=1}^{d}
\gamma_{t,i}^2g_{t,i}^2
\le
\mathcal O\!\left(
\log\left(
1+\frac{S_{\mu\wedge T}}{vdT}
\right)
\right),
\quad
\sum_{t=1}^{\mu\wedge T}
\|\gamma_t\odot m_t\|^2
\le
\mathcal O\!\left(
\log\left(
1+\frac{S_{\mu\wedge T}}{vdT}
\right)
\right).
\end{align}
\end{lem}

%Lemma~\ref{lemma_sum_two'} provides the key logarithmic pathwise bounds. Their logarithmic growth, induced by the normalization in \(\gamma_t\), is the main reason Adam achieves a sharper convergence bound, as illustrated in \Cref{subsec:self_normalization}.

Lemma 2 is where Adam’s second-moment normalization becomes essential: by rescaling each coordinate with its running second-moment estimate, Adam converts cumulative stochastic-gradient and momentum energies into logarithmic trajectory-wise quantities. This logarithmic control is what enables sharp tail bounds under only a bounded-variance assumption.

We next bound the martingale term $\left|\sum_{t=1}^{\tau_G\wedge T-1} D_t\right|$ through its quadratic variation using the Burkholder-Davis-Gundy (BDG) inequality (see Lemma~\ref{lem:burkholder}). Combining this with the logarithmic energy bounds in Lemma~\ref{lemma_sum_two'} and carefully constructing the moment orders yield:
\begin{align}\label{eq:T1_final_kept}
\mathbb E\!\left[
\left|
\sum_{t=1}^{\tau_G\wedge T-1}
D_t
\right|^{\frac{p}{6}}
\right]
\le
\left(\frac{p}{6}\right)^{\frac{p}{6}}
G^{\frac{p}{12}}
\mathcal O\!\left(
\mathbb E\!\left[
\log^{\frac{p}{12}}
\left(
1+\frac{S_{\tau_G\wedge T-1}}{vdT}
\right)
\right]
+1
\right),
\end{align}
where $S_{\tau_G\wedge T-1} := d v + \sum_{t=1}^{{\tau_G\wedge T-1}}\|g_t\|^2.$ Similarly, the stopped residual term can be bounded as
\begin{align}\label{eq:T2_final_kept}
\mathbb E\!\left[
\left(
\sum_{t=1}^{\tau_G\wedge T-1}
P_t
\right)^{\frac{p}{6}}
\right]
\le
G^{\frac{p}{12}}
\mathcal O\!\left(
\mathbb E\!\left[
\log^{\frac{p}{6}}
\left(
1+\frac{S_{\tau_G\wedge T-1}}{vdT}
\right)
\right]
+1
\right).
\end{align}
%The same localization principle also controls the predictable residual terms. 
Both the martingale and residual terms are controlled at the smaller $G^{p/12}$ scale, below the stopped scale $G^{p/6}$. These exponents are chosen to keep the exponential moment of $W_{\tau_G \wedge T}$ finite in Step 4.
%This allows the later use of Young-type inequalities to separate the \(G\)-dependent terms from the moment of \(W_{\tau_G\wedge T}\), rather than having the error terms overwhelm the stopping threshold.

%\textcolor{blue}{Thus both the martingale accumulation and the predictable residual accumulation remain below the full \(G^{p/6}\) stopped scale. The \(G^{p/12}\) dependence is the quantitative reason why the proof can later use Young-type inequalities to separate the \(G\)-dependent terms from the moment of \(W_{\tau_G\wedge T}\), rather than having the error terms overwhelm the stopping threshold.}

%\textbf{Step 4 (Bounds on Adaptive Gradient Energy).} 

\textbf{Step 4 (All-Order and Stretched-Exponential Moment Bounds on $W_{\tau_G\wedge T}$).} This step first combines \cref{eq:traj_bound_k'}, \cref{eq:T1_final_kept}, and \cref{eq:T2_final_kept} to convert the all-order moment bounds for the martingale and residual terms into an all-order moment bound for $W_{\tau_G\wedge T}$. By ensuring these bounds to be sufficiently sharp, we tightly control the growth of $W_{\tau_G\wedge T}$'s moments with respect to the moment order. We then use a Taylor-series argument to upgrade this all-order moment control to the following stretched-exponential moment bound, where we show that the factorial denominator in the Taylor expansion dominates the moment growth, leading to finite stretched-exponential series:
%To this end, we first convert the all-order moment bounds for the martingale and residual terms into the following stretched-exponential moment bound for $W_{\tau_G \wedge T}$ by combining \cref{eq:traj_bound_k'}, \cref{eq:T1_final_kept}, and \cref{eq:T2_final_kept} with a Taylor-series argument:
\begin{equation}\label{eq:exp_mgf_W}
\mathbb E\!\left[
\exp\left\{
\left(2W_{\tau_G\wedge T}\right)^{1/6}
\right\}
\right]
\le
1+e^{\bar f^{1/6}(\y_1)}+e^{G^{1/7}}+UG^2,
\end{equation}
where \(U\), defined in \cref{U_U}, is independent of both \(T\) and \(\delta\).

\textbf{Step 5 (Bounds on Adaptive Gradient Energy and Convergence Rate for Adam).} This step first provides a high-probability control on the adaptive gradient energy (Proposition~\ref{thm:precond_energy_informal}), and then translates such a bound to the final convergence bound for Adam (Theorem~\ref{thm:main_informal}).

\section{Conclusion}
\label{sec:conclusion}
In this paper, we provided a provable explanation for a common empirical observation: Adam can outperform vanilla SGD under the classical bounded variance noise assumption. Our main insight is that Adam’s diagonal second-moment normalization yields sharper tail control of the iterates’ quadratic variation and the adaptive gradient energy. Using a stopping-time localization and martingale argument, we derive the high-probability convergence guarantee. Future work includes tightening the de-preconditioning step, extending to other adaptive methods, and exploring structured settings with state-dependent heavy-tailed noise.

%Concretely, under standard smoothness and bounded variance, we obtain a trajectory-wise high-probability bound on the cumulative preconditioned gradient energy with only $\mathrm{polylog}(\delta^{-1})$ dependence, leading to
%\[
%\frac{1}{T}\sum_{t=1}^T \|\nabla f(\x_t)\|^2= \widetilde{\mathcal{O}}\!\left(\frac{1}{\sqrt{\delta T}}\right)\quad \text{w.h.p.}
%\]
%In contrast, any comparable high-probability guarantee for SGD must incur at least $\widetilde{\Omega}\!\left(\frac{1}{\delta\sqrt{T}}\right)$ dependence on $\delta$, giving the first provable separation between Adam and SGD in a convergent regime.

%Future work includes tightening the de-preconditioning step, extensions to other adaptive methods, and studying structured settings with state-dependent heavy-tailed noise.

% Acknowledgments are commented out for the anonymous submission to avoid revealing identity.
% Restore this section in the camera-ready version.
% \begin{ack}
% The work of Ruinan Jin and Yingbin Liang was supported in part by the U.S.\ National Science Foundation under the grants ECCS-2113860, ECCS-2413528, and DMS-2134145. The work of Shaofeng Zou was supported in part by U.S.\ National Science Foundation under the grants ECCS-2438392 (CAREER) and CCF-2438429, and by  DARPA under Agreement No.D25AP00191.
% \end{ack}
\bibliographystyle{apalike}
\bibliography{sample}

\appendix
\section{Additional Related Work}\label{sec:extendedrelatedwork}

\paragraph{Convergence for Adam Variants.} There has also been extensive work on Adam-type algorithms, such as Adam-norm \citep{wang2024convergence}, AdamW \citep{zhou2024towards,li2025d}, AdaGrad-norm \citep{wang2023convergence,faw2023beyond}, Adam+\citep{liu2020}, and several other variants of Adam (e.g.,~\citep{guo2021novel,huang2021super,zhou2024on}). Related studies have further investigated their generalization properties and comparisons with gradient descent (e.g.,~\cite{wilson2017marginal,zhang2020why,zhou2020towards,zou2023understanding}), as well as their connections to SignGD and related methods (e.g.,~\cite{balles2018dissecting,bernstein2019signsgd,kunstner2023noise,xie2024implicit}). 

\paragraph{High-Probability Convergence Upper Bounds for SGD.} There has been extensive work on establishing convergence guarantees for SGD. Below we summarize the work on high-probability upper bounds for SGD (which are more relevant to this paper). Note that the related work on lower bounds for SGD are provided in the main body of the paper.

%There is a substantial literature on high-probability guarantees for SGD and related methods, and the confidence dependence is largely dictated by the stochastic-gradient tail condition. 

For convex and strongly convex nonsmooth objectives, \citet{harvey2019tight} proved high-probability upper and lower bounds under almost-surely bounded stochastic subgradients, with \(\mathcal O(\log(\delta^{-1}))\) confidence dependence and matching lower-bound phenomena in the \(T\)-dependence; \citet{harvey2019simple} obtained \(\mathcal O(\log(\delta^{-1}))\)-type high-probability dependence for strongly convex SGD under the same bounded stochastic-subgradient regime, and \citet{davis2021low} converted low-probability convex-optimization guarantees into high-confidence guarantees with \(\mathcal O(\log(\delta^{-1}))\) overhead. 

For smooth nonconvex objectives, \(\mathcal O(\log(\delta^{-1}))\) dependence is typically obtained under explicit light-tail assumptions. Specifically, \citet{liu2023high} assumed sub-Gaussian stochastic-gradient noise and obtained \(\mathcal O(\log(\delta^{-1}))\)-type dependence, while \citet{madden2024high} treated norm-sub-Gaussian and norm-sub-Weibull noise and obtained \(\mathcal O(\log(\delta^{-1}))\)-type dependence up to tail-parameter-dependent factors. Under heavy-tailed noise, high-probability guarantees usually require robustification rather than vanilla SGD: \citet{gorbunov2020stochastic} studied accelerated gradient clipping and obtained high-probability complexity bounds with logarithmic confidence dependence \(\mathcal O(\log(\delta^{-1}))\); \citet{cutkosky2021high} obtained nonconvex high-probability rates under bounded \(p\)-th moments with \( p\in(1,2]\), with confidence dependence entering through \(\mathcal O(\log(\delta^{-1}))\); and \citet{nguyen2023high} proved clipped-SGD high-probability rates under bounded \(p\)-th moments, with \(\mathcal O(\mathrm{polylog}(\delta^{-1}))\)-type confidence dependence up to problem- and moment-dependent logarithmic factors.
At the expectation level, \citet{jiang2025provable} gave upper and lower bounds comparing AdaGrad and SGD in stochastic nonconvex optimization, including \(T\)-dependent SGD lower bounds. 
Thus, existing high-probability upper bounds for SGD either achieve \(\mathcal O(\log(\delta^{-1}))\)-type dependence under light-tail or almost-sure boundedness assumptions, or use robustification to handle heavy-tailed noise. 

In the classical second-moment regime, expectation-based analyses such as \citet{ghadimi2013stochastic} under bounded variance, and related analyses such as \citet{nguyen2018sgd} under comparable growth-type conditions without uniformly bounded stochastic gradients, yield high-probability guarantees only through Markov-type conversion, thereby incurring \(\mathcal O(\delta^{-1})\)-type confidence dependence. 
Although \citet{wang2024provable} showed that vanilla SGD may fail under initialization-free stepsize choices, \citet{li2023convex} showed convergence under generalized smoothness with initialization-dependent stepsizes and bounded-variance noise. 
In this second-moment regime, \citet{li2023convex} obtains \(\mathcal O(\delta^{-1})\)-type confidence dependence for vanilla SGD. 
Our lower bound targets vanilla SGD under the same classical smooth nonconvex stationarity criterion with only \(L\)-smoothness and bounded variance, and proves an \(\Omega(\delta^{-1})\)-type confidence lower bound, matching the \(\mathcal O(\delta^{-1})\) dependence of second-moment upper-bound regimes.

\section{Useful Lemmas}\label{useful_lemma}

\begin{lem}[Lemma B.2 in \citet{jin2024comprehensive}]\label{loss_bound}
Suppose that $f(x)$ is differentiable and lower bounded, i.e. $ f^{\ast} = \inf_{x\in \ \mathbb{R}^{d}}f(x) >-\infty$, and $\nabla f(x)$ is Lipschitz continuous with parameter $\mathcal{L} > 0$. Then $\forall \ x\in \ \mathbb{R}^{d}$, we have
\begin{align*}
\big\|\nabla f(x)\big\|^{2}\le {2\mathcal{L}}\big(f(x)-f^{*}\big).
\end{align*}
\end{lem}
\begin{lem}[Burkholder--Davis--Gundy (BDG) inequality, stopped version]
\label{lem:burkholder}
Let \(\{(M_t,\mathscr F_t)\}_{t\ge0}\) be a real-valued martingale with
\(M_0=0\) almost surely, and let \(\tau\) be a stopping time with respect to
\(\{\mathscr F_t\}_{t\ge0}\). Then, for every real number \(q\ge1\), there exist
universal constants \(\lambda_0,\lambda_1>0\), independent of \(q\), \(\tau\), and
\(T\), such that, for all \(T\ge1\),
\[
\lambda_0 q^{-q}
\mathbb{E}\!\left[
S(M_{\tau\wedge T})^q
\right]
\le
\mathbb{E}\!\left[
(M_{\tau\wedge T}^{*})^q
\right]
\le
\lambda_1 q^{q}
\mathbb{E}\!\left[
S(M_{\tau\wedge T})^q
\right],
\]
where
\[
M_{\tau\wedge T}^{*}
:=
\sup_{0\le t\le \tau\wedge T}|M_t|,
\]
and
\[
S(M_{\tau\wedge T})
:=
\left(
\sum_{t=1}^{\tau\wedge T}
(M_t-M_{t-1})^2
\right)^{1/2}.
\]
%Here \(S(M_{\tau\wedge T})\) is the square function, i.e., the square root of the pathwise quadratic variation of the stopped martingale \(M_{\tau\wedge T}\).
Thus the inequality compares the \(q\)-th moment of the stopped martingale
maximal process \(M_{\tau\wedge T}^{*}\) with the \(q\)-th moment of its
quadratic-variation scale.
\end{lem}

\begin{lem}\label{sum:expect:ab}
Let $(\Omega,\mathscr{F},(\mathscr{F}_n)_{n\geq 0},\mathbb{P})$ be a filtered probability space, and let
$\{X_n\}_{n\geq 0}$ be an $(\mathscr{F}_n)$–adapted process such that $X_n \in L^1(\mathbb{P})$ for all $n\geq 0$.
Let $s$ and $t$ be bounded stopping times taking values in $\{0,1,\dots N\}$ for some finite $N\in\mathbb{N}$, satisfying
\( s \le t\le N \) almost surely. Assume in addition that $s$ is a predictable stopping time, i.e. \(\{s = n\} \in \mathscr{F}_{n-1},\ \text{for all } n\ge 1.\)
Then \[\mathbb{E}\left[\sum_{n=s}^{t} X_n\right]=\mathbb{E}\left[\sum_{n=s}^{t} \mathbb{E}\left[X_n \mid \mathscr{F}_{n-1}\right]\right].\]
\end{lem}
\begin{lem}\label{property_3}
For Algorithm~\ref{alg:adam}, at any iteration step \(t\), the following
inequality holds:
\begin{align*}
m_{t,i}^{2}-m_{t-1,i}^{2}
\le
-(1-\beta_{1})m_{t-1,i}^{2}
+
(1-\beta_{1})g_{t,i}^{2}.
\end{align*}
Here, \(m_{t,i}\) denotes the \(i\)-th component of the first-moment term
\(m_t\) at iteration step \(t\).
\end{lem}
\begin{lem}\label{property_2}
Consider the iterates generated by Algorithm~\ref{alg:adam}. Let
\(\bar f(\x):=f(\x)-f^\star+1\), where
\(f^\star:=\inf_{\x\in\mathbb R^d} f(\x)\).
Define
\[
\y_1:=\x_1,
\]
and, for \(t\ge2\),
\[
\y_t
:=
\frac{\x_t-\beta_1\x_{t-1}}{1-\beta_1}.
\]
Then, for any iteration step \(t\ge1\), the following two-sided comparison holds:
\begin{align*}
\bar f(\y_t)
&\le
2\bar f(\x_t)
+
\frac{L\beta_1^2}{(1-\beta_1)^2}
\|\gamma_{t-1}\odot m_{t-1}\|^2,
\\
\bar f(\x_t)
&\le
2\bar f(\y_t)
+
\frac{L\beta_1^2}{(1-\beta_1)^2}
\|\gamma_{t-1}\odot m_{t-1}\|^2.
\end{align*}
Here, \(m_{t-1}\) is the first-moment term in Algorithm~\ref{alg:adam}; for
\(t=1\), we use the convention \(m_0=0\), so that
\(\gamma_0\odot m_0=\mathbf 0\).
\end{lem}
\begin{lem}\label{lem:vt_comparable}
Consider the second-moment recursion in Algorithm~\ref{alg:adam},
\[
v_t
=
\beta_2 v_{t-1}
+
(1-\beta_2)(g_t\odot g_t),
\qquad t\ge 1,
\]
with \(v_0=v\mathbf{1}\) for some \(v>0\), and let
\(\beta_2=1-\frac{1}{T}\) with \(T\ge 2\).
Then, for any \(1\le k<h\le T\) and any \(i\in[d]\), one has
\[
v_{k,i}
\le
4v_{h,i}.
\]
Here, \(v_{t,i}\) denotes the \(i\)-th component of the second-moment term
\(v_t\) at iteration \(t\).
\end{lem}

\begin{lem}\label{property_3.5}
Consider Adam described in Algorithm~\ref{alg:adam}, and assume \(m_0=0\).
For any deterministic iteration index \(t\le T\) and any stopping time \(\mu\), the following inequalities hold:
\begin{align}
\|m_t\|^2
&\le
(1-\beta_1)
\sum_{k=1}^{t}
\beta_1^{t-k}
\|g_k\|^2,
\label{property_3.5_m_bound}
\\
\|\gamma_t\odot m_t\|^2
&\le
4(1-\beta_1)
\sum_{k=1}^{t}
\beta_1^{t-k}
\|\gamma_k\odot g_k\|^2,
\label{property_3.5_gamma_m_bound}
\\
\sum_{t=1}^{\mu\wedge T-1}
\|\gamma_t\odot m_t\|^2
&\le
4
\sum_{t=1}^{\mu\wedge T-1}
\|\gamma_t\odot g_t\|^2.
\label{property_3.5_sum_bound}
\end{align}
\end{lem}

\begin{lem}\label{property_2.54}
Consider Adam described in Algorithm~\ref{alg:adam}, 
with the parameter settings identical to those in Theorem~\ref{thm:main_informal}. 
Then, for any iteration step $t \in [1,T]$, the following inequality holds:
\begin{align*}
\|\x_{t+1}-\x_t\|
  &= \|\gamma_{t} \odot m_t\|
  \le \sqrt{d}\eta  D_{\beta_1},
\end{align*}
where $D_{\beta_1}$ is a constant depending at most on $\beta_1$.
\end{lem}

\section{Formal Statements and Proofs for Adam}
%\subsection{The Proof of Proposition \ref{prop:vt-closed-form}}\label{p_prop:vt-closed-form}
% \subsection{The Proof of \cref{eq:vt-closed-form}}\label{p_prop:vt-closed-form}
% \begin{proof}
% Divide the update by $\beta_2^t$:
% \[
% \frac{v_{t,i}}{\beta_2^t}=\frac{v_{t-1,i}}{\beta_2^{t-1}}+(1-\beta_2)\frac{g_{t,i}^2}{\beta_2^t}.
% \]
% Summing over $t=1,\dots,k$ yields
% \[
% \frac{v_{k,i}}{\beta_2^k}-v=(1-\beta_2)\sum_{t=1}^{k}\frac{g_{t,i}^2}{\beta_2^t}.
% \]
% Multiplying by $\beta_2^k$ gives
% \[
% v_{k,i}=\beta_2^k v+(1-\beta_2)\sum_{t=1}^{k}\beta_2^{\,k-t}g_{t,i}^2.
% \]
% Re-naming $k$ as $t$ completes the proof.
% \end{proof}
\subsection{Formal Statement and Proof of Proposition \ref{thm:precond_energy_informal}}\label{fasdfsadneaiuncfaleiruerqegtw}
\begin{proo}[Formal Statement of Proposition \ref{thm:precond_energy_informal}]\label{thm:precond_energy_bound}
Suppose that Assumptions~\ref{ass:nonneg}--\ref{ass:abc} hold.
Let \(\bar f(\x):=f(\x)-f^\star+1\), where \(f^\star:=\inf_{\x\in\mathbb R^d}f(\x)\).
Consider Adam, as specified in Algorithm~\ref{alg:adam}, run for \(T\ge 10\) iterations with parameters
\(\beta_1\in[0,1)\), \(\beta_2 = 1 - 1/T\), and step-size \(\gamma=\eta/\sqrt{T}\).
For any \(0<\delta<1\), and assume that the step-size constant \(\eta\) satisfies
\begin{align*}
\frac{1}{\eta}
\ge
\max\!\left\{
\frac{8d\epsilon}{v^{3/2}}+\frac{8d}{\sqrt v},
\;
\frac{D_{\beta_1}\beta_1\sqrt{dL}}{1-\beta_1},
\;
1
\right\},
\end{align*}
where \(D_{\beta_1}\) is defined in Lemma~\ref{property_2.54}. Then, with probability at least \(1-\delta\), the following holds
\begin{align*}
\sum_{t=1}^{T}&
\sum_{i=1}^{d}
\gamma_{t,i}
\bigl(\nabla f(\x_t)\bigr)_i^{2}
\\
&\le
16\max\Bigg\{
\left[
\log\!\left(
\frac{24l\bigl(1+e^{2\cdot 6^{1/6}B_0^{1/6}}\bigr)}{\delta}
\right)
\right]^6,
\left[
2\log\!\left(
\frac{24l}{\delta}
\right)
\right]^6,
300^6,
\left[
2\log\!\left(
\frac{24lU^2}{\delta}
\right)
\right]^6
\Bigg\},
\end{align*}
where
\[
l
:=
\exp\left\{
\frac{L\beta_1^2 dD_{\beta_1}^2}{(1-\beta_1)^2}
\right\},\quad B_0
:=
2
\left(
1+\frac{2L}{\sqrt v+\epsilon}
\right)
\bar f(\mathbf{y}_1),
\]
and the constant \(U\) is defined in \cref{U_U}, and depends only on the fixed problem parameters and the initialization, but not on \(T\) or \(\delta\).
\end{proo}

\begin{proof}
We can first derive the following two elementary equalities from the recursive definition 
of $v_{t,i}$ given in Algorithm~\ref{alg:adam}, 
where the second one can be viewed as an iterative expansion of the first recursion:
\begin{align}\label{v_10}
v_{t,i} &= \left(1-\frac{1}{T}\right)v_{t-1,i} + \frac{1}{T}g_{t,i}^2,
\end{align}
\begin{align}\label{v_11}
v_{t,i} &= \left(1-\frac{1}{T}\right)^t v_{0,i} 
+ \frac{1}{T}\sum_{s=1}^{t}\left(1-\frac{1}{T}\right)^{t-s} g_{s,i}^2.
\end{align}
These two identities will be repeatedly used in the subsequent analysis. Meanwhile, in the subsequent analysis, we will repeatedly make use of an auxiliary variable $S_T$ defined below
\begin{align}\label{S_T_T}
S_T := d v + \sum_{t=1}^{T}\|g_t\|^2.
\end{align}

%\noindent\textbf{Step 1 (Eliminating the Momentum Term via Variable Substitution).} 

%\noindent We commence our proof with a classical substitution \citep{liu2020improved}. 
To proceed the proof, we first apply the standard momentum-removal transformation \citep{liu2020improved}, and define
\begin{align}\label{y}
\y_1:=\x_1,
\qquad
\y_t
:=
\frac{\x_t-\beta_1\x_{t-1}}{1-\beta_1},
\qquad t\ge 2.
\end{align}
Equivalently, using the Adam update 
\[\x_t-\x_{t-1}=-\gamma_{t-1}\odot m_{t-1},\] 
we have, for every \(t\ge2\),
\begin{align*}
\y_t
=
\x_t
+
\frac{\beta_1}{1-\beta_1}(\x_t-\x_{t-1})
=
\x_t
-
\frac{\beta_1}{1-\beta_1}
\gamma_{t-1}\odot m_{t-1}.
\end{align*}
For \(t=1\), the same shifted representation can be obtained through the initialization
\(\y_1=\x_1\) and \(m_0=\mathbf 0\). 
This auxiliary sequence removes the first-order momentum term while keeping the adaptive-stepsize variation explicit.

\noindent\textbf{Step 1 (Martingale-Preserving Descent Inequality).}

\noindent Using the auxiliary variable \( \mathbf{y}_t \), we derive a descent inequality that describes the iteration-wise evolution of \( f(\mathbf{x}) \).  
\begin{lem}[Descent Lemma; Formal Statement of Lemma \ref {descent_lemma'}]\label{descent_lemma}
Let $\{\x_t\}_{t=1}^{T}$ and $\{v_t\}_{t=1}^{T}$ be the sequences generated by Algorithm~\ref{alg:adam}, and $\{\y_t\}_{t=1}^{T}$ be generated by \cref{y}.  
Suppose that \textnormal{Assumptions~\ref{ass:nonneg}--\ref{ass:abc}} hold. Then, for every $1\le t\le T-1$,
\begin{align}\label{descent_ineq}
&\left(
f(\y_{t+1})
+
\sum_{i=1}^{d}\gamma_{t,i}(\nabla f(\x_{t+1}))_i^2
\right)
-
\left(
f(\y_t)
+
\sum_{i=1}^{d}\gamma_{t-1,i}(\nabla f(\x_t))_i^2
\right)
\notag\\
&\qquad \le
-\frac{1}{8}
\sum_{i=1}^{d}
\gamma_{t-1,i}(\nabla f(\x_t))_i^2
+
D_t
+
P_t .
\end{align}
Here the first term on the right-hand side is the principal descent term,
$\{(D_t,\mathscr F_t)\}_{t\ge1}$ is the martingale difference noise induced by stochastic gradients and by the adaptive-stepsize fluctuation, and $P_t$ collects all deterministic and higher-order residual terms. The martingale term is given by
\begin{align}\label{D_t_def}
D_t:=D_{t,1}+D_{t,2}+D_{t,3},
\end{align}
where
\begin{align}\label{defination_1}
D_{t,1}
&:=
\sum_{i=1}^{d}
\Expect\!\left[
\gamma_{t,i}(\nabla f(\x_t))_i g_{t,i}
\,\middle|\,
\mathscr F_{t-1}
\right]
-
\sum_{i=1}^{d}
\gamma_{t,i}(\nabla f(\x_t))_i g_{t,i},
\notag\\
D_{t,2}
&:=
\sum_{i=1}^{d}
(\nabla f(\x_t))_i^2
\left(
\Expect\!\left[
|\Delta_{t,1,i}|
\,\middle|\,
\mathscr F_{t-1}
\right]
-
|\Delta_{t,1,i}|
\right),
\notag\\
D_{t,3}
&:=
\frac{\beta_1}{1-\beta_1}
\sum_{i=1}^{d}
\left[
(\gamma_{t-1,i}-\gamma_{t,i})
-
\Expect\!\left[
\gamma_{t-1,i}-\gamma_{t,i}
\,\middle|\,
\mathscr F_{t-1}
\right]
\right]
(\nabla f(\y_t))_i m_{t-1,i}.
\end{align}
The residual term is given by
\begin{align}\label{P_t_def}
P_t
:=
&
\left(
\frac{\beta_1^2L}{2(1-\beta_1)^2}
+
\frac{\beta_1^2L^2\eta}{4(1-\beta_1)^2\sqrt v}
\right)
\|\gamma_{t-1}\odot m_{t-1}\|^2
\notag\\
&+
\frac{3L}{2}
\sum_{i=1}^{d}
\gamma_{t,i}^2 g_{t,i}^2
+
\frac{\beta_1^2L}{(1-\beta_1)^2}
\sum_{i=1}^{d}
(\gamma_{t-1,i}-\gamma_{t,i})^2m_{t-1,i}^2
\notag\\
&+
\frac{140dL^2\eta}{\sqrt v}
\|\gamma_t\odot m_t\|^2
+
C\sum_{i=1}^{d}
\Expect\!\left[
|\Delta_{t,1,i}|
\,\middle|\,
\mathscr F_{t-1}
\right]
\notag\\
&+
8\frac{\beta_1(1+\beta_1)}{(1-\beta_1)^2}
\frac{T}{T-1}
\sum_{i=1}^{d}
\frac{
|(\nabla f(\x_t))_i|
+\sqrt C
+\frac{\epsilon}{\sqrt T+\sqrt{T-1}}
}{\eta}
\gamma_{t-1,i}^2m_{t-1,i}^2
\notag\\
&+
\frac{\beta_1}{1-\beta_1}
\frac{1}{T}
\sum_{i=1}^{d}
\gamma_{t-1,i}
\left|
(\nabla f(\y_t))_i m_{t-1,i}
\right|.
\end{align}
The auxiliary stepsize-difference term is given by
\begin{align}\label{Delta_t1_def}
\Delta_{t,1,i}
:=
\frac{\eta}{\sqrt{T-1}}
\frac{1}{\sqrt{v_{t-1,i}}+\epsilon}
-
\frac{\eta}{\sqrt T}
\frac{1}{\sqrt{v_{t,i}}+\epsilon}.
\end{align}
\end{lem}
\begin{proof}[Proof of Lemma \ref{descent_lemma}]
See Appendix~\ref{p_descent_lemma}.  
\end{proof}

%With the descent lemma established above, we are now prepared to carry out the subsequent analysis.

\noindent\textbf{Step 2 (Introducing a Stopped Process).}

We define the following nonnegative reference function:
\[
\bar f := f - f^* + 1.
\]
High-probability guarantees can be recast as controlling the large-deviation event
\[
\sup_{1\le t\le T} \bar f(\mathbf{x}_t) > G.
\]
For any $G\ge 1$, we set the stepsize constant $\eta$ to satisfy
\begin{align}\label{eta}
    \frac{1}{\eta}\ge \max\left\{\left(\frac{8d\epsilon}{v^{3/2}}+\frac{8d}{\sqrt{v}}\right),\frac{D_{\beta_1}\beta_1\sqrt{dL}}{1-\beta_1},1\right\}.
\end{align}
Then according to standard probabilistic reasoning, 
the maximal deviation event is contained in the first hitting event:
\begin{align*}
\left[ \sup_{1 \le t \le T} \bar f(\mathbf{x}_t) 
  > G \right]
  \subseteq\left[\tau_{G}\le T\right]= \bigcup_{i=1}^{T} \left[ \tau_{G} = i \right],
\end{align*}
where the stopping time $\tau_{G}$ is defined as follows:
\begin{align}\label{tau}
\tau_{G} := \min\left\{t : \bar f(\mathbf{x}_t) \ge G\right\}.
\end{align}
It is straightforward to verify that \(\tau_{G}\) is a stopping time.  
Moreover, \(\tau_{G}\) is in fact a \emph{predictable stopping time}, i.e., for every \(n\),  
the event \([\tau_{G}=n]\) belongs to \(\mathscr{F}_{n-1}\).

\noindent Then, by applying \emph{Markov's} inequality, we can immediately obtain the following upper bound for the large deviation probability.
\begin{align}\label{Lebron}
\mathbb{P}\!\left(\sup_{1\le t\le T}\bar f(\x_t)>G\right)
&\le\mathbb{P}\!\left([\tau_{G}\le T]\right)
=\sum_{i=1}^{T}\mathbb{P}\!\left(\left\{\tau_{G}=i\right\}\right)
\nonumber \\
&\mathop{\le}^{(a)}
\frac{1}{e^{G^{{1}/{6}}}}
\sum_{i=1}^{T}
\Expect\!\left[
e^{\bar f^{1/6}(\x_i)}
\textbf{1}_{\left\{\tau_{G} = i \right\}}
\right]
\notag\\
&\le
\frac{1}{e^{G^{1/6}}}
\Expect\!\left[
e^{\bar f^{{1}/{6}}(\x_{\tau_{G}\wedge T})}
\right]
\nonumber \\
&\mathop{\le}^{(b)}
\frac{l}{e^{G^{1/6}}}
\Expect\!\left[
e^{\left(2\bar f(\y_{\tau_{G}\wedge T})\right)^{{1}/{6}}}
\right],
\end{align}
where (a) follows from Markov's inequality. To obtain (b), let
\[
A_0:=\frac{L\beta_1^2dD_{\beta_1}^2}{(1-\beta_1)^2},
\qquad
l:=e^{A_0}.
\]
By Lemma~\ref{property_2}, Lemma~\ref{property_2.54}, and \(\eta\le1\),
\[
\bar f(\x_{\tau_G\wedge T})
\le
2\bar f(\y_{\tau_G\wedge T})+A_0.
\]
Since \(\bar f(\y_{\tau_G\wedge T})\ge1\), concavity of \(z^{1/6}\) on
\([0,\infty)\) gives
\[
\left(2\bar f(\y_{\tau_G\wedge T})+A_0\right)^{1/6}
-
\left(2\bar f(\y_{\tau_G\wedge T})\right)^{1/6}
\le
\frac{A_0}{6(2\bar f(\y_{\tau_G\wedge T}))^{5/6}}
\le
A_0,
\]
and hence
\[
e^{\bar f^{1/6}(\x_{\tau_G\wedge T})}
\le
l\,e^{(2\bar f(\y_{\tau_G\wedge T}))^{1/6}}.
\]

Therefore, our objective reduces to estimating $\Expect\!\left[e^{(2\bar f(\y_{\tau_{G}\wedge T}))^{1/6}}\right]$. Since \(\bar f(\y_{\tau_G\wedge T})\le W_{\tau_G\wedge T}\), it suffices to control \(\Expect[e^{(2W_{\tau_G\wedge T})^{1/6}}]\). By applying the \emph{Taylor's} expansion \(e^{X}=\sum_{p=0}^{+\infty}X^p/p!\), it suffices to estimate, for any \(p\ge 12\),  
the \((p/6)\)-th moment of \(2W_{\tau_G\wedge T}\).

\noindent Next, based on the descent lemma, we proceed to the corresponding estimation. Summing \cref{descent_ineq} over $t$ from $1$ to $\tau_{G} \wedge T-1$, 
we obtain the following inequality:
\begin{align*}
\bar f(\mathbf{y}_{\tau_{G} \wedge T })&
+\sum_{i=1}^{d}\gamma_{\tau_{G}\wedge T-1,i}(\nabla f(\x_{\tau_{G}\wedge T}))_i^2+ \frac{1}{8}\sum_{t=1}^{\tau_{G} \wedge T-1}\sum_{i=1}^{d}\gamma_{t-1,i}(\nabla f(\mathbf{x}_t))_i^2
\notag\\&\le
\bar f(\mathbf{y}_1)
+ \sum_{t=1}^{\tau_{G} \wedge T-1} P_t+\sum_{t=1}^{\tau_{G}\wedge T-1}D_t\notag\\&\le\bar f(\mathbf{y}_1)
+ \sum_{t=1}^{\tau_{G} \wedge T-1} P_t+\left|\sum_{t=1}^{\tau_{G}\wedge T-1}D_t\right|.
\end{align*}
% In the expression above, for convenience, we introduce the following constant~$G$:
% \begin{align}\label{G}
% G := \log^{6}\left( \frac{1}{\gamma} \right).
% \end{align}
Note that, since \(\tau_{G}\) is a predictable time, and
the quantity \(\tau_{G}\wedge T - 1\) is a stopping time.  
\textit{We emphasize that, from this point onward, whenever stopping-time arguments or inequalities are invoked, 
we will no longer restate this fact explicitly.} 

% \noindent In step~($*$) above, we note that $\bar f(\mathbf{y}_t) < G$ for all $t \in [1, \tau_{G} \wedge T-1]$.

\noindent A rearrangement of the left-hand side yields the following concise inequality:
\begin{align*}
\bar f(\mathbf{y}_{\tau_G\wedge T})
&+
\frac{1}{8}
\sum_{t=1}^{\tau_G\wedge T}
\sum_{i=1}^{d}
\gamma_{t-1,i}
\bigl(\nabla f(\mathbf{x}_t)\bigr)_i^2 \\
&\le
\bar f(\mathbf{y}_1)
+
\sum_{i=1}^{d}
\gamma_{0,i}
\bigl(\nabla f(\mathbf{x}_1)\bigr)_i^2
+
\sum_{t=1}^{\tau_G\wedge T-1}P_t
+
\left|
\sum_{t=1}^{\tau_G\wedge T-1}D_t
\right|.
\end{align*}
Here the extra term
\[
\sum_{i=1}^{d}
\gamma_{0,i}
\bigl(\nabla f(\mathbf{x}_1)\bigr)_i^2
\]
comes from the initial endpoint of the telescoping sum. Since \(\mathbf{y}_1=\mathbf{x}_1\) and
\[
\gamma_0=\gamma(\sqrt v+\epsilon)^{-1}\mathbf 1,
\]
the \(L\)-smoothness lower-bound estimate (Lemma \ref{loss_bound}) gives
\[
\|\nabla f(\mathbf{x}_1)\|^2
=
\|\nabla f(\mathbf{y}_1)\|^2
\le
2L\bar f(\mathbf{y}_1).
\]
Therefore,
\[
\sum_{i=1}^{d}
\gamma_{0,i}
\bigl(\nabla f(\mathbf{x}_1)\bigr)_i^2
\le
\frac{2L\gamma}{\sqrt v+\epsilon}\bar f(\mathbf{y}_1).
\]
Recall the stepsize condition
\[
\frac{1}{\eta}
\ge
\max\!\left\{
\frac{8d\epsilon}{v^{3/2}}+\frac{8d}{\sqrt v},
\;
\frac{D_{\beta_1}\beta_1\sqrt{dL}}{1-\beta_1},
\;
1
\right\}.
\]
In particular, \(\eta\le1\). Since \(\gamma=\eta/\sqrt T\) and \(T\ge10\), we have
\(\gamma\le1\). Hence
\[
\sum_{i=1}^{d}
\gamma_{0,i}
\bigl(\nabla f(\mathbf{x}_1)\bigr)_i^2
\le
\frac{2L}{\sqrt v+\epsilon}\bar f(\mathbf{y}_1).
\]
Define the initial quantity
\[
B_0
:=
2
\left(
1+\frac{2L}{\sqrt v+\epsilon}
\right)
\bar f(\mathbf{y}_1).
\]
Then
\begin{align*}
\bar f(\mathbf{y}_{\tau_G\wedge T})
+
\frac{1}{8}
\sum_{t=1}^{\tau_G\wedge T}
\sum_{i=1}^{d}
\gamma_{t-1,i}
\bigl(\nabla f(\mathbf{x}_t)\bigr)_i^2
&\le
B_0
+
\sum_{t=1}^{\tau_G\wedge T-1}P_t
+
\left|
\sum_{t=1}^{\tau_G\wedge T-1}D_t
\right|.
\end{align*}
We then apply Lemma~\ref{lem:vt_comparable} to replace
\(\gamma_{t-1,i}\) in the summand on the left-hand side by
\(\frac12\gamma_{t,i}\), yielding
\begin{align}\label{first_hand}
W_{\tau_G\wedge T}:=
\bar f(\mathbf{y}_{\tau_G\wedge T})
+
\frac{1}{16}
\overbrace{
\sum_{t=1}^{\tau_G\wedge T}
\sum_{i=1}^{d}
\gamma_{t,i}
\bigl(\nabla f(\mathbf{x}_t)\bigr)_i^2
}^{\mathcal E_{\tau_G\wedge T}}
&\le
B_0
+
\sum_{t=1}^{\tau_G\wedge T-1}P_t
+
\left|
\sum_{t=1}^{\tau_G\wedge T-1}D_t
\right|.
\end{align}
%We refer to the left-hand side as \(W_{\tau_G\wedge T}\). Then, we aim is to bound the \(\frac{p}{6}\)-th moment of \(W_{\tau_G\wedge T}\).
The above equation defines the adaptive gradient energy $\mathcal E_{\tau_G\wedge T}$, which is the main quantity we bound later and which ultimately determines the convergence rate.

%\noindent\textbf{Step 3 (Bounding the \((p/6)\)-th Moment of \(W_{\tau_G\wedge T}\) via \(S_{\tau_G\wedge T-1}\)).}

\noindent\textbf{Step 3 (All-Order Bounds on Stopped Martingale and Residual Terms).}

\noindent Before carrying out the higher-order estimates, we first record two pathwise bounds that will be used repeatedly throughout Step~3. The first one is the logarithmic self-normalization estimate induced by Adam's second-moment accumulator. It controls both the adaptive stochastic-gradient energy and the adaptive momentum energy. The second one is a deterministic summability bound for the positive stepsize-difference term \(\Delta_{t,1,i}\). Together, these two lemmas provide the pathwise ingredients needed to control the martingale and residual terms below.

\begin{lem}[Formal Statement of Lemma~\ref{lemma_sum_two'}]\label{lemma_sum_two}
Consider Adam described in Algorithm~\ref{alg:adam}, with the parameter settings identical to those in Lemma~\ref{descent_lemma}. Then, for any stopping time \(\mu\), the following inequalities hold:
\begin{align*}
\sum_{t=1}^{\mu\wedge T}
\sum_{i=1}^{d}
\gamma_{t,i}^2g_{t,i}^2
&\le
4\eta^2d
\log\left(
1+\frac{S_{\mu\wedge T}}{vdT}
\right),
\\
\sum_{t=1}^{\mu\wedge T}
\|\gamma_t\odot m_t\|^2
&\le
16\eta^2d
\log\left(
1+\frac{S_{\mu\wedge T}}{vdT}
\right),
\end{align*}
where \(S_t:=vd+\sum_{s=1}^{t}\|g_s\|^2,\ t\ge 1.\) 
%This lemma is the rigorous version of Lemma~\ref{lemma_sum_two'} stated in the main text.
\end{lem}
\begin{proof}[Proof of Lemma~\ref{lemma_sum_two}] 
See Appendix~\ref{p_lemma_sum_two}.
\end{proof}

\noindent The next lemma controls the cumulative positive part of the adaptive
preconditioning-vector variation. This term appears in the residual estimates and also in the fluctuation term \(D_{t,2}\), so it is useful to isolate its purely pathwise summability before estimating the martingale and predictable residual
contributions.

\begin{lem}\label{lemma_s_delta}
Consider Adam described in Algorithm~\ref{alg:adam}, with the parameter settings identical to those in Lemma~\ref{descent_lemma}, and let \(\Delta_{t,1,i}\) be defined as in \cref{Delta_t1_def}. Then
\[
\sum_{t=1}^{T}
\sum_{i=1}^{d}
|\Delta_{t,1,i}|
\le
\left(
\frac{8d\epsilon}{v^{3/2}}
+
\frac{8d}{\sqrt v}
\right)\eta.
\]
\end{lem}
\begin{proof}[Proof of Lemma~\ref{lemma_s_delta}] 
See Appendix~\ref{p_lemma_s_delta}.
\end{proof}

\noindent We now proceed to develop the all-order bound. 

\noindent For any \(p\ge 12\), we take the \((p/6)\)-th moment on both sides of \cref{first_hand}. 
Since \(p/6\ge 1\), the elementary inequality
\[
(a+b+c)^{p/6}
\le
3^{\frac p6-1}
\left(
a^{p/6}+b^{p/6}+c^{p/6}
\right),
\qquad a,b,c\ge0,
\]
gives
\begin{align}\label{W_moment_initial_split}
\Expect\!\left[
W_{\tau_G\wedge T}^{\frac{p}{6}}
\right]
&\le
3^{\frac p6-1}
B_0^{\frac{p}{6}}
+
3^{\frac p6-1}
\Expect\!\left[
\left(
\sum_{t=1}^{\tau_G\wedge T-1} P_t
\right)^{\frac{p}{6}}
\right]
+
3^{\frac p6-1}
\Expect\!\left[
\left|
\sum_{t=1}^{\tau_G\wedge T-1}D_t
\right|^{\frac{p}{6}}
\right].
\end{align}
Thus it remains to bound the martingale term and the residual term separately. 
The following substeps isolate the three martingale components \(D_{t,1},D_{t,2},D_{t,3}\), and then the adaptive residuals in \(P_t\). 
These steps are important because each component must carry at most a \(G^{p/12}\)-order factor before the Young-inequality step.

\textbf{Substep 3.1 (Bounding \(\mathbb{E}\!\left[\left|\sum_{t=1}^{\tau_G\wedge T-1}D_t\right|^{p/6}\right]\)).} 
%in terms of \(S_{\tau_G\wedge T-1}\), as defined in \cref{S_T_T}).}

%We first record the BDG inequality used below.

We apply the Burkholder--Davis--Gundy (BDG) inequality (presented in Lemma~\ref{lem:burkholder}) to the martingale generated by the cumulative
sum of \(\{D_t\}_{t\ge1}\). Define
\[
M_0:=0,
\qquad
M_n:=\sum_{t=1}^{n}D_t,
\qquad n\ge1.
\]
Since \(\{(D_t,\mathscr F_t)\}_{t\ge1}\) is a martingale difference sequence,
\(\{(M_n,\mathscr F_n)\}_{n\ge0}\) is a martingale. Moreover, since
\(\tau_G\) is predictable, \(\widehat\tau_G:=\tau_G-1\) is a stopping time with
respect to \(\{\mathscr F_t\}_{t\ge0}\). Hence
\[
M_{\widehat\tau_G\wedge T}
=
\sum_{t=1}^{\tau_G\wedge T-1}D_t .
\]
Applying Lemma~\ref{lem:burkholder} with \(q=p/6\) and deterministic horizon
\(T\), we obtain
\begin{align}\label{BDG_D_split}
\Expect\!\left[
\left|
\sum_{t=1}^{\tau_G\wedge T-1}D_t
\right|^{\frac p6}
\right]
&=
\Expect\!\left[
\left|
M_{\widehat\tau_G\wedge T}
\right|^{\frac p6}
\right]
\notag\\
&\le
\Expect\!\left[
\left(
M_{\widehat\tau_G\wedge T}^{*}
\right)^{\frac p6}
\right]
\notag\\
&\mathop{\le}^{(a)}
\lambda_1
\left(\frac p6\right)^{\frac p6}
\Expect\!\left[
\left(
\sum_{t=1}^{\tau_G\wedge T-1}D_t^2
\right)^{\frac p{12}}
\right]
\notag\\
&\mathop{\le}^{(b)}
\lambda_1
3^{\frac p6-1}
\left(\frac p6\right)^{\frac p6}
\Bigg\{
\Expect\!\left[
\left(
\sum_{t=1}^{\tau_G\wedge T-1}D_{t,1}^2
\right)^{\frac p{12}}
\right]
\notag\\
&\hspace{5.5em}
+
\Expect\!\left[
\left(
\sum_{t=1}^{\tau_G\wedge T-1}D_{t,2}^2
\right)^{\frac p{12}}
\right]
\notag\\
&\hspace{5.5em}
+
\Expect\!\left[
\left(
\sum_{t=1}^{\tau_G\wedge T-1}D_{t,3}^2
\right)^{\frac p{12}}
\right]
\Bigg\}.
\end{align}
Here \((a)\) follows from Lemma~\ref{lem:burkholder}. For \((b)\), since
\(D_t=D_{t,1}+D_{t,2}+D_{t,3}\), we have
\[
D_t^2
\le
3\bigl(D_{t,1}^2+D_{t,2}^2+D_{t,3}^2\bigr).
\]
Therefore,
\[
\sum_{t=1}^{\tau_G\wedge T-1}D_t^2
\le
3\sum_{j=1}^{3}
\sum_{t=1}^{\tau_G\wedge T-1}D_{t,j}^2.
\]
Since \(p\ge12\), we have \(p/12\ge1\). Applying
\[
(A_1+A_2+A_3)^{\frac p{12}}
\le
3^{\frac p{12}-1}
\left(
A_1^{\frac p{12}}
+
A_2^{\frac p{12}}
+
A_3^{\frac p{12}}
\right),
\qquad A_1,A_2,A_3\ge0,
\]
gives the factor \(3^{p/6-1}\) in \cref{BDG_D_split}. Thus the
\((p/6)\)-th moment of the accumulated martingale error is reduced to
\((p/12)\)-th moments of the corresponding quadratic-variation components.

\noindent\textbf{Bounding the \(D_{t,1}\) term.}
For \(D_{t,1}\), by its definition in \cref{defination_1}, we have
\begin{align}\label{sup}
&
\Expect\!\left[
\left(
\sum_{t=1}^{\tau_G\wedge T-1}D_{t,1}^2
\right)^{\frac p{12}}
\right]
\notag\\
&\le
\Expect\!\left[
\left(
2d
\sum_{t=1}^{\tau_G\wedge T-1}
\left(
\sum_{i=1}^{d}
\gamma_{t,i}^2(\nabla f(\x_t))_i^2g_{t,i}^2
+
\sum_{i=1}^{d}
\Expect\!\left[
\gamma_{t,i}^2(\nabla f(\x_t))_i^2g_{t,i}^2
\,\middle|\,
\mathscr F_{t-1}
\right]
\right)
\right)^{\frac p{12}}
\right].
\end{align}
On the stopped trajectory, for every \(1\le t\le \tau_G\wedge T-1\),
\[
(\nabla f(\x_t))_i^2
\le
\|\nabla f(\x_t)\|^2
\le
2L\bar f(\x_t)
\le
2LG,
\qquad i=1,\ldots,d,
\]
where the second inequality follows from Lemma~\ref{loss_bound}. Hence the
gradient factors in \cref{sup} can be replaced by the stopped scale \(2LG\),
which gives
\begin{align}
&
\Expect\!\left[
\left(
\sum_{t=1}^{\tau_G\wedge T-1}D_{t,1}^2
\right)^{\frac p{12}}
\right]
\notag\\
&\le
(4LGd)^{\frac p{12}}
\Expect\!\left[
\left(
\sum_{t=1}^{\tau_G\wedge T-1}
\sum_{i=1}^{d}
\gamma_{t,i}^2g_{t,i}^2
+
\sum_{t=1}^{\tau_G\wedge T-1}
\sum_{i=1}^{d}
\Expect\!\left[
\gamma_{t,i}^2g_{t,i}^2
\,\middle|\,
\mathscr F_{t-1}
\right]
\right)^{\frac p{12}}
\right]
\notag\\
&\le
\frac12
(8LGd)^{\frac p{12}}
\Bigg(
\Expect\!\left[
\left(
\sum_{t=1}^{\tau_G\wedge T-1}
\sum_{i=1}^{d}
\gamma_{t,i}^2g_{t,i}^2
\right)^{\frac p{12}}
\right]
\notag\\
&\hspace{9em}
+
\Expect\!\left[
\left(
\sum_{t=1}^{\tau_G\wedge T-1}
\sum_{i=1}^{d}
\Expect\!\left[
\gamma_{t,i}^2g_{t,i}^2
\,\middle|\,
\mathscr F_{t-1}
\right]
\right)^{\frac p{12}}
\right]
\Bigg),
\label{sup_reduced}
\end{align}
where the last step uses
\[
(a+b)^m\le 2^{m-1}(a^m+b^m),
\qquad m=\frac p{12}\ge1.
\]

\noindent We now clarify the stopping-time convention needed to apply
Lemma~\ref{lemma_sum_two} to the first expectation in \cref{sup_reduced}.
Since \(\tau_G\) is predictable, \(\widehat\tau_G:=\tau_G-1\) is a stopping
time. We interpret
\[
\tau_G\wedge T-1
=
(\tau_G\wedge T)-1
=
\widehat\tau_G\wedge (T-1).
\]
Thus, by taking
\[
\mu:=\widehat\tau_G\wedge (T-1),
\]
we have \(\mu\le T-1<T\) and hence \(\mu\wedge T=\mu=\tau_G\wedge T-1\).
This convention will be used in the subsequent applications of
Lemma~\ref{lemma_sum_two} without being repeated each time.
\noindent Lemma~\ref{lemma_sum_two} therefore gives
\begin{align}\label{superadam_1}
\mathbb{E}\!\left[
\left(
\sum_{t=1}^{\tau_G\wedge T-1}
\sum_{i=1}^{d}
\gamma_{t,i}^2g_{t,i}^2
\right)^{\frac p{12}}
\right]
\le
(4\eta^2d)^{\frac p{12}}
\Expect\!\left[
\log^{\frac p{12}}
\left(
1+\frac{S_{\tau_G\wedge T-1}}{vdT}
\right)
\right],
\end{align}
where \(S_{\tau_G\wedge T-1}\) is defined in \cref{S_T_T}.

For the conditional-expectation term, we use the following reduction lemma.
\begin{lem}\label{lem_important_-1}
Let \(\{\mathcal F_n\}_{n\ge1}\) be a filtration, and let \(\{(Z_n,\mathcal F_n)\}_{n\ge1}\) be a nonnegative adapted process. Define
\[
X_n:=\sum_{k=1}^{n}Z_k,
\qquad
Y_n:=\sum_{k=1}^{n}\mathbb E[Z_k\mid \mathcal F_{k-1}].
\]
Then, for any stopping time \(\mu\), any deterministic index \(n<+\infty\), and any \(s\ge1\),
\[
\|Y_{\mu\wedge n}\|_{L^s}
\le
s\|X_{\mu\wedge n}\|_{L^s}.
\]
\end{lem}
\begin{proof}[Proof of Lemma \ref{lem_important_-1}]
This lemma was stated in \citet[Lemma B.3]{jin2024comprehensive}. For completeness, we provide a self-contained proof in Appendix~\ref{p_lem_important_-1}.
\end{proof}

Applying Lemma~\ref{lem_important_-1} with
\[
Z_t=\sum_{i=1}^{d}\gamma_{t,i}^2g_{t,i}^2,
\qquad
s=\frac p{12},
\]
and then applying Lemma~\ref{lemma_sum_two}, yield
\begin{align}\label{conditional_gamma_g_bound}
&
\mathbb{E}\!\left[
\left(
\sum_{t=1}^{\tau_G\wedge T-1}
\sum_{i=1}^{d}
\mathbb{E}\!\left[
\gamma_{t,i}^2g_{t,i}^2
\mid
\mathscr F_{t-1}
\right]
\right)^{\frac p{12}}
\right]
\notag\\
&\qquad\le
\left(\frac p{12}\right)^{\frac p{12}}
(4\eta^2d)^{\frac p{12}}
\Expect\!\left[
\log^{\frac p{12}}
\left(
1+\frac{S_{\tau_G\wedge T-1}}{vdT}
\right)
\right].
\end{align}
Combining \cref{sup}, \cref{superadam_1}, and \cref{conditional_gamma_g_bound}, and using \(\eta\le1\), yield
\begin{align}\label{D_t1_moment_bound}
&
\Expect\!\left[
\left(
\sum_{t=1}^{\tau_G\wedge T-1}
D_{t,1}^2
\right)^{\frac p{12}}
\right]
\notag\\
&\qquad\le
\frac12
(32Ld^2)^{\frac p{12}}
G^{\frac p{12}}
\left(
1+\left(\frac p{12}\right)^{\frac p{12}}
\right)
\Expect\!\left[
\log^{\frac p{12}}
\left(
1+\frac{S_{\tau_G\wedge T-1}}{vdT}
\right)
\right].
\end{align}

\noindent\textbf{Bounding the \(D_{t,2}\) term.}
For \(D_{t,2}\), by its definition in \cref{defination_1}, we have
\begin{align}\label{sup_1}
&
\Expect\!\left[
\left(
\sum_{t=1}^{\tau_G\wedge T-1}
D_{t,2}^2
\right)^{\frac p{12}}
\right]
\notag\\
&\qquad\le
(4LGd)^{\frac p{12}}
\Expect\!\left[
\left(
\sum_{t=1}^{\tau_G\wedge T-1}
\sum_{i=1}^{d}
|\Delta_{t,1,i}|^2
+
\sum_{t=1}^{\tau_G\wedge T-1}
\sum_{i=1}^{d}
\Expect\!\left[
|\Delta_{t,1,i}|^2
\,\middle|\,
\mathscr F_{t-1}
\right]
\right)^{\frac p{12}}
\right]
\notag\\
&\qquad\le
\frac12
(8LGd)^{\frac p{12}}
\Bigg(
\Expect\!\left[
\left(
\sum_{t=1}^{\tau_G\wedge T-1}
\sum_{i=1}^{d}
|\Delta_{t,1,i}|^2
\right)^{\frac p{12}}
\right]
\notag\\
&\hspace{7em}
+
\Expect\!\left[
\left(
\sum_{t=1}^{\tau_G\wedge T-1}
\sum_{i=1}^{d}
\Expect\!\left[
|\Delta_{t,1,i}|^2
\,\middle|\,
\mathscr F_{t-1}
\right]
\right)^{\frac p{12}}
\right]
\Bigg).
\end{align}
Here we use the same stopped-gradient bound as in the \(D_{t,1}\)-estimate:
for every \(1\le t\le \tau_G\wedge T-1\),
\[
(\nabla f(\x_t))_i^2
\le
\|\nabla f(\x_t)\|^2
\le
2L\bar f(\x_t)
\le
2LG,
\qquad i=1,\ldots,d.
\]

We now control the two terms on the right-hand side of \cref{sup_1}. Since the
summation is over a stopped subinterval, Lemma~\ref{lemma_s_delta} gives
\[
\sum_{t=1}^{\tau_G\wedge T-1}
\sum_{i=1}^{d}
|\Delta_{t,1,i}|
\le
\sum_{t=1}^{T}
\sum_{i=1}^{d}
|\Delta_{t,1,i}|
\le
\left(
\frac{8d\epsilon}{v^{3/2}}
+
\frac{8d}{\sqrt v}
\right)\eta.
\]
By the stepsize condition in \cref{eta},
\[
\left(
\frac{8d\epsilon}{v^{3/2}}
+
\frac{8d}{\sqrt v}
\right)\eta
\le
1.
\]
Therefore,
\[
\sum_{t=1}^{\tau_G\wedge T-1}
\sum_{i=1}^{d}
|\Delta_{t,1,i}|
\le 1.
\]
Since all summands are nonnegative, this also implies
\[
\sum_{t=1}^{\tau_G\wedge T-1}
\sum_{i=1}^{d}
|\Delta_{t,1,i}|^2
\le
\left(
\sum_{t=1}^{\tau_G\wedge T-1}
\sum_{i=1}^{d}
|\Delta_{t,1,i}|
\right)^2
\le
1.
\]
Hence
\begin{align}\label{poinner}
\mathbb{E}\!\left[
\left(
\sum_{t=1}^{\tau_G\wedge T-1}
\sum_{i=1}^{d}
|\Delta_{t,1,i}|^2
\right)^{\frac p{12}}
\right]
\le
1.
\end{align}

For the conditional-expectation term, we apply Lemma~\ref{lem_important_-1} to
the nonnegative adapted array
\(
X_{t,i}:=|\Delta_{t,1,i}|^2.
\)
This gives
\begin{align}\label{ssuupp_adam_0}
&\mathbb{E}\!\left[
\left(
\sum_{t=1}^{\tau_G\wedge T-1}
\sum_{i=1}^{d}
\mathbb{E}\!\left[
|\Delta_{t,1,i}|^2
\,\middle|\,
\mathscr F_{t-1}
\right]
\right)^{\frac p{12}}
\right]
\notag\\
&\qquad\le
\left(\frac p{12}\right)^{\frac p{12}}
\mathbb{E}\!\left[
\left(
\sum_{t=1}^{\tau_G\wedge T-1}
\sum_{i=1}^{d}
|\Delta_{t,1,i}|^2
\right)^{\frac p{12}}
\right]
\notag\\
&\qquad\le
\left(\frac p{12}\right)^{\frac p{12}}.
\end{align}
Combining \cref{sup_1}, \cref{poinner}, and \cref{ssuupp_adam_0}, we obtain
\begin{align}\label{D_t2_moment_bound}
\Expect\!\left[
\left(
\sum_{t=1}^{\tau_G\wedge T-1}
D_{t,2}^2
\right)^{\frac p{12}}
\right]
\le
\frac12
(8LGd)^{\frac p{12}}
\left(
1+
\left(\frac p{12}\right)^{\frac p{12}}
\right).
\end{align}

\noindent\textbf{Bounding the \(D_{t,3}\) term.}
This is the adaptive-stepsize martingale fluctuation. We keep it separate because it is the most delicate component.

We define
\[
H_t
:=
\frac{\beta_1}{1-\beta_1}
\sum_{i=1}^{d}
(\gamma_{t-1,i}-\gamma_{t,i})
(\nabla f(\y_t))_i
m_{t-1,i}.
\]
By the definition of \(D_{t,3}\), we have \(D_{t,3}=H_t-\Expect[H_t\mid\mathscr F_{t-1}]\). Hence
\[
D_{t,3}^2
\le
2H_t^2
+
2\Expect[H_t^2\mid\mathscr F_{t-1}].
\]
By \cref{W_1},
\[
\gamma_{t-1,i}-\gamma_{t,i}
=
\Delta_{t,1,i}+\Delta_{t,2,i}.
\]
Moreover,
\begin{align}
&0\le
 \Delta_{t,1,i}
\le
\frac{\eta}{\sqrt{T-1}}
\frac{1}{\sqrt{v_{t-1,i}}+\epsilon}
\le
2\gamma_{t-1,i}, \nonumber \\
&|\Delta_{t,2,i}|
\le
\frac1T\gamma_{t-1,i}
\le
\gamma_{t-1,i}.
\end{align}
Thus,
\[
|\gamma_{t-1,i}-\gamma_{t,i}|
\le
3\gamma_{t-1,i}.
\]
By applying 
\[\left(\sum_{i=1}^{d}a_i\right)^2\le d\sum_{i=1}^{d}a_i^2,\] 
we obtain
\begin{align}
H_t^2
&\le
9d
\left(\frac{\beta_1}{1-\beta_1}\right)^2
\sum_{i=1}^{d}
\gamma_{t-1,i}^2
(\nabla f(\y_t))_i^2
m_{t-1,i}^2
\notag\\
&\le
9d
\left(\frac{\beta_1}{1-\beta_1}\right)^2
\|\nabla f(\y_t)\|^2
\|\gamma_{t-1}\odot m_{t-1}\|^2.
\label{Dt3_inner_bound}
\end{align}
On the stopped trajectory, Lemma~\ref{loss_bound} gives
\[
\|\nabla f(\x_t)\|^2
\le
2L\bar f(\x_t)
\le
2LG.
\]
Furthermore, by \(L\)-smoothness and the definition of \(\y_t\),
\begin{align}
\|\nabla f(\y_t)\|^2
&\le
2\|\nabla f(\x_t)\|^2
+
2L^2\|\y_t-\x_t\|^2
\notag\\
&\le
4LG
+
2L^2
\left(\frac{\beta_1}{1-\beta_1}\right)^2
\|\gamma_{t-1}\odot m_{t-1}\|^2
\notag\\
&\le
4LG
+
2L^2
\left(\frac{\beta_1}{1-\beta_1}\right)^2
d\eta^2D_{\beta_1}^2,
\label{grad_y_stopped_bound}
\end{align}
where the last inequality uses
\[
\|\gamma_{t-1}\odot m_{t-1}\|^2
\le
d\eta^2D_{\beta_1}^2.
\]
Combining the preceding bounds yields the pointwise estimate
\begin{align}
D_{t,3}^2
&\le
36d
\left(\frac{\beta_1}{1-\beta_1}\right)^2
\left[
4LG
+
2L^2
\left(\frac{\beta_1}{1-\beta_1}\right)^2
d\eta^2D_{\beta_1}^2
\right]
\|\gamma_{t-1}\odot m_{t-1}\|^2.
\label{D_t3_square_bound}
\end{align}
By applying Lemma~\ref{lemma_sum_two} and using the index shift \(m_0=0\), we have
\[
\sum_{t=1}^{\tau_G\wedge T-1}
\|\gamma_{t-1}\odot m_{t-1}\|^2
\le
16\eta^2d
\log\left(
1+\frac{S_{\tau_G\wedge T-1}}{vdT}
\right).
\]
Therefore,
\begin{align}\label{D_t3_moment_bound}
&
\Expect\!\left[
\left(
\sum_{t=1}^{\tau_G\wedge T-1}
D_{t,3}^2
\right)^{\frac{p}{12}}
\right]
\notag\\
&\le
\left[
576d^2\eta^2
\left(\frac{\beta_1}{1-\beta_1}\right)^2
\left(
4LG
+
2L^2
\left(\frac{\beta_1}{1-\beta_1}\right)^2
d\eta^2D_{\beta_1}^2
\right)
\right]^{\frac{p}{12}}
\Expect\!\left[
\log^{\frac{p}{12}}
\left(
1+\frac{S_{\tau_G\wedge T-1}}{vdT}
\right)
\right].
\end{align}

\noindent\textbf{Combining the three martingale components.}
Combining \cref{BDG_D_split}, \cref{D_t1_moment_bound}, \cref{D_t2_moment_bound}, and \cref{D_t3_moment_bound}, we obtain
\begin{align}\label{adam_adam_5}
&
\mathbb{E}\!\left[
\left|
\sum_{t=1}^{\tau_G\wedge T-1}D_t
\right|^{\frac{p}{6}}
\right]
\notag\\
&\le
\lambda_1
3^{\frac{p}{12}}
\left(\frac{p}{6}\right)^{\frac{p}{6}}
\Bigg[
\frac{1}{2}
(32Ld^2)^{\frac{p}{12}}
G^{\frac{p}{12}}
\left(
1+\left(\frac{p}{12}\right)^{\frac{p}{12}}
\right)
\Expect\!\left[
\log^{\frac{p}{12}}
\left(
1+\frac{S_{\tau_G\wedge T-1}}{vdT}
\right)
\right]
\notag\\
&\quad
+
\frac{1}{2}
(8LGd)^{\frac{p}{12}}
\left(
1+\left(\frac{p}{12}\right)^{\frac{p}{12}}
\right)
\notag\\
&\quad
+
\left[
576d^2\eta^2
\left(\frac{\beta_1}{1-\beta_1}\right)^2
\left(
4LG
+
2L^2
\left(\frac{\beta_1}{1-\beta_1}\right)^2
d\eta^2D_{\beta_1}^2
\right)
\right]^{\frac{p}{12}}
\Expect\!\left[
\log^{\frac{p}{12}}
\left(
1+\frac{S_{\tau_G\wedge T-1}}{vdT}
\right)
\right]
\Bigg].
\end{align}
In particular, since \(G\ge1\) and \(\eta\le1\), the \(D_{t,1}\), \(D_{t,2}\), and \(D_{t,3}\) contributions all carry at most a \(G^{p/12}\)-order factor before applying Young's inequality.

%\noindent We now turn to the predictable residual term.

\textbf{Substep 3.2 (Bounding \(\mathbb{E}[(\sum_{t=1}^{\tau_G\wedge T-1}P_t)^{p/6}]\) via \(S_{\tau_G\wedge T-1}\)).}

\noindent According to the definition of \(P_t\) in \cref{P_t_def}, we first derive a pointwise upper bound for \(P_t\) on the stopped trajectory. 
By \cref{W_1},
\[
\gamma_{t-1,i}-\gamma_{t,i}
=
\Delta_{t,1,i}+\Delta_{t,2,i}.
\]
Moreover, the estimates above imply
\[
|\gamma_{t-1,i}-\gamma_{t,i}|
\le
3\gamma_{t-1,i}.
\]
Also, on the stopped trajectory,
\[
|(\nabla f(\x_t))_i|
\le
\|\nabla f(\x_t)\|
\le
\sqrt{2LG}.
\]
Using \(T\ge 10\), \(\eta\le1\), and
\[
\frac{T}{T-1}\le2,
\qquad
\frac{\epsilon}{\sqrt T+\sqrt{T-1}}\le \epsilon,
\]
we obtain, for \(1\le t\le\tau_G\wedge T-1\),
\begin{align}\label{superadam_0}
P_t
&\le
\left(
\alpha_{1,0}
+
\frac{1}{\eta}\alpha_{1,1}(G)
\right)
\|\gamma_{t-1}\odot m_{t-1}\|^2
+
\alpha_2
\sum_{i=1}^{d}
\gamma_{t,i}^2g_{t,i}^2
\notag\\
&\quad
+
\alpha_3
\|\gamma_t\odot m_t\|^2
+
C\sum_{i=1}^{d}
\Expect\!\left[
|\Delta_{t,1,i}|
\,\middle|\,
\mathscr F_{t-1}
\right]
\notag\\
&\quad
+
\frac{\beta_1}{1-\beta_1}
\frac{1}{T}
\sum_{i=1}^{d}
\gamma_{t-1,i}
\left|(\nabla f(\y_t))_i m_{t-1,i}\right|,
\end{align}
where
\begin{align}\label{alpha_def}
\alpha_{1,0}
&:=
\frac{\beta_1^2L}{2(1-\beta_1)^2}
+
\frac{\beta_1^2L^2}{4(1-\beta_1)^2\sqrt v}
+
\frac{9\beta_1^2L}{(1-\beta_1)^2},
\notag\\
\alpha_{1,1}(G)
&:=
16
\frac{\beta_1(1+\beta_1)}{(1-\beta_1)^2}
\left(
\sqrt{2LG}
+
\sqrt C
+
\epsilon
\right),
\notag\\
\alpha_2
&:=
\frac{3L}{2},
\qquad
\alpha_3
:=
\frac{140dL^2}{\sqrt v}.
\end{align}
The factor \(1/\eta\) in front of \(\alpha_{1,1}(G)\) is kept explicitly at the pointwise level. It will be offset by the \(\eta^2\)-factor in the logarithmic momentum bound below.

\noindent By Lemma~\ref{lemma_sum_two} and the index shift \(m_0=0\), we have
\begin{align}\label{moment_path_bounds_for_P}
\sum_{t=1}^{\tau_G\wedge T-1}
\|\gamma_{t-1}\odot m_{t-1}\|^2
&\le
16\eta^2d
\log\left(
1+\frac{S_{\tau_G\wedge T-1}}{vdT}
\right),
\notag\\
\sum_{t=1}^{\tau_G\wedge T-1}
\|\gamma_t\odot m_t\|^2
&\le
16\eta^2d
\log\left(
1+\frac{S_{\tau_G\wedge T-1}}{vdT}
\right),
\notag\\
\sum_{t=1}^{\tau_G\wedge T-1}
\sum_{i=1}^{d}
\gamma_{t,i}^2g_{t,i}^2
&\le
4\eta^2d
\log\left(
1+\frac{S_{\tau_G\wedge T-1}}{vdT}
\right).
\end{align}
By Lemma~\ref{lemma_sum_two} and the index shift \(m_0=0\), the expression inside the expectation is bounded pathwise as follows:
\begin{align*}
&
\sum_{t=1}^{\tau_G\wedge T-1}
\left[
\left(
\alpha_{1,0}
+
\frac{1}{\eta}\alpha_{1,1}(G)
\right)
\|\gamma_{t-1}\odot m_{t-1}\|^2
+
\alpha_2
\sum_{i=1}^{d}
\gamma_{t,i}^2g_{t,i}^2
+
\alpha_3
\|\gamma_t\odot m_t\|^2
\right]
\\
&\qquad \le
\left[
16\eta^2d\alpha_{1,0}
+
16\eta d\alpha_{1,1}(G)
+
4\eta^2d\alpha_2
+
16\eta^2d\alpha_3
\right]
\log\left(
1+\frac{S_{\tau_G\wedge T-1}}{vdT}
\right).
\end{align*}
Indeed, the only slightly delicate term is the one containing \(1/\eta\), which can be bounded as follows:
\[
\frac{1}{\eta}\alpha_{1,1}(G)
\sum_{t=1}^{\tau_G\wedge T-1}
\|\gamma_{t-1}\odot m_{t-1}\|^2
\le
16\eta d\alpha_{1,1}(G)
\log\left(
1+\frac{S_{\tau_G\wedge T-1}}{vdT}
\right),
\]
where the logarithmic momentum bound contributes the factor \(16\eta^2d\).

Now recall the stepsize condition
\begin{align*}
\frac{1}{\eta}
\ge
\max\!\left\{
\frac{8d\epsilon}{v^{3/2}}+\frac{8d}{\sqrt v},
\;
\frac{D_{\beta_1}\beta_1\sqrt{dL}}{1-\beta_1},
\;
1
\right\}.
\end{align*}
In particular, \(0<\eta\le1\). Hence \(\eta^2\le1\) and \(\eta\le1\), and therefore
\begin{align*}
&16\eta^2d\alpha_{1,0}
+
16\eta d\alpha_{1,1}(G)
+
4\eta^2d\alpha_2
+
16\eta^2d\alpha_3
\\
&\qquad\le
16d\alpha_{1,0}
+
16d\alpha_{1,1}(G)
+
4d\alpha_2
+
16d\alpha_3 .
\end{align*}
Consequently,
\begin{align}\label{superadam_6}
&
\mathbb{E}\!\left[
\left(
\sum_{t=1}^{\tau_G\wedge T-1}
\left[
\left(
\alpha_{1,0}
+
\frac{1}{\eta}\alpha_{1,1}(G)
\right)
\|\gamma_{t-1}\odot m_{t-1}\|^2
+
\alpha_2
\sum_{i=1}^{d}
\gamma_{t,i}^2g_{t,i}^2
+
\alpha_3
\|\gamma_t\odot m_t\|^2
\right]
\right)^{\frac{p}{6}}
\right]
\notag\\
&\qquad \le
\left[
16d\alpha_{1,0}
+
16d\alpha_{1,1}(G)
+
4d\alpha_2
+
16d\alpha_3
\right]^{\frac{p}{6}}
\Expect\!\left[
\log^{\frac{p}{6}}
\left(
1+\frac{S_{\tau_G\wedge T-1}}{vdT}
\right)
\right].
\end{align}
Furthermore, since \(G\ge1\),
\[
\alpha_{1,1}(G)
\le
16
\frac{\beta_1(1+\beta_1)}{(1-\beta_1)^2}
\left(
\sqrt{2L}
+
\sqrt C
+
\epsilon
\right)
G^{1/2}.
\]
Hence
\begin{align}\label{P_main_coefficient_G_half}
&
16d\alpha_{1,0}
+
16d\alpha_{1,1}(G)
+
4d\alpha_2
+
16d\alpha_3
\notag\\
&\qquad\le
\Bigg[
16d
\left(
\frac{\beta_1^2L}{2(1-\beta_1)^2}
+
\frac{\beta_1^2L^2}{4(1-\beta_1)^2\sqrt v}
+
\frac{9\beta_1^2L}{(1-\beta_1)^2}
\right)
\notag\\
&\qquad
+
256d
\frac{\beta_1(1+\beta_1)}{(1-\beta_1)^2}
\left(
\sqrt{2L}
+
\sqrt C
+
\epsilon
\right)
+
4d\alpha_2
+
\frac{2240d^2L^2}{\sqrt v}
\Bigg]G^{1/2}.
\end{align}
Thus the main residual contribution in \(P_t\) carries at most a \(G^{p/12}\)-order factor after taking the \((p/6)\)-th power.

For the conditional \(\Delta_{t,1,i}\)-term, applying Lemma~\ref{lem_important_-1} with
\[
Z_t:=\sum_{i=1}^{d}|\Delta_{t,1,i}|
\]
and then applying Lemma~\ref{lemma_s_delta}, we obtain
\begin{align}\label{conditional_delta_p_bound}
&
\mathbb{E}\!\left[
\left(
\sum_{t=1}^{\tau_G\wedge T-1}
C\sum_{i=1}^{d}
\Expect\!\left[
|\Delta_{t,1,i}|
\,\middle|\,
\mathscr F_{t-1}
\right]
\right)^{\frac{p}{6}}
\right]
\notag\\
&\qquad \le
C^{\frac{p}{6}}
\left(\frac{p}{6}\right)^{\frac{p}{6}}
\mathbb{E}\!\left[
\left(
\sum_{t=1}^{\tau_G\wedge T-1}
\sum_{i=1}^{d}
|\Delta_{t,1,i}|
\right)^{\frac{p}{6}}
\right]
\le
C^{\frac{p}{6}}
\left(\frac{p}{6}\right)^{\frac{p}{6}}.
\end{align}

It remains to control the last adaptive-stepsize residual term in \cref{superadam_0}. 
By Cauchy--Schwarz inequality, we have
\begin{align*}
&
\sum_{t=1}^{\tau_G\wedge T-1}
\sum_{i=1}^{d}
\gamma_{t-1,i}
\left|
(\nabla f(\y_t))_i m_{t-1,i}
\right|
\\
&\qquad \le
\left(
\sum_{t=1}^{\tau_G\wedge T-1}
\|\nabla f(\y_t)\|^2
\right)^{1/2}
\left(
\sum_{t=1}^{\tau_G\wedge T-1}
\|\gamma_{t-1}\odot m_{t-1}\|^2
\right)^{1/2}.
\end{align*}
On the stopped trajectory,
\[
\|\nabla f(\y_t)\|^2
\le
4LG
+
2L^2
\left(\frac{\beta_1}{1-\beta_1}\right)^2
d\eta^2D_{\beta_1}^2.
\]
Together with \cref{moment_path_bounds_for_P}, this implies
\begin{align}\label{last_P_term_bound}
&
\sum_{t=1}^{\tau_G\wedge T-1}
\frac{\beta_1}{1-\beta_1}
\frac{1}{T}
\sum_{i=1}^{d}
\gamma_{t-1,i}
\left|
(\nabla f(\y_t))_i m_{t-1,i}
\right|
\notag\\
&\qquad \le
\frac{4\beta_1\eta\sqrt d}{(1-\beta_1)\sqrt T}
\left[
4LG
+
2L^2
\left(\frac{\beta_1}{1-\beta_1}\right)^2
d\eta^2D_{\beta_1}^2
\right]^{1/2}
\log^{1/2}\left(
1+\frac{S_{\tau_G\wedge T-1}}{vdT}
\right).
\end{align}

Using \((a+b+c)^m\le 3^m(a^m+b^m+c^m)\) for \(m\ge1\), and combining
\cref{superadam_6}, \cref{conditional_delta_p_bound}, and \cref{last_P_term_bound}, we obtain
\begin{align}\label{P_t_moment_bound_new}
&
\mathbb{E}\!\left[
\left(
\sum_{t=1}^{\tau_G\wedge T-1}P_t
\right)^{\frac{p}{6}}
\right]
\notag\\
&\le
3^{\frac{p}{6}}
\left[
16d\alpha_{1,0}
+
16d\alpha_{1,1}(G)
+
4d\alpha_2
+
16d\alpha_3
\right]^{\frac{p}{6}}
\Expect\!\left[
\log^{\frac{p}{6}}
\left(
1+\frac{S_{\tau_G\wedge T-1}}{vdT}
\right)
\right]
\notag\\
&\quad
+
3^{\frac{p}{6}}
C^{\frac{p}{6}}
\left(\frac{p}{6}\right)^{\frac{p}{6}}
\notag\\
&\quad
+
3^{\frac{p}{6}}
\left[
\frac{4\beta_1\eta\sqrt d}{(1-\beta_1)\sqrt T}
\left(
4LG
+
2L^2
\left(\frac{\beta_1}{1-\beta_1}\right)^2
d\eta^2D_{\beta_1}^2
\right)^{1/2}
\right]^{\frac{p}{6}}
\Expect\!\left[
\log^{\frac{p}{12}}
\left(
1+\frac{S_{\tau_G\wedge T-1}}{vdT}
\right)
\right].
\end{align}
By \cref{P_main_coefficient_G_half}, the first term on the right-hand side of \cref{P_t_moment_bound_new} contributes at most the order \(G^{p/12}\) before Young's inequality. The final adaptive correction term also contributes at most order \(G^{p/12}\), because the only \(G\)-dependent factor inside the bracket is a square-root term, and it is raised to the \(p/6\)-th power.

\textbf{Step 4 (All-Order and Stretched-Exponential Moment Bounds on $W_{\tau_G\wedge T}$).}

In this step, we first provide bounds on the all-order moment for $W_{\tau_G\wedge T}$ and then provide the stretched-exponential moment bound needed for \((2W_{\tau_G\wedge T})^{1/6}\).

\textbf{Substep 5.1: (Bounding All-Order Moments for $W_{\tau_G\wedge T}$).} 
%$\Expect\!\left[W_{\tau_G\wedge T}^{\frac{p}{6}}\right]$.}
Combining \cref{W_moment_initial_split}, \cref{adam_adam_5}, and
\cref{P_t_moment_bound_new}, we obtain
\begin{align}\label{W_bound_before_log}
&\Expect\!\left[
W_{\tau_G\wedge T}^{\frac{p}{6}}
\right] \nonumber \\
&\le
3^{\frac p6-1}
B_0^{\frac{p}{6}}
\notag\\
&\quad
+
3^{\frac p6-1}
\lambda_1
3^{\frac{p}{12}}
\left(\frac{p}{6}\right)^{\frac{p}{6}}
\Bigg[
\frac{1}{2}
(32Ld^2)^{\frac{p}{12}}
G^{\frac{p}{12}}
\left(
1+\left(\frac{p}{12}\right)^{\frac{p}{12}}
\right)
\times
\Expect\!\left[
\log^{\frac{p}{12}}
\left(
1+\frac{S_{\tau_G\wedge T-1}}{vdT}
\right)
\right]
\notag\\
&\quad
+
\frac{1}{2}
(8LGd)^{\frac{p}{12}}
\left(
1+\left(\frac{p}{12}\right)^{\frac{p}{12}}
\right)
\notag\\
&\quad
+
\left[
576d^2\eta^2
\left(\frac{\beta_1}{1-\beta_1}\right)^2
\left(
4LG
+
2L^2
\left(\frac{\beta_1}{1-\beta_1}\right)^2
d\eta^2D_{\beta_1}^2
\right)
\right]^{\frac{p}{12}}\notag\\
&\hspace{8em}
\times
\Expect\!\left[
\log^{\frac{p}{12}}
\left(
1+\frac{S_{\tau_G\wedge T-1}}{vdT}
\right)
\right]
\Bigg]
\notag\\
&\quad
+
3^{\frac p6-1}
3^{\frac{p}{6}}
\left[
16d\alpha_{1,0}
+
16d\alpha_{1,1}(G)
+
4d\alpha_2
+
16d\alpha_3
\right]^{\frac{p}{6}}
\times
\Expect\!\left[
\log^{\frac{p}{6}}
\left(
1+\frac{S_{\tau_G\wedge T-1}}{vdT}
\right)
\right]
\notag\\
&\quad
+
3^{\frac p6-1}
3^{\frac{p}{6}}
C^{\frac{p}{6}}
\left(\frac{p}{6}\right)^{\frac{p}{6}}
\notag\\
&\quad
+
3^{\frac p6-1}
3^{\frac{p}{6}}
\left[
\frac{4\beta_1\eta\sqrt d}{(1-\beta_1)\sqrt T}
\left(
4LG
+
2L^2
\left(\frac{\beta_1}{1-\beta_1}\right)^2
d\eta^2D_{\beta_1}^2
\right)^{1/2}
\right]^{\frac{p}{6}}
\notag\\
&\hspace{8em}\times
\Expect\!\left[
\log^{\frac{p}{12}}
\left(
1+\frac{S_{\tau_G\wedge T-1}}{vdT}
\right)
\right].
\end{align}

We next define
\begin{align}\label{M_T}
M_T
:=
\mathbb{E}\!\left[
1+\frac{S_{\tau_G\wedge T-1}}{vdT}
\right],
\end{align}
and apply the following inequalities
\[
\log^m(1+x)
\le
2^{2m}m^m(1+x)^{1/4},
\qquad x\ge0,\ m\ge1,
\]
and
\[
\Expect\!\left[|X|^{1/4}\right]
\le
\Expect\!\left[|X|\right]^{1/4},
\]
to \cref{W_bound_before_log}, and obtain
\begin{align}\label{W_bound_after_log}
&\Expect\!\left[
W_{\tau_G\wedge T}^{\frac{p}{6}}
\right] \nonumber \\
&\le
3^{\frac p6-1}B_0^{\frac{p}{6}}
\notag\\
&\quad
+
3^{\frac p6-1}
\lambda_1
3^{\frac{p}{12}}
\left(\frac{p}{6}\right)^{\frac{p}{6}}
\frac{1}{2}
(32Ld^2)^{\frac{p}{12}}
G^{\frac{p}{12}}
\left(
1+\left(\frac{p}{12}\right)^{\frac{p}{12}}
\right)
2^{\frac{p}{6}}
\left(\frac{p}{12}\right)^{\frac{p}{12}}
M_T^{1/4}
\notag\\
&\quad
+
3^{\frac p6-1}
\lambda_1
3^{\frac{p}{12}}
\left(\frac{p}{6}\right)^{\frac{p}{6}}
\frac{1}{2}
(8LGd)^{\frac{p}{12}}
\left(
1+\left(\frac{p}{12}\right)^{\frac{p}{12}}
\right)
\notag\\
&\quad
+
3^{\frac p6-1}
\lambda_1
3^{\frac{p}{12}}
\left(\frac{p}{6}\right)^{\frac{p}{6}}
\left[
576d^2\eta^2
\left(\frac{\beta_1}{1-\beta_1}\right)^2
\left(
4LG
+
2L^2
\left(\frac{\beta_1}{1-\beta_1}\right)^2
d\eta^2D_{\beta_1}^2
\right)
\right]^{\frac{p}{12}}
\notag\\
&\qquad\qquad\cdot
2^{\frac{p}{6}}
\left(\frac{p}{12}\right)^{\frac{p}{12}}
M_T^{1/4}
\notag\\
&\quad
+
3^{\frac p6-1}
3^{\frac{p}{6}}
\left[
16d\alpha_{1,0}
+
16d\alpha_{1,1}(G)
+
4d\alpha_2
+
16d\alpha_3
\right]^{\frac{p}{6}}
2^{\frac{p}{3}}
\left(\frac{p}{6}\right)^{\frac{p}{6}}
M_T^{1/4}
\notag\\
&\quad
+
3^{\frac p6-1}
3^{\frac{p}{6}}
C^{\frac{p}{6}}
\left(\frac{p}{6}\right)^{\frac{p}{6}}
\notag\\
&\quad
+
3^{\frac p6-1}
3^{\frac{p}{6}}
\left[
\frac{4\beta_1\eta\sqrt d}{(1-\beta_1)\sqrt T}
\left(
4LG
+
2L^2
\left(\frac{\beta_1}{1-\beta_1}\right)^2
d\eta^2D_{\beta_1}^2
\right)^{1/2}
\right]^{\frac{p}{6}}
\notag\\
&\qquad\qquad\cdot
2^{\frac{p}{6}}
\left(\frac{p}{12}\right)^{\frac{p}{12}}
M_T^{1/4}.
\end{align}

We next show that every \(G\)-dependent coefficient in \cref{W_bound_after_log} is of the order at most \(G^{p/12}\). More explicitly, since \(G\ge1\) and \(\eta\le1\), we have
\[
\left(
4LG
+
2L^2
\left(\frac{\beta_1}{1-\beta_1}\right)^2
d\eta^2D_{\beta_1}^2
\right)^{\frac{p}{12}}
\le
\left(
4L
+
2L^2
\left(\frac{\beta_1}{1-\beta_1}\right)^2
dD_{\beta_1}^2
\right)^{\frac{p}{12}}
G^{\frac{p}{12}},
\]
and
\[
\begin{aligned}
&\left[
\frac{4\beta_1\eta\sqrt d}{(1-\beta_1)\sqrt T}
\left(
4LG
+
2L^2
\left(\frac{\beta_1}{1-\beta_1}\right)^2
d\eta^2D_{\beta_1}^2
\right)^{1/2}
\right]^{\frac{p}{6}}
\\
&\le
\left[
\frac{4\beta_1\sqrt d}{1-\beta_1}
\left(
4L
+
2L^2
\left(\frac{\beta_1}{1-\beta_1}\right)^2
dD_{\beta_1}^2
\right)^{1/2}
\right]^{\frac{p}{6}}
G^{\frac{p}{12}} .
\end{aligned}
\]
Moreover, by \cref{P_main_coefficient_G_half},
\begin{align*}
&
\left[
16d\alpha_{1,0}
+
16d\alpha_{1,1}(G)
+
4d\alpha_2
+
16d\alpha_3
\right]^{\frac{p}{6}}
\\
&\qquad \le
\Bigg[
16d
\left(
\frac{\beta_1^2L}{2(1-\beta_1)^2}
+
\frac{\beta_1^2L^2}{4(1-\beta_1)^2\sqrt v}
+
\frac{9\beta_1^2L}{(1-\beta_1)^2}
\right)
\\
&\qquad \qquad
+
256d
\frac{\beta_1(1+\beta_1)}{(1-\beta_1)^2}
\left(
\sqrt{2L}
+
\sqrt C
+
\epsilon
\right)
+
4d\alpha_2
+
\frac{2240d^2L^2}{\sqrt v}
\Bigg]^{\frac{p}{6}}
G^{\frac{p}{12}} .
\end{align*}
Substituting the above bounds into \cref{W_bound_after_log} yields exactly five terms of order \(G^{p/12}\). 
Each of these terms takes either the form \(A G^{p/12}M_T^{1/4}\) or \(A G^{p/12}\), where \(A\ge0\) denotes a generic constant that may differ from term to term. 
We then apply Young's inequality
\begin{equation}
ab
\le
\frac{a^r}{r}
+
\frac{b^{r'}}{r'},
\qquad
a,b\ge0,\quad r,r'>1,\quad \frac1r+\frac1{r'}=1,
\end{equation}
to separate the \(G\)-dependent factor from its coefficient and raise the \(G\)-dependent order from \(G^{p/12}\) to \(G^{p/7}\). 
Specifically, after the factor \(2^{p/6}\) from
\((2W_{\tau_G\wedge T})^{p/6}=2^{p/6}W_{\tau_G\wedge T}^{p/6}\) is included, we choose
\[
r:=\frac{12}{7},
\qquad
r':=\frac{12}{5},
\qquad
\frac1r+\frac1{r'}=1.
\]

We first apply Young's inequality to the terms of the form \(2^{p/6}A G^{p/12}M_T^{1/4}\), where \(A\ge0\), where we set 
\[
a
:=
\left(\frac{2}{7}\right)^{\frac{7}{12}}
G^{\frac{p}{12}},
\qquad
b
:=
\left(\frac{2}{7}\right)^{-\frac{7}{12}}
2^{\frac p6}A M_T^{\frac14}, \quad \text{so that}\;\; ab=2^{p/6}A G^{p/12}M_T^{1/4}.
\]
Then Young's inequality gives
\begin{align}
2^{\frac p6}A\,G^{\frac{p}{12}}M_T^{\frac14}
&\le
\frac{1}{r}
\left[
\left(\frac{2}{7}\right)^{\frac{7}{12}}
G^{\frac{p}{12}}
\right]^r
+
\frac{1}{r'}
\left[
\left(\frac{2}{7}\right)^{-\frac{7}{12}}
2^{\frac p6}A M_T^{\frac14}
\right]^{r'}
\notag\\
&=
\frac{1}{6}
G^{\frac p7}
+
\frac{5}{12}
\left(\frac72\right)^{\frac75}
2^{\frac{2p}{5}}
A^{\frac{12}{5}}M_T^{\frac35}.
\label{eq:young_G_M_reduction}
\end{align}
Similarly, we obtain the following bound for the other form \(2^{p/6}A\,G^{\frac{p}{12}}\): 
\begin{align}
2^{\frac p6}A\,G^{\frac{p}{12}}
&\le
\frac16
G^{\frac p7}
+
\frac{5}{12}
\left(\frac72\right)^{\frac75}
2^{\frac{2p}{5}}
A^{\frac{12}{5}}.
\label{eq:young_G_reduction}
\end{align}
%Indeed,
%\[
%\left(G^{p/12}\right)^{12/7}=G^{p/7},\qquad \left(M_T^{1/4}\right)^{12/5}=M_T^{3/5}.
%\]
%There are exactly five \(G^{p/12}\)-order contributions in \cref{W_bound_after_log}: the \(D_{t,1}\)-term, the \(D_{t,2}\)-term, the \(D_{t,3}\)-term, the main predictable residual term in \(P_t\), and the final adaptive correction term in \(P_t\). The \(D_{t,2}\)-term does not contain \(M_T^{1/4}\), so it is reduced using \cref{eq:young_G_reduction}; the other four are reduced using \cref{eq:young_G_M_reduction}. 

Applying \cref{eq:young_G_M_reduction} and \cref{eq:young_G_reduction} to the five \(G^{\frac{p}{12}}\)-terms in \cref{W_bound_after_log}, we obtain
\begin{align}\label{important_1}
\Expect\!\left[
\left(2W_{\tau_G\wedge T}\right)^{\frac{p}{6}}
\right]
&\le
3^{\frac p6-1}(2B_0)^{\frac{p}{6}}
+
G^{\frac{p}{7}}
+
3^{\frac p3-1}
2^{\frac p6}
C^{\frac{p}{6}}
\left(\frac{p}{6}\right)^{\frac{p}{6}}
\notag\\
&\quad
+
\mathcal A_{1,p}M_T^{\frac35}
+
\mathcal A_{2,p}
+
\mathcal A_{3,p}M_T^{\frac35}
+
\mathcal A_{4,p}M_T^{\frac35}
+
\mathcal A_{5,p}M_T^{\frac35},
\end{align}
where the original five \(G^{\frac{p}{12}}\)-terms are bounded in total by 
\[5\cdot \frac16 G^{\frac p7}\le G^{\frac p7},\] 
and
\begin{align*}
\mathcal A_{1,p}
&:=
\frac{5}{12}
\left(\frac{7}{2}\right)^{\frac{7}{5}}
2^{\frac{2p}{5}}
\Bigg[
3^{\frac p6-1}
\lambda_1
3^{\frac{p}{12}}
\left(\frac{p}{6}\right)^{\frac{p}{6}}
\frac{1}{2}
(32Ld^2)^{\frac{p}{12}}
\left(
1+\left(\frac{p}{12}\right)^{\frac{p}{12}}
\right)
2^{\frac{p}{6}}
\left(\frac{p}{12}\right)^{\frac{p}{12}}
\Bigg]^{\frac{12}{5}},
\\[3pt]
\mathcal A_{2,p}
&:=
\frac{5}{12}
\left(\frac{7}{2}\right)^{\frac{7}{5}}
2^{\frac{2p}{5}}
\Bigg[
3^{\frac p6-1}
\lambda_1
3^{\frac{p}{12}}
\left(\frac{p}{6}\right)^{\frac{p}{6}}
\frac{1}{2}
(8Ld)^{\frac{p}{12}}
\left(
1+\left(\frac{p}{12}\right)^{\frac{p}{12}}
\right)
\Bigg]^{\frac{12}{5}},
\\[3pt]
\mathcal A_{3,p}
&:=
\frac{5}{12}
\left(\frac{7}{2}\right)^{\frac{7}{5}}
2^{\frac{2p}{5}}
\Bigg[
3^{\frac p6-1}
\lambda_1
3^{\frac{p}{12}}
\left(\frac{p}{6}\right)^{\frac{p}{6}}
\\
&\qquad\qquad\times
\left[
576d^2
\left(\frac{\beta_1}{1-\beta_1}\right)^2
\left(
4L
+
2L^2
\left(\frac{\beta_1}{1-\beta_1}\right)^2
dD_{\beta_1}^2
\right)
\right]^{\frac{p}{12}}
\\
&\qquad\qquad\times
2^{\frac{p}{6}}
\left(\frac{p}{12}\right)^{\frac{p}{12}}
\Bigg]^{\frac{12}{5}},
\\[3pt]
\mathcal A_{4,p}
&:=
\frac{5}{12}
\left(\frac{7}{2}\right)^{\frac{7}{5}}
2^{\frac{2p}{5}}
\Bigg[
3^{\frac p6-1}
3^{\frac{p}{6}}
\Bigg[
16d
\left(
\frac{\beta_1^2L}{2(1-\beta_1)^2}
+
\frac{\beta_1^2L^2}{4(1-\beta_1)^2\sqrt v}
+
\frac{9\beta_1^2L}{(1-\beta_1)^2}
\right)
\\
&\qquad\qquad
+
256d
\frac{\beta_1(1+\beta_1)}{(1-\beta_1)^2}
\left(
\sqrt{2L}
+
\sqrt C
+
\epsilon
\right)
+
4d\alpha_2
+
\frac{2240d^2L^2}{\sqrt v}
\Bigg]^{\frac{p}{6}}
2^{\frac{p}{3}}
\left(\frac{p}{6}\right)^{\frac{p}{6}}
\Bigg]^{\frac{12}{5}},
\\[3pt]
\mathcal A_{5,p}
&:=
\frac{5}{12}
\left(\frac{7}{2}\right)^{\frac{7}{5}}
2^{\frac{2p}{5}}
\Bigg[
3^{\frac p6-1}
3^{\frac{p}{6}}
\left[
\frac{4\beta_1\sqrt d}{1-\beta_1}
\left(
4L
+
2L^2
\left(\frac{\beta_1}{1-\beta_1}\right)^2
dD_{\beta_1}^2
\right)^{\frac12}
\right]^{\frac{p}{6}}
2^{\frac{p}{6}}
\left(\frac{p}{12}\right)^{\frac{p}{12}}
\Bigg]^{\frac{12}{5}}.
\end{align*}
This completes the explicit treatment of all martingale components and residual terms. In particular, \(D_{t,3}\) and all terms in \(P_t\) have been isolated above, and every \(G\)-dependent contribution entering \cref{important_1} is no larger than the order \(G^{p/12}\) before the Young-inequality reduction \cref{eq:young_G_M_reduction}--\cref{eq:young_G_reduction}.

Next, we proceed to bound \(M_T\).
%\textbf{Bounding \(\mathbb{E}[S_{\tau_G\wedge T-1}]\) for the final \(W_{\tau_G\wedge T}\) \((p/6)\)-th moment.} 
First, according to the definition of \(M_T\) in \cref{M_T}, we first bound $S_{\tau_G\wedge T-1}$ as follows:
\begin{align}\label{superadam_100}
\Expect\!\left[S_{\tau_G\wedge T-1}\right]
&=
dv
+
\Expect\!\left[
\sum_{t=1}^{\tau_G\wedge T-1}
\|g_t\|^2
\right]
\notag\\
&\mathop{=}^{(a)}
dv
+
\Expect\!\left[
\sum_{t=1}^{\tau_G\wedge T-1}
\Expect\!\left[
\|g_t\|^2
\mid
\mathscr F_{t-1}
\right]
\right]
\notag\\
&\mathop{\le}^{(b)}
dv
+
(C+2LG)T,
\end{align}
where (a) follows from Lemma~\ref{sum:expect:ab}, and (b) follows from Assumption~\ref{ass:abc}.

Moreover, on the stopped trajectory \(1\le t\le \tau_G\wedge T-1\), Lemma~\ref{loss_bound} gives
\[
\|\nabla f(\x_t)\|^2
\le
2L\bar f(\x_t)
\le
2LG.
\]
Consequently,
\[
M_T
=
\Expect\!\left[
1+\frac{S_{\tau_G\wedge T-1}}{vdT}
\right]
\le
1+\frac{1}{T}
+
\frac{C+2LG}{vd}
\le
2+\frac{C+2LG}{vd}.
\]
Since \(G\ge1\), defining
\[
K_0:=2+\frac{C+2L}{vd},
\]
we obtain
\begin{align}\label{M_T_bound}
M_T\le K_0G.
\end{align}

\noindent Combining \cref{important_1} with \cref{M_T_bound}, we obtain
\begin{align}\label{important_2}
\Expect\!\left[
\left(2W_{\tau_G\wedge T}\right)^{\frac{p}{6}}
\right]
&\le
3^{\frac p6-1}(2B_0)^{\frac{p}{6}}
+
G^{\frac{p}{7}}
+
2^{\frac p6}3^{\frac p3-1}C^{\frac{p}{6}}
\left(\frac{p}{6}\right)^{\frac{p}{6}}
\notag\\
&\quad
+
K_0^{\frac{3}{5}}G^{\frac{3}{5}}
\left(
\mathcal A_{1,p}
+
\mathcal A_{3,p}
+
\mathcal A_{4,p}
+
\mathcal A_{5,p}
\right)
+
\mathcal A_{2,p}.
\end{align}
%The independent \(G^{p/7}\) term is produced by the Young-inequality step in \cref{important_1}. The key point is that every residual term in \(P_t\) contributes at most a \(G^{p/12}\)-order factor before applying Young's inequality. Consequently, each such term can be split into a small contribution absorbed by \(G^{p/7}\) and a remaining term proportional to \(M_T^{3/5}\). Using \cref{M_T_bound}, the latter is further bounded by \(K_0^{3/5}G^{3/5}\).

Since the above inequality holds for \(p\ge12\), we complement the remaining cases \(1\le p<12\) using monotonicity of \(L^q\)-norms. Specifically, for \(1\le p<12\),
\begin{align}\label{important_3}
\Expect\!\left[
\left(2W_{\tau_G\wedge T}\right)^{\frac{p}{6}}
\right]
\le
\Expect\!\left[
\left(2W_{\tau_G\wedge T}\right)^{2}
\right]^{\frac{p}{12}}
\le
\widetilde C G^2,
\end{align}
where the second inequality follows by applying \cref{important_2} with \(p=12\). Since \(p\) ranges over the finite interval \([1,12)\), the constant \(\widetilde C\) can be chosen for the bound to hold for all \(1\le p<12\).

\noindent We have thus completed the final \((p/6)\)-th moment bound of \(2W_{\tau_G\wedge T}\). Next, we proceed to bound the exponential moment based on this result.

\textbf{Substep 5.2 (Bounding Stretched-Exponential Moment for $W_{\tau_G\wedge T}$).} We multiply both sides of \cref{important_2} by \(1/p!\). 
Then, for any integer \(N>12\), summing over all integers \(12\le p\le N\), and using \cref{important_3} to control the finitely many terms \(1\le p<12\), we obtain
\begin{align}\label{dasdassda}
\Expect\!\left[
\sum_{p=0}^{N}
\frac{\left(2W_{\tau_G\wedge T}\right)^{\frac{p}{6}}}{p!}
\right]
&=
\sum_{p=0}^{N}
\Expect\!\left[
\frac{\left(2W_{\tau_G\wedge T}\right)^{\frac{p}{6}}}{p!}
\right]
\notag\\
&\le
1+12\widetilde C G^2
+
\sum_{p=12}^{N}
\frac{3^{\frac p6-1}(2B_0)^{\frac{p}{6}}}{p!}
+
\sum_{p=12}^{N}
\frac{G^{\frac{p}{7}}}{p!}
\notag\\
&\quad
+
K_0^{\frac{3}{5}}G^{\frac{3}{5}}
\sum_{p=12}^{N}
\frac{
\mathcal A_{1,p}
+\mathcal A_{3,p}
+\mathcal A_{4,p}
+\mathcal A_{5,p}
}{p!}
+
\sum_{p=12}^{N}
\frac{\mathcal A_{2,p}}{p!}
\notag\\
&\quad
+
\sum_{p=12}^{N}
\frac{
2^{\frac p6}3^{\frac p3-1}
C^{\frac{p}{6}}
\left(\frac{p}{6}\right)^{\frac{p}{6}}
}{p!},
\end{align}
where \(\widetilde C\) is the constant introduced in \cref{important_3}. 
%\[
%\frac{T}{T-1}\le 2,\qquad \frac{\epsilon}{\sqrt T+\sqrt{T-1}}\le \epsilon, \qquad \frac{\eta}{\sqrt T}\le 1.
%\]

\noindent We now bound the series in \cref{dasdassda}. First,
\[
\sum_{p=12}^{N}
\frac{3^{\frac p6-1}(2B_0)^{\frac{p}{6}}}{p!}
\le
\sum_{p=0}^{\infty}
\frac{\left(6^{1/6}B_0^{1/6}\right)^p}{p!}
=
e^{6^{1/6}B_0^{1/6}},
\]
and
\[
\sum_{p=12}^{N}
\frac{G^{\frac{p}{7}}}{p!}
\le
\sum_{p=0}^{\infty}
\frac{G^{\frac{p}{7}}}{p!}
=
e^{G^{1/7}}.
\]
Moreover, by Stirling's formula,
\[
p!\sim \left(\frac{p}{e}\right)^p,
\]
each of the following series is absolutely convergent:
\begin{align*}
\overline C_1
&:=
\sum_{p=12}^{\infty}
\frac{\mathcal A_{1,p}}{p!},
&
\overline C_2
&:=
\sum_{p=12}^{\infty}
\frac{\mathcal A_{2,p}}{p!},
&
\overline C_3
&:=
\sum_{p=12}^{\infty}
\frac{\mathcal A_{3,p}}{p!},
\\
\overline C_4
&:=
\sum_{p=12}^{\infty}
\frac{\mathcal A_{4,p}}{p!},
&
\overline C_5
&:=
\sum_{p=12}^{\infty}
\frac{\mathcal A_{5,p}}{p!},
&
\overline C_6
&:=
\sum_{p=12}^{\infty}
\frac{
2^{\frac p6}3^{\frac p3-1}
C^{\frac{p}{6}}
\left(\frac{p}{6}\right)^{\frac{p}{6}}
}{p!}.
\end{align*}
Indeed, the numerators grow no faster than \(B^p p^{ap}\) for some fixed \(B>0\) and some \(a<1\), whereas \(p!\) grows like \(p^p e^{-p}\).

\noindent Define
\begin{align}\label{U_U}
U
:=
12\widetilde C
+
K_0^{\frac{3}{5}}
\left(
\overline C_1+\overline C_3+\overline C_4+\overline C_5
\right)
+
\overline C_2
+
\overline C_6 .
\end{align}
Since \(G\ge1\), the polynomial factors \(G^{3/5}\) and all constant terms are bounded by \(G^2\). Therefore, \cref{dasdassda} can be further bounded as:
\begin{align}\label{dasdassda_final}
\Expect\!\left[
\sum_{p=0}^{N}
\frac{\left(2W_{\tau_G\wedge T}\right)^{\frac{p}{6}}}{p!}
\right]
\le
1
+
e^{6^{1/6}B_0^{1/6}}
+
e^{G^{1/7}}
+
UG^2 .
\end{align}

\noindent Taking the limit \(N\to+\infty\) on the left-hand side of \cref{dasdassda_final} and applying the monotone convergence theorem, we obtain
\begin{align}\label{fabulous}
\Expect\!\left[
\exp\left\{
\left(2W_{\tau_G\wedge T}\right)^{\frac{1}{6}}
\right\}
\right]
&=
\Expect\!\left[
\lim_{N\to+\infty}
\sum_{p=0}^{N}
\frac{\left(2W_{\tau_G\wedge T}\right)^{\frac{p}{6}}}{p!}
\right]
\notag\\
&=
\lim_{N\to+\infty}
\Expect\!\left[
\sum_{p=0}^{N}
\frac{\left(2W_{\tau_G\wedge T}\right)^{\frac{p}{6}}}{p!}
\right]
\notag\\
&\le
1
+
e^{6^{1/6}B_0^{1/6}}
+
e^{G^{1/7}}
+
UG^2 .
\end{align}

\textbf{Step 5 (Bounds on Adaptive Gradient Energy).} 

This step develops the bound on the adaptive gradient energy, and hence completes the proof of Proposition \ref{thm:precond_energy_informal}.

According to \cref{Lebron}, \cref{fabulous}, and \(\bar f(\y_{\tau_G\wedge T})\le W_{\tau_G\wedge T}\), we obtain
\begin{align}\label{addam_0}
\mathbb P\!\left(
\sup_{1\le t\le T}\bar f(\x_t)>G
\right)
&\le
\mathbb P(\tau_G\le T)
\notag\\
&\le
\frac{l}{e^{G^{1/6}}}
\left(
1
+
e^{6^{1/6}B_0^{1/6}}
+
e^{G^{1/7}}
+
UG^2
\right)
\notag\\
&\le
\frac{l}{e^{G^{1/6}}}
\left(
1
+
e^{6^{1/6}B_0^{1/6}}
+
e^{G^{1/7}}
+
UG^2
\right)^2.
\end{align}
On the other hand, by \cref{fabulous} and the inequality
\[
W_{\tau_G\wedge T}
\ge
\frac{1}{16}
\overbrace{
\sum_{t=1}^{\tau_G\wedge T}
\sum_{i=1}^{d}
\gamma_{t,i}(\nabla f(\x_t))_i^2
}^{\mathcal E_{\tau_G\wedge T}},
\]
we also have
\begin{align}\label{eq:etau}
\Expect\!\left[
\exp\left\{
\left(
\frac{1}{16}
\mathcal E_{\tau_G\wedge T}
\right)^{\frac{1}{6}}
\right\}
\right]
&\le
1
+
e^{6^{1/6}B_0^{1/6}}
+
e^{G^{1/7}}
+
UG^2 .
\end{align}
It follows from Markov's inequality that
\begin{align}\label{energy_tail}
\mathbb P\!\left(
\frac{1}{16}\mathcal E_{\tau_G\wedge T}>G
\right)
&\le
\frac{l}{e^{G^{1/6}}}
\Expect\!\left[
\exp\left\{
\left(
\frac{1}{16}
\mathcal E_{\tau_G\wedge T}
\right)^{\frac{1}{6}}
\right\}
\right]
\notag\\
&\le
\frac{l}{e^{G^{1/6}}}
\left(
1
+
e^{6^{1/6}B_0^{1/6}}
+
e^{G^{1/7}}
+
UG^2
\right)
\notag\\
&\le
\frac{l}{e^{G^{1/6}}}
\left(
1
+
e^{6^{1/6}B_0^{1/6}}
+
e^{G^{1/7}}
+
UG^2
\right)^2.
\end{align}
Combining \cref{addam_0} and \cref{energy_tail}, and using a union bound, we obtain
\begin{align}
\mathbb P\!\left(
\frac{1}{16}\mathcal E_T\le G
\right)
&\ge
\mathbb P\!\left(
\left\{
\frac{1}{16}\mathcal E_{\tau_G\wedge T}\le G
\right\}
\cap
\{\tau_G>T\}
\right)
\notag\\
&=
1-
\mathbb P\!\left(
\left\{
\frac{1}{16}\mathcal E_{\tau_G\wedge T}>G
\right\}
\cup
\{\tau_G\le T\}
\right)
\notag\\
&\ge
1-
\mathbb P\!\left(
\frac{1}{16}\mathcal E_{\tau_G\wedge T}>G
\right)
-
\mathbb P(\tau_G\le T)
\notag\\
&\ge
1-
\frac{2l}{e^{G^{1/6}}}
\left(
1
+
e^{6^{1/6}B_0^{1/6}}
+
e^{G^{1/7}}
+
UG^2
\right)^2.
\end{align}
For any \(0<\delta<1\), we set
\[
G=G_\delta
:=
\max\left\{
\left[
\log\!\left(
\frac{24l\bigl(1+e^{2\cdot 6^{1/6}B_0^{1/6}}\bigr)}{\delta}
\right)
\right]^6,
\left[
2\log\!\left(
\frac{24l}{\delta}
\right)
\right]^6,
300^6,
\left[
2\log\!\left(
\frac{24lU^2}{\delta}
\right)
\right]^6
\right\}.
\]
With this choice of \(G\), we have
\[
\mathbb P\!\left(
\frac{1}{16}\mathcal E_T\le G_\delta
\right)
\ge
1-\delta.
\]
It then follows that, with probability at least \(1-\delta\),
\[
\mathcal E_T
\le
16
\max\left\{
\left[
\log\!\left(
\frac{24l\bigl(1+e^{2\cdot 6^{1/6}B_0^{1/6}}\bigr)}{\delta}
\right)
\right]^6,
\left[
2\log\!\left(
\frac{24l}{\delta}
\right)
\right]^6,
300^6,
\left[
2\log\!\left(
\frac{24lU^2}{\delta}
\right)
\right]^6
\right\}.
\]
This completes the proof.
\end{proof}
\subsection{Formal Statement and Proof of Theorem~\ref{thm:main_informal}}\label{formal_thm_1}

\begin{thm}[Formal Statement of Theorem~\ref{thm:main_informal}]\label{lemma_0_0}
Suppose that Assumptions~\ref{ass:nonneg}--\ref{ass:abc} hold.
Let \(\bar f(\x):=f(\x)-f^\star+1\), where
\(f^\star:=\inf_{\x\in\mathbb R^d} f(\x)\).
Consider Adam, as specified by Algorithm~\ref{alg:adam}, run for \(T\ge 10\) iterations with parameters
\(\beta_1\in[0,1)\), \(\beta_2=1-1/T\), and stepsize \(\gamma=\eta/\sqrt T\).
For any \(0<\delta<1\), assume that the step-size constant \(\eta\) satisfies
\begin{align*}
\frac{1}{\eta}
\ge
\max\left\{
\frac{8d\epsilon}{v^{3/2}}+\frac{8d}{\sqrt v},
\;
\frac{D_{\beta_1}\beta_1\sqrt{dL}}{1-\beta_1},
\;
1
\right\},
\end{align*}
where \(D_{\beta_1}\) is defined in Lemma~\ref{property_2.54}.
Then, with probability at least \(1-\delta\), one has
\begin{align*}
\frac{1}{T}
\sum_{t=1}^{T}
\|\nabla f(\x_t)\|^2
&\le
\frac{32}{\eta}
\sqrt{\frac{2F}{\delta}}
\max\Bigg\{
\left[
\log\!\left(
\frac{48l\bigl(1+e^{2\cdot 6^{1/6}B_0^{1/6}}\bigr)}{\delta}
\right)
\right]^6,
\left[
2\log\!\left(
\frac{48l}{\delta}
\right)
\right]^6,
\\
&\hspace{5em}
300^6,
\left[
2\log\!\left(
\frac{48lU^2}{\delta}
\right)
\right]^6
\Bigg\}
\frac{1}{\sqrt T},
\end{align*}
where
\[
l
:=
\exp\left\{
\frac{L\beta_1^2 dD_{\beta_1}^2}{(1-\beta_1)^2}
\right\},\quad B_0
:=
2
\left(
1+\frac{2L}{\sqrt v+\epsilon}
\right)
\bar f(\mathbf{y}_1),
\]
the constant \(U\) is defined in \cref{U_U}, and \(F\) is defined in \cref{FFF}. Both constants depend only on the fixed problem parameters and the initialization, but not on \(T\) or \(\delta\).
\end{thm}

\begin{proof}
% It follows directly from \cref{fabulous} that 
% \begin{align}
% \Expect\!\left[
% \exp\left\{
% \left(
% \frac{1}{16}\mathcal E_{\tau_G\wedge T}
% \right)^{\frac{1}{6}}
% \right\}
% \right]
% \le
% 1+e^{6^{1/6}B_0^{1/6}}+e^{G^{1/7}}+UG^2,
% \label{energy_exp_moment_for_avg}
% \end{align}
% where \(U\) is defined in \cref{U_U}. 
Define
\[
\Sigma_T
:=
\sum_{i=1}^{d}v_{T,i}.
\]
By Lemma~\ref{lem:vt_comparable}, for every \(0\le t\le T\) and every \(i\in[d]\),
\[
v_{t,i}\le 4v_{T,i}.
\]
Hence
\[
\sqrt{v_{t,i}}+\epsilon
\le
2\sqrt{v_{T,i}}+\epsilon
\le
2(\sqrt{\Sigma_T}+\epsilon).
\]
Therefore,
\[
\gamma_{t,i}
=
\frac{\eta}{\sqrt T}
\frac{1}{\sqrt{v_{t,i}}+\epsilon}
\ge
\frac{\eta}{2\sqrt T(\sqrt{\Sigma_T}+\epsilon)}.
\]
Consequently,
\[
\mathcal E_{\tau_G\wedge T}
=
\sum_{t=1}^{\tau_G\wedge T}
\sum_{i=1}^{d}
\gamma_{t,i}(\nabla f(\x_t))_i^2
\ge
\frac{\eta}{2\sqrt T(\sqrt{\Sigma_T}+\epsilon)}
\sum_{t=1}^{\tau_G\wedge T}
\|\nabla f(\x_t)\|^2.
\]
Thus, the above inequality and \cref{eq:etau} yield 
\begin{align}
\Expect\!\left[
\exp\left\{
\left(
\frac{\eta}{32\sqrt T(\sqrt{\Sigma_T}+\epsilon)}
\sum_{t=1}^{\tau_G\wedge T}
\|\nabla f(\x_t)\|^2
\right)^{\frac{1}{6}}
\right\}
\right]
\le
1+e^{6^{1/6}B_0^{1/6}}+e^{G^{1/7}}+UG^2.
\label{unweighted_stopped_exp_moment}
\end{align}
By Markov's inequality, we have
\begin{align}
&\mathbb P\!\left(
\frac{\eta}{32\sqrt T(\sqrt{\Sigma_T}+\epsilon)}
\sum_{t=1}^{\tau_G\wedge T}
\|\nabla f(\x_t)\|^2
>
G
\right)
\notag\\
&\qquad\le
\frac{1}{e^{G^{1/6}}}
\Expect\!\left[
\exp\left\{
\left(
\frac{\eta}{32\sqrt T(\sqrt{\Sigma_T}+\epsilon)}
\sum_{t=1}^{\tau_G\wedge T}
\|\nabla f(\x_t)\|^2
\right)^{\frac{1}{6}}
\right\}
\right]
\notag\\
&\qquad \le
\frac{1}{e^{G^{1/6}}}
\left(
1+e^{6^{1/6}B_0^{1/6}}+e^{G^{1/7}}+UG^2
\right).
\label{stopped_unweighted_tail}
\end{align}
In addition, \cref{addam_0} yields
\begin{align}
\mathbb P(\tau_G\le T)
\le
\frac{l}{e^{G^{1/6}}}
\left(
1+e^{6^{1/6}B_0^{1/6}}+e^{G^{1/7}}+UG^2
\right)^2 .
\label{stopping_tail_for_avg}
\end{align}
Since \(l\ge 1\), \cref{stopped_unweighted_tail} is further bounded by the right-hand side of \cref{stopping_tail_for_avg}. Hence, by a union bound,
\begin{align}\label{dsqwerqwr}
&\mathbb P\!\left(
\frac{\eta}{32\sqrt T(\sqrt{\Sigma_T}+\epsilon)}
\sum_{t=1}^{T}
\|\nabla f(\x_t)\|^2
\le
G
\right)
\notag\\
&\ge
\mathbb P\!\left(
\left\{
\frac{\eta}{32\sqrt T(\sqrt{\Sigma_T}+\epsilon)}
\sum_{t=1}^{\tau_G\wedge T}
\|\nabla f(\x_t)\|^2
\le
G
\right\}
\cap
\{\tau_G>T\}
\right)
\notag\\
&=
1-
\mathbb P\!\left(
\left\{
\frac{\eta}{32\sqrt T(\sqrt{\Sigma_T}+\epsilon)}
\sum_{t=1}^{\tau_G\wedge T}
\|\nabla f(\x_t)\|^2
>
G
\right\}
\cup
\{\tau_G\le T\}
\right)
\notag\\
&\ge
1-
\mathbb P\!\left(
\frac{\eta}{32\sqrt T(\sqrt{\Sigma_T}+\epsilon)}
\sum_{t=1}^{\tau_G\wedge T}
\|\nabla f(\x_t)\|^2
>
G
\right)
-
\mathbb P(\tau_G\le T)
\notag\\
&\ge
1-
\frac{2l}{e^{G^{1/6}}}
\left(
1+e^{6^{1/6}B_0^{1/6}}+e^{G^{1/7}}+UG^2
\right)^2 .
\end{align}

We next bound \(\Sigma_T\). Recalling \cref{v_11}, for every \(t\le T\), we have
\[
v_{t,i}
=
\left(1-\frac{1}{T}\right)^t v_{0,i}
+
\frac{1}{T}
\sum_{s=1}^{t}
\left(1-\frac{1}{T}\right)^{t-s}g_{s,i}^2
\le
v+\frac{1}{T}\sum_{s=1}^{t}g_{s,i}^2.
\]
Consequently,
\[
\Sigma_T
=
\sum_{i=1}^{d}v_{T,i}
\le
dv+\frac{1}{T}\sum_{s=1}^{T}\|g_s\|^2.
\]
Taking expectations gives
\begin{align}\label{sssadam_0}
\Expect[\Sigma_T]
&\le
dv
+
\frac{1}{T}
\Expect\!\left[
\sum_{t=1}^{T}
\|g_t\|^2
\right]
\notag\\
&=
dv
+
\frac{1}{T}
\Expect\!\left[
\sum_{t=1}^{T}
\Expect\!\left[
\|g_t\|^2
\mid
\mathscr F_{t-1}
\right]
\right]
\notag\\
&\le
dv
+
C
+
\frac{2L}{T}
\Expect\!\left[
\sum_{t=1}^{T}
\bar f(\x_t)
\right]
\notag\\
&\le
dv+C+2L
\Expect\!\left[
\sup_{1\le t\le T}\bar f(\x_t)
\right].
\end{align}
To bound the expectation in the last step, we derive the tail-integral identity
directly from the definition of expectation, Fubini's theorem, and the monotone
convergence theorem. 

Let \(\mathbb P_{\sup}\) denote the distribution of
\(\sup_{1\le t\le T}\bar f(\x_t)\) on \([0,+\infty)\). For every \(R>0\), by the
definition of expectation,
\[
\Expect\!\left[
\left(\sup_{1\le t\le T}\bar f(\x_t)\right)\wedge R
\right]
=
\int_{[0,+\infty)}
(x\wedge R)\,d\mathbb P_{\sup}(x).
\]
For every \(x\ge0\),
\[
x\wedge R
=
\int_{0}^{R}\mathbf 1_{\{G<x\}}\,dG.
\]
Therefore,
\begin{align*}
\Expect\!\left[
\left(\sup_{1\le t\le T}\bar f(\x_t)\right)\wedge R
\right]
&=
\int_{[0,+\infty)}
\int_{0}^{R}
\mathbf 1_{\{G<x\}}\,dG\,d\mathbb P_{\sup}(x)
\\
&=
\int_{0}^{R}
\int_{[0,+\infty)}
\mathbf 1_{\{G<x\}}\,d\mathbb P_{\sup}(x)\,dG
\\
&=
\int_{0}^{R}
\mathbb P\!\left(
\sup_{1\le t\le T}\bar f(\x_t)>G
\right)
\,dG,
\end{align*}
where the interchange of the two integrals follows from Fubini's theorem applied
to the bounded nonnegative integrand. Letting \(R\to+\infty\), the monotone
convergence theorem yields
\[
\Expect\!\left[
\sup_{1\le t\le T}\bar f(\x_t)
\right]
=
\int_{0}^{+\infty}
\mathbb P\!\left(
\sup_{1\le t\le T}\bar f(\x_t)>G
\right)
\,dG.
\]
Thus, using the tail estimate for \(G\ge1\), we obtain
\begin{align*}
\Expect\!\left[
\sup_{1\le t\le T}\bar f(\x_t)
\right]
&=
\int_{0}^{+\infty}
\mathbb P\!\left(
\sup_{1\le t\le T}\bar f(\x_t)>G
\right)
\,dG
\\
&\le
1+
\int_{1}^{+\infty}
\frac{l}{e^{G^{1/6}}}
\left(
1+e^{6^{1/6}B_0^{1/6}}+e^{G^{1/7}}+UG^2
\right)^2
\,dG
\\
&<+\infty .
\end{align*}
where the finiteness of \(S\) in the last step follows because the factor \(e^{-G^{1/6}}\) dominates both \(e^{2G^{1/7}}\) and every polynomial factor as \(G\to+\infty\). Substituting this bound into \cref{sssadam_0}, we obtain
\[
\Expect[\Sigma_T]
\le
dv+C+2LS.
\]
Moreover, since
\[
(\sqrt{\Sigma_T}+\epsilon)^2
\le
2\Sigma_T+2\epsilon^2,
\]
we have
\begin{align}\label{FFF}
\Expect\!\left[
(\sqrt{\Sigma_T}+\epsilon)^2
\right]
\le
2(dv+C+2LS)+2\epsilon^2
=:F.
\end{align}
By Markov's inequality, for every \(H>0\),
\[
\mathbb P\!\left(
\sqrt{\Sigma_T}+\epsilon>H
\right)
=
\mathbb P\!\left(
(\sqrt{\Sigma_T}+\epsilon)^2>H^2
\right)
\le
\frac{F}{H^2}.
\]
Combining this bound with \cref{dsqwerqwr}, we obtain
\begin{align}
&\mathbb P\!\left(
\frac{\eta}{32\sqrt T}
\sum_{t=1}^{T}
\|\nabla f(\x_t)\|^2
\le
HG
\right)
\notag\\
&\ge
\mathbb P\!\left(
\left\{
\frac{\eta}{32\sqrt T(\sqrt{\Sigma_T}+\epsilon)}
\sum_{t=1}^{T}
\|\nabla f(\x_t)\|^2
\le
G
\right\}
\cap
\{\sqrt{\Sigma_T}+\epsilon\le H\}
\right)
\notag\\
&\ge
1-
\frac{2l}{e^{G^{1/6}}}
\left(
1+e^{6^{1/6}B_0^{1/6}}+e^{G^{1/7}}+UG^2
\right)^2
-
\frac{F}{H^2}.
\label{avg_grad_tail_before_choice}
\end{align}
In \cref{avg_grad_tail_before_choice}, by choosing $H=\sqrt{\frac{2F}{\delta}}$ and 
\begin{align*}
G=G_\delta'
:=
\max\Bigg\{
&
\left[
\log\!\left(
\frac{48l\bigl(1+e^{2\cdot 6^{1/6}B_0^{1/6}}\bigr)}{\delta}
\right)
\right]^6,
\left[
2\log\!\left(
\frac{48l}{\delta}
\right)
\right]^6,
\\
&
300^6,
\left[
2\log\!\left(
\frac{48lU^2}{\delta}
\right)
\right]^6
\Bigg\},
\end{align*}
we obtain 
\[
\mathbb P\!\left(
\frac{\eta}{32\sqrt T}
\sum_{t=1}^{T}
\|\nabla f(\x_t)\|^2
\le
HG_\delta'
\right)
\ge
1-\delta.
\]
Equivalently, with probability at least \(1-\delta\),
\begin{align*}
\frac{\eta}{\sqrt T}
\sum_{t=1}^{T}
\|\nabla f(\x_t)\|^2
&\le
32\sqrt{\frac{2F}{\delta}}
\max\Bigg\{
\left[
\log\!\left(
\frac{48l\bigl(1+e^{2\cdot 6^{1/6}B_0^{1/6}}\bigr)}{\delta}
\right)
\right]^6,
\\
&\qquad
\left[
2\log\!\left(
\frac{48l}{\delta}
\right)
\right]^6,
300^6,
\left[
2\log\!\left(
\frac{48lU^2}{\delta}
\right)
\right]^6
\Bigg\}.
\end{align*}
Dividing by \(\eta\sqrt T\), we finally obtain
\begin{align*}
\frac{1}{T}
\sum_{t=1}^{T}
\|\nabla f(\x_t)\|^2
&\le
\frac{32}{\eta}
\sqrt{\frac{2F}{\delta}}
\max\Bigg\{
\left[
\log\!\left(
\frac{48l\bigl(1+e^{2\cdot 6^{1/6}B_0^{1/6}}\bigr)}{\delta}
\right)
\right]^6,
\\
&\qquad
\left[
2\log\!\left(
\frac{48l}{\delta}
\right)
\right]^6,
300^6,
\left[
2\log\!\left(
\frac{48lU^2}{\delta}
\right)
\right]^6
\Bigg\}
\frac{1}{\sqrt T}.
\end{align*}
This completes the proof.
\end{proof}

\section{Proofs of All Lemmas}
\subsection{The Proof of Lemma \ref{sum:expect:ab}}
\begin{proof}
Define \(Y_n := X_n - \mathbb{E}[X_n \mid \mathscr{F}_{n-1}]\) and \(M_n := \sum_{k=1}^n Y_k\); then \((M_n)\) is an \((\mathscr{F}_n)\)-martingale.  
For bounded stopping times \(s-1<s \le t \le N\), the \emph{optional stopping} theorem yields
\(\mathbb{E}[M_t - M_{s-1}] = 0\), i.e. \(\mathbb{E}\left[\sum_{n=s}^t Y_n\right]=0\).  
Hence
\[
\mathbb{E}\left[\sum_{n=s}^t X_n\right]
=
\mathbb{E}\left[\sum_{n=s}^t \mathbb{E}[X_n \mid \mathscr{F}_{n-1}]\right].
\]
\end{proof}
\subsection{The Proof of Lemma \ref{property_3}}
\begin{proof}
According to Algorithm \ref{alg:adam}, we have the following iterative equations
\begin{align*}
m_{t,i}=\beta_{1}m_{t-1,i}+(1-\beta_{1})g_{t,i}.
\end{align*}
We take the square of the 2-norm on both sides, which yields
\begin{align*}
m_{t,i}^{2}&=(\beta_{1}m_{t-1,i}+(1-\beta_{1})g_{t,i})^{2}\notag\\&=\beta_{1}^{2}m_{t-1,i}^{2}+2\beta_{1}(1-\beta_{1})m_{t-1,i}g_{t,i}+(1-\beta_{1})^{2}g_{t,i}^{2}\notag\\&\mathop{\le}^{(a)}\beta_{1}m_{t-1,i}^{2}+(1-\beta_{1})g_{t,i}^{2}.%\\&\mathop{\le}^{\text{Eq. \ref{jh_100}}}\beta_{1}\|m_{t-1}\|^{2}+(1-\beta_{1})\|g_{t}\|^{2} .
\end{align*}
In step (a), we used the \emph{AM-GM} inequality, i.e.,
\[2\beta_{1}(1-\beta_{1})m_{t-1,i}g_{t,i}\le \beta_{1}(1-\beta_{1})m_{t-1,i}^{2}+\beta_{1}(1-\beta_{1})g_{t,i}^{2},\] that is,
\begin{align*}
m_{t,i}^{2}-m_{t-1,i}^{2}\le -(1-\beta_{1})m_{t-1,i}^{2}+(1-\beta_{1})g_{t,i}^{2},
\end{align*}
which completes the proof.
\end{proof}
\subsection{The Proof of Lemma \ref{property_2}}
\begin{proof}
We prove the two-sided comparison between \(\bar f(\x_t)\) and \(\bar f(\y_t)\).

First, by \(L\)-smoothness, we have
\begin{align*}
f(\y_t)
&\le
f(\x_t)
+
\langle \nabla f(\x_t),\y_t-\x_t\rangle
+
\frac{L}{2}\|\y_t-\x_t\|^2 .
\end{align*}
Using Young's inequality,
\[
\langle \nabla f(\x_t),\y_t-\x_t\rangle
\le
\frac{1}{2L}\|\nabla f(\x_t)\|^2
+
\frac{L}{2}\|\y_t-\x_t\|^2 .
\]
Therefore,
\begin{align*}
\bar f(\y_t)
&=
f(\y_t)-f^\star+1
\\
&\le
\bar f(\x_t)
+
\frac{1}{2L}\|\nabla f(\x_t)\|^2
+
L\|\y_t-\x_t\|^2 .
\end{align*}
Since \(f\) is \(L\)-smooth and lower bounded by \(f^\star\), we have
\[
\|\nabla f(\x_t)\|^2
\le
2L\bigl(f(\x_t)-f^\star\bigr)
\le
2L\bar f(\x_t).
\]
Thus,
\begin{align}
\bar f(\y_t)
&\le
2\bar f(\x_t)
+
L\|\y_t-\x_t\|^2 .
\label{eq:y_by_x_bar_f}
\end{align}

Conversely, applying the same argument with \(\x_t\) and \(\y_t\) interchanged gives
\begin{align*}
f(\x_t)
&\le
f(\y_t)
+
\langle \nabla f(\y_t),\x_t-\y_t\rangle
+
\frac{L}{2}\|\x_t-\y_t\|^2
\\
&\le
f(\y_t)
+
\frac{1}{2L}\|\nabla f(\y_t)\|^2
+
L\|\x_t-\y_t\|^2 .
\end{align*}
Hence
\begin{align*}
\bar f(\x_t)
&\le
\bar f(\y_t)
+
\frac{1}{2L}\|\nabla f(\y_t)\|^2
+
L\|\x_t-\y_t\|^2 .
\end{align*}
Again using \(L\)-smoothness and the lower bound \(f^\star\),
\[
\|\nabla f(\y_t)\|^2
\le
2L\bigl(f(\y_t)-f^\star\bigr)
\le
2L\bar f(\y_t).
\]
Therefore,
\begin{align}
\bar f(\x_t)
&\le
2\bar f(\y_t)
+
L\|\x_t-\y_t\|^2 .
\label{eq:x_by_y_bar_f}
\end{align}

Finally, by the definition of \(\y_t\),
\[
\y_t-\x_t
=
-\frac{\beta_1}{1-\beta_1}\gamma_{t-1}\odot m_{t-1}.
\]
Substituting this identity into \cref{eq:y_by_x_bar_f} and
\cref{eq:x_by_y_bar_f} yields
\[
\bar f(\y_t)
\le
2\bar f(\x_t)
+
\frac{L\beta_1^2}{(1-\beta_1)^2}
\|\gamma_{t-1}\odot m_{t-1}\|^2,
\]
and
\[
\bar f(\x_t)
\le
2\bar f(\y_t)
+
\frac{L\beta_1^2}{(1-\beta_1)^2}
\|\gamma_{t-1}\odot m_{t-1}\|^2.
\]
This proves the desired two-sided comparison.
\end{proof}
\subsection{The Proof of Lemma \ref{lem:vt_comparable}}
\begin{proof}
Fix $i\in[d]$. Since $g_{t,i}^2\ge 0$, the update implies
\[
v_{t,i}
=\Bigl(1-\frac{1}{T}\Bigr)v_{t-1,i}+\frac{1}{T}g_{t,i}^2
\;\ge\;\Bigl(1-\frac{1}{T}\Bigr)v_{t-1,i}.
\]
Hence, $v_{t-1,i}\le \bigl(1-\frac{1}{T}\bigr)^{-1}v_{t,i}$. Iterating from $t=h$ down to $t=k+1$ gives
\[
v_{k,i}\;\le\;\Bigl(1-\frac{1}{T}\Bigr)^{-(h-k)}v_{h,i}
\;\le\;\Bigl(1-\frac{1}{T}\Bigr)^{-T}v_{h,i}.
\]
For $T\ge2$, we have $\bigl(1-\frac{1}{T}\bigr)^T\ge \bigl(\frac12\bigr)^2=\frac14$. Hence,
$\bigl(1-\frac{1}{T}\bigr)^{-T}\le 4$, which completes the proof.
\end{proof}
\subsection{The Proof of Lemma \ref{property_3.5}}\label{p_property_3.5}
\begin{proof}
We first prove the unscaled momentum bound. Since
\[
m_t
=
(1-\beta_1)\sum_{k=1}^{t}\beta_1^{t-k}g_k,
\]
by Jensen's inequality,
\begin{align*}
\|m_t\|^2
&=
\left\|
(1-\beta_1)\sum_{k=1}^{t}\beta_1^{t-k}g_k
\right\|^2
\\
&\le
(1-\beta_1)\sum_{k=1}^{t}\beta_1^{t-k}\|g_k\|^2.
\end{align*}
This proves the first claim. This argument also covers the case \(\beta_1=0\).

\noindent We next prove the scaled momentum bound. Since \(\gamma_t\) is time-varying, one cannot apply the recursion directly to \(\gamma_t\odot m_t\). Instead, we again use the explicit representation of \(m_t\):
\[
\gamma_t\odot m_t
=
(1-\beta_1)
\sum_{k=1}^{t}
\beta_1^{t-k}
(\gamma_t\odot g_k).
\]
By the coordinatewise comparison of the adaptive stepsizes, for every \(1\le k\le t\le T-1\) and every \(i\in[d]\),
\[
\gamma_{t,i}\le 2\gamma_{k,i}.
\]
Hence, coordinatewise,
\[
|\gamma_{t,i}m_{t,i}|
\le
2(1-\beta_1)
\sum_{k=1}^{t}
\beta_1^{t-k}
|\gamma_{k,i}g_{k,i}|.
\]
Using Jensen's inequality for the convex function \(x\mapsto x^2\), we obtain
\begin{align*}
(\gamma_{t,i}m_{t,i})^2
&\le
4
\left(
(1-\beta_1)
\sum_{k=1}^{t}
\beta_1^{t-k}
|\gamma_{k,i}g_{k,i}|
\right)^2
\\
&\le
4(1-\beta_1)
\sum_{k=1}^{t}
\beta_1^{t-k}
(\gamma_{k,i}g_{k,i})^2.
\end{align*}
Summing over \(i=1,\ldots,d\) gives
\begin{align*}
\|\gamma_t\odot m_t\|^2
\le
4(1-\beta_1)
\sum_{k=1}^{t}
\beta_1^{t-k}
\|\gamma_k\odot g_k\|^2.
\end{align*}

\noindent Let
\[
N:=\mu\wedge T-1.
\]
Since the preceding bound is pathwise, summing over \(t=1,\ldots,N\) yields
\begin{align*}
\sum_{t=1}^{N}\|\gamma_t\odot m_t\|^2
&\le
4(1-\beta_1)
\sum_{t=1}^{N}
\sum_{k=1}^{t}
\beta_1^{t-k}
\|\gamma_k\odot g_k\|^2
\\
&=
4(1-\beta_1)
\sum_{k=1}^{N}
\sum_{t=k}^{N}
\beta_1^{t-k}
\|\gamma_k\odot g_k\|^2
\\
&\le
4(1-\beta_1)
\sum_{k=1}^{N}
\left(
\sum_{j=0}^{\infty}\beta_1^j
\right)
\|\gamma_k\odot g_k\|^2
\\
&=
4\sum_{k=1}^{N}
\|\gamma_k\odot g_k\|^2.
\end{align*}
Therefore,
\begin{align*}
\sum_{t=1}^{\mu\wedge T-1}
\|\gamma_t\odot m_t\|^2
\le
4
\sum_{t=1}^{\mu\wedge T-1}
\|\gamma_t\odot g_t\|^2.
\end{align*}

\noindent Moreover, since \(m_0=0\), the shifted version also follows immediately:
\begin{align*}
\sum_{t=1}^{\mu\wedge T-1}
\|\gamma_{t-1}\odot m_{t-1}\|^2
&=
\sum_{s=1}^{\mu\wedge T-2}
\|\gamma_s\odot m_s\|^2
\\
&\le
4
\sum_{s=1}^{\mu\wedge T-2}
\|\gamma_s\odot g_s\|^2
\\
&\le
4
\sum_{s=1}^{\mu\wedge T-1}
\|\gamma_s\odot g_s\|^2.
\end{align*}
This completes the proof.
\end{proof}
\subsection{The Proof of Lemma \ref{property_2.54}}\label{p_property_2.54}
\begin{proof}
We define \[\tilde{\gamma}_{t, i} := \frac{\eta }{\sqrt{T}}\frac{1}{\sqrt{v_{t,i}}},\] and according to Lemma~\ref{property_3}, we can readily obtain
\begin{align*}
 \tilde{\gamma}_{t,i}^2m_{t,i}^2-\tilde{\gamma}_{t,i}^2m_{t-1,i}^2 \le -(1-\beta_1)\tilde{\gamma}_{t,i}^2m_{t-1,i}^2+(1-\beta_1)\tilde{\gamma}_{t,i}^2g_{t,i}^2.  
\end{align*}
By rearranging the terms, we obtain
\begin{align*}
 \tilde{\gamma}_{t,i}^2m_{t,i}^2 &\le \beta_1\tilde{\gamma}_{t,i}^2m_{t-1,i}^2+(1-\beta_1)\gamma^2_{t,i}g_{t,i}^2 \\
 &=\beta_1\frac{\eta ^2}{T}\frac{m_{t-1,i}^2}{v_{t,i}} +(1-\beta_1)\tilde{\gamma}_{t,i}^2g_{t,i}^2\\
 &\mathop{\le}^{(a)}  \beta_{1}\frac{\eta ^2}{T-1}\frac{m_{t-1,i}^2}{v_{t-1,i}}+(1-\beta_1)\tilde{\gamma}_{t,i}^2g_{t,i}^2\\&=\beta_{1}\frac{T}{T-1}\frac{\eta ^2}{T}\frac{m_{t-1,i}^2}{v_{t-1,i}}+(1-\beta_1)\tilde{\gamma}_{t,i}^2g_{t,i}^2\\&=\frac{T\beta_1}{T-1}\tilde{\gamma}_{t-1,i}^2m^2_{t-1,i}+(1-\beta_1)\tilde{\gamma}_{t,i}^2g_{t,i}^2,
\end{align*}
where (a) follows from \cref{v_10}.

From the above recurrence relation, and using \(m_0=0\), we can derive the following general bound:
\begin{align}\label{C}
\tilde{\gamma}_{t,i}^2m_{t,i}^2
&\le
(1-\beta_1)
\sum_{k=1}^{t}
\left(\frac{T\beta_1}{T-1}\right)^{t-k}
\tilde{\gamma}_{k,i}^2g_{k,i}^2
\notag\\
&\mathop{\le}^{(a)}
4\eta ^2(1-\beta_1)
\sum_{r=0}^{t-1}
\left(\frac{T\beta_1}{T-1}\right)^r
\notag\\
&\le
4\eta ^2(1-\beta_1)
\sup_{T\ge 10}
\sum_{r=0}^{T-1}
\left(\frac{T\beta_1}{T-1}\right)^r,
\end{align}
where (a) follows because
\begin{align*}
v_{k,i}
&=
\left(1-\frac{1}{T}\right)^k v_{0,i}
+
\frac{1}{T}
\sum_{s=1}^{k}
\left(1-\frac{1}{T}\right)^{k-s} g_{s,i}^2
\ge
\frac{1}{T}\left(1-\frac{1}{T}\right)^T g_{k,i}^2
\ge
\frac{1}{4T}g_{k,i}^2.
\end{align*}
We next bound the geometric factor. Let
\[
r(T):=\frac{T\beta_1}{T-1},
\qquad T\ge10.
\]
Since \(r(T)\to\beta_1<1\), there exists \(M_{\beta_1}\ge10\) such that
\(r(T)\le(1+\beta_1)/2<1\) for all \(T\ge M_{\beta_1}\). Hence
\[
\sum_{r=0}^{T-1}r(T)^r
\le
\frac{2}{1-\beta_1},
\qquad T\ge M_{\beta_1}.
\]
On the finite interval \(10\le T\le M_{\beta_1}\), the same finite sum is bounded. Therefore,
\[
\sup_{T\ge 10}
\sum_{r=0}^{T-1}
\left(\frac{T\beta_1}{T-1}\right)^r
<+\infty.
\]
We define
\[
D_{\beta_1}^2
:=
4(1-\beta_1)
\sup_{T\ge 10}
\sum_{r=0}^{T-1}
\left(\frac{T\beta_1}{T-1}\right)^r .
\]
Combining the above with \cref{C}, we obtain:
\begin{align*}
 \tilde{\gamma}_{t,i}^2m_{t,i}^2\le \eta^2D^2_{\beta_1}   
\end{align*}
It is evident that \( \gamma_{t, i} \le \tilde{\gamma}_{t, i} \), and thus we obtain
\begin{align*}
    \gamma_{t,i}^2m_{t,i}^2\le \tilde{\gamma}_{t,i}^2m_{t,i}^2\le \eta ^2 D_{\beta_1}^2.
\end{align*}
Then we obtain
\begin{align*}
    \|\x_{t+1}-\x_t\|=\|\gamma_{t}\odot m_{t}\|\le \sqrt{d}\eta D_{\beta_1}.
\end{align*}
This completes the proof.   
\end{proof}
\subsection{The Proof of Lemma \ref{lemma_s_delta}}\label{p_lemma_s_delta}
\begin{proof}
For each coordinate, define
\begin{equation}\label{S_v}
S_{t,i}:=v+\sum_{s=1}^{t}g_{s,i}^2,\qquad S_{0,i}:=v .
\end{equation}
Since \(T\ge10\), the coefficients in the expansions of \(T v_{t,i}\) and
\((T-1)v_{t-1,i}\) are bounded below by \(1/3\). Hence, for all
\(1\le t\le T\),
\[
T v_{t,i}\ge \frac13 S_{t,i},
\qquad
(T-1)v_{t-1,i}\ge \frac13 S_{t-1,i}.
\]
Also, by \cref{v_11}, \(v_{t,i}\ge v/4\) and \(v_{t-1,i}\ge v/4\).
Moreover,
\[
T v_{t,i}=(T-1)v_{t-1,i}+g_{t,i}^2\ge (T-1)v_{t-1,i},
\]
so \(\Delta_{t,1,i}\ge0\). Therefore,
\begin{align*}
|\Delta_{t,1,i}|
&=
\eta
\frac{
\sqrt{Tv_{t,i}}-\sqrt{(T-1)v_{t-1,i}}
+\epsilon(\sqrt T-\sqrt{T-1})
}{
\sqrt{T(T-1)}
(\sqrt{v_{t-1,i}}+\epsilon)(\sqrt{v_{t,i}}+\epsilon)
}
\\
&=:A_{t,i}+B_{t,i},
\end{align*}
where
\[
A_{t,i}:=
\eta
\frac{
\sqrt{Tv_{t,i}}-\sqrt{(T-1)v_{t-1,i}}
}{
\sqrt{T(T-1)}
(\sqrt{v_{t-1,i}}+\epsilon)(\sqrt{v_{t,i}}+\epsilon)
}
\]
and
\[
B_{t,i}:=
\eta
\frac{
\epsilon(\sqrt T-\sqrt{T-1})
}{
\sqrt{T(T-1)}
(\sqrt{v_{t-1,i}}+\epsilon)(\sqrt{v_{t,i}}+\epsilon)
}.
\]
For the first term, rationalizing only the square-root difference gives
\begin{align*}
A_{t,i}
&=
\eta
\frac{g_{t,i}^2}{
\sqrt{T(T-1)}
(\sqrt{v_{t-1,i}}+\epsilon)(\sqrt{v_{t,i}}+\epsilon)
\bigl(\sqrt{Tv_{t,i}}+\sqrt{(T-1)v_{t-1,i}}\bigr)
}
\\
&\le
\frac{
6\eta g_{t,i}^2
}{
\sqrt{S_{t-1,i}}\sqrt{S_{t,i}}
\bigl(\sqrt{S_{t-1,i}}+\sqrt{S_{t,i}}\bigr)
}
\\
&=
6\eta
\left(
\frac{1}{\sqrt{S_{t-1,i}}}
-
\frac{1}{\sqrt{S_{t,i}}}
\right).
\end{align*}
For the second term, using \(v_{t,i},v_{t-1,i}\ge v/4\) yields
\[
\sum_{t=1}^{T}B_{t,i}
\le
\eta
\frac{
\epsilon T(\sqrt T-\sqrt{T-1})
}{
\sqrt{T(T-1)}(\sqrt v/2+\epsilon)^2
}
\le
\frac{\eta}{2\sqrt v},
\]
where we used
\[
\frac{T(\sqrt T-\sqrt{T-1})}{\sqrt{T(T-1)}}\le 1,
\qquad
\frac{\epsilon}{(\sqrt v/2+\epsilon)^2}\le \frac{1}{2\sqrt v}.
\]
Taking the summation over \(t=1,\ldots,T\), the \(A_{t,i}\)-term telescopes:
\[
\sum_{t=1}^{T}
\left(
\frac{1}{\sqrt{S_{t-1,i}}}
-
\frac{1}{\sqrt{S_{t,i}}}
\right)
\le
\frac{1}{\sqrt v}.
\]
Therefore,
\begin{align*}
\sum_{t=1}^{T}\sum_{i=1}^{d}|\Delta_{t,1,i}|
&\le
\frac{13d\eta}{2\sqrt v}
\le
\left(\frac{8d\epsilon}{v^{3/2}}+\frac{8d}{\sqrt{v}}\right)\eta .
\end{align*}
With this, we complete the proof.
\end{proof}

\subsection{The Proof of Lemma \ref{descent_lemma} (Formal Version of Lemma \ref {descent_lemma'})}\label{p_descent_lemma}
\begin{proof}
First, by applying the second-order Taylor expansion to \( f(\y_{t+1}) - f(\y_t) \), we obtain
\begin{align}\label{adam_1}
&f(\y_{t+1})-f(\y_t) \notag\\
&\le 
\nabla f(\y_t)^{\top}\left(\y_{t+1}-\y_t\right)
+\frac{L}{2}\left\|\y_{t+1}-\y_t\right\|^2 \notag\\
&\mathop{\le}^{(a)}
-\sum_{i=1}^{d}\gamma_{t,i}(\nabla f(\y_t))_ig_{t,i}
+
\underbrace{
\frac{\beta_1}{1-\beta_1}
\sum_{i=1}^{d}
(\gamma_{t-1,i}-\gamma_{t,i})(\nabla f(\y_t))_i m_{t-1,i}
}_{H_t}
\notag\\
&\qquad
+L\sum_{i=1}^{d}\gamma_{t,i}^2g_{t,i}^2
+
\frac{\beta_1^2L}{(1-\beta_1)^2}
\sum_{i=1}^{d}
(\gamma_{t-1,i}-\gamma_{t,i})^2m_{t-1,i}^2
\notag\\
&\le
\underbrace{-\sum_{i=1}^{d}\gamma_{t,i}(\nabla f(\x_t))_ig_{t,i}}_{W_t}
+
\sum_{i=1}^{d}
\left|(\nabla f(\y_t))_i-(\nabla f(\x_t))_i\right|
\left|\gamma_{t,i}g_{t,i}\right|
+
H_t
\notag\\
&\qquad
+L\sum_{i=1}^{d}\gamma_{t,i}^2g_{t,i}^2
+
\frac{\beta_1^2L}{(1-\beta_1)^2}
\sum_{i=1}^{d}
(\gamma_{t-1,i}-\gamma_{t,i})^2m_{t-1,i}^2
\notag\\
&\mathop{\le}^{(b)}
W_t
+
H_t
+
\frac{\beta_1^2L}{2(1-\beta_1)^2}
\|\gamma_{t-1}\odot m_{t-1}\|^2
+
\frac{3L}{2}\sum_{i=1}^{d}\gamma_{t,i}^2g_{t,i}^2
\notag\\
&\qquad
+
\frac{\beta_1^2L}{(1-\beta_1)^2}
\sum_{i=1}^{d}
(\gamma_{t-1,i}-\gamma_{t,i})^2m_{t-1,i}^2 .
\end{align}
To obtain step (a), we use the exact increment of \(\y_t\). By \cref{y} and the Adam updates,
\begin{align}\label{y_1}
\begin{aligned}
\y_{t+1}-\y_t
&=
\frac{\x_{t+1}-\beta_1\x_t}{1-\beta_1}
-
\frac{\x_t-\beta_1\x_{t-1}}{1-\beta_1} \\
&=
\frac{\x_{t+1}-\x_t-\beta_1(\x_t-\x_{t-1})}{1-\beta_1} \\
&=
\frac{-\gamma_t\odot m_t+\beta_1\gamma_{t-1}\odot m_{t-1}}{1-\beta_1} \\
&=
\frac{
-\gamma_t\odot\big(\beta_1m_{t-1}+(1-\beta_1)g_t\big)
+\beta_1\gamma_{t-1}\odot m_{t-1}
}{1-\beta_1} \\
&=
-\gamma_t\odot g_t
+
\frac{\beta_1}{1-\beta_1}
(\gamma_{t-1}-\gamma_t)\odot m_{t-1}.
\end{aligned}
\end{align}
Therefore,
\begin{align*}
\frac{L}{2}\left\|\y_{t+1}-\y_t\right\|^2
&=
\frac{L}{2}
\left\|
-\gamma_t\odot g_t
+
\frac{\beta_1}{1-\beta_1}
(\gamma_{t-1}-\gamma_t)\odot m_{t-1}
\right\|^2 \\
&\le
L\|\gamma_t\odot g_t\|^2
+
\frac{\beta_1^2L}{(1-\beta_1)^2}
\|(\gamma_{t-1}-\gamma_t)\odot m_{t-1}\|^2 \\
&=
L\sum_{i=1}^{d}\gamma_{t,i}^2g_{t,i}^2
+
\frac{\beta_1^2L}{(1-\beta_1)^2}
\sum_{i=1}^{d}
(\gamma_{t-1,i}-\gamma_{t,i})^2m_{t-1,i}^2 .
\end{align*}
To obtain step (b), we use Young's inequality and \(L\)-smoothness:
\begin{align*}
&\sum_{i=1}^{d}
\left|(\nabla f(\y_t))_i-(\nabla f(\x_t))_i\right|
\left|\gamma_{t,i}g_{t,i}\right| \\
&\le
\frac{L}{2}\sum_{i=1}^{d}\gamma_{t,i}^{2}g_{t,i}^{2}
+
\frac{1}{2L}
\sum_{i=1}^{d}
\left((\nabla f(\x_t))_i-(\nabla f(\y_t))_i\right)^2 \\
&=
\frac{L}{2}\sum_{i=1}^{d}\gamma_{t,i}^{2}g_{t,i}^{2}
+
\frac{1}{2L}
\|\nabla f(\x_t)-\nabla f(\y_t)\|^2 \\
&\le
\frac{L}{2}\sum_{i=1}^{d}\gamma_{t,i}^{2}g_{t,i}^{2}
+
\frac{L}{2}\|\x_t-\y_t\|^2 \\
&=
\frac{L}{2}\sum_{i=1}^{d}\gamma_{t,i}^{2}g_{t,i}^{2}
+
\frac{\beta_1^2L}{2(1-\beta_1)^2}
\|\gamma_{t-1}\odot m_{t-1}\|^2 .
\end{align*}
We now further analyze \( W_t \) in \cref{adam_1} (up to \cref{fdsafdsafasdf}). We first have
\begin{align}\label{W}
    W_t&=-\sum_{i=1}^{d}\gamma_{t,i}(\nabla f(\x_{t}))_ig_{t,i}\notag\\&=-\sum_{i=1}^{d}\Expect\!\left[\gamma_{t,i}(\nabla f(\x_{t}))_ig_{t,i}|\mathscr{F}_{t-1}\right]\notag\\&\quad+\underbrace{\sum_{i=1}^{d}\Expect\!\left[\gamma_{t,i}(\nabla f(\x_{t}))_ig_{t,i}|\mathscr{F}_{t-1}\right]-\sum_{i=1}^{d}\gamma_{t,i}(\nabla f(\x_{t}))_ig_{t,i}}_{D_{t,1}}\notag\\&=-\sum_{i=1}^{d}\gamma_{t-1,i}(\nabla f(\x_{t}))^2_i+\underbrace{\sum_{i=1}^{d}\Expect\!\left[\left(\gamma_{t-1,i}-\gamma_{t,i}\right)(\nabla f(\x_{t}))_ig_{t,i}|\mathscr{F}_{t-1}\right]}_{W_{t,1}}+D_{t}.
\end{align}
In this way, we decompose \( W_t \) into the sum of an \( \mathscr{F}_{t-1} \)-measurable random variable \( W_{t,1} \) and a martingale difference sequence \( \{ (D_{t,1}, \mathscr{F}_t) \}_{t=1}^{T} \). 

Next, we bound \( W_{t,1} \). To this end, we first reformulate the difference \( \gamma_{t-1, i} - \gamma_{t, i} \) and obtain
\begin{align}\label{W_1}
&\gamma_{t-1,i}-\gamma_{t,i} \nonumber\\
&=\frac{\eta }{\sqrt{T}}\left(\frac{1}{\sqrt{v_{t-1,i}}+\epsilon}-\frac{1}{\sqrt{v_{t,i}}+\epsilon}\right)\notag\\
&=\underbrace{\frac{\eta }{\sqrt{T-1}} \frac{1}{\sqrt{v_{t-1,i}}+\epsilon}   -\frac{\eta }{\sqrt{T}}\frac{1}{\sqrt{v_{t,i}}+\epsilon}}_{\Delta_{t,1,i}}+\underbrace{\eta \left(\frac{1}{\sqrt{T}}-\frac{1}{\sqrt{T-1}}\right)\frac{1}{\sqrt{v_{t-1,i}}+\epsilon}}_{\Delta_{t,2,i}}.
\end{align}
Then we have
\begin{align}\label{sadam_0}
 W_{t,1}&\le\sum_{i=1}^{d}\Expect\!\left[|\Delta_{t,1,i}||(\nabla f(\x_{t}))_ig_{t,i}||\mathscr{F}_{t-1}\right]+\sum_{i=1}^{d}\Expect\!\left[\Delta_{t,2,i}(\nabla f(\x_{t}))_ig_{t,i}|\mathscr{F}_{t-1}\right]\notag\\
 &\mathop{\le}^{(a)}2^{1/4}\sum_{i=1}^{d}\sqrt{\gamma_{t-1,i}}|(\nabla f(\x_{t}))_i|\Expect\!\left[\sqrt{|\Delta_{t,1,i}|}\left|g_{t,i}\right|\big|\mathscr{F}_{t-1}\right]+\sum_{i=1}^{d}\Delta_{t,2,i}(\nabla f(\x_t))_i^2\notag\\
 &\mathop{\le}^{(b)}\frac{1}{2}\sum_{i=1}^{d}\gamma_{t-1,i}(\nabla f(\x_t))_i^{2}+\frac{1}{\sqrt{2}}\sum_{i=1}^{d}\Expect^2\left[\sqrt{|\Delta_{t,1,i}|}g_{t,i}|\mathscr{F}_{t-1}\right]\notag\\
 &\mathop{\le}^{(c)}\frac{1}{2}\sum_{i=1}^{d}\gamma_{t-1,i}(\nabla f(\x_t))_i^{2}+\sum_{i=1}^{d}\Expect\!\left[|\Delta_{t,1,i}||\mathscr{F}_{t-1}\right]\Expect\!\left[g^2_{t,i}|\mathscr{F}_{t-1}\right]\notag\\
 &\mathop{\le}^{(d)} \frac{1}{2}\sum_{i=1}^{d}\gamma_{t-1,i}(\nabla f(\x_t))_i^{2}+\sum_{i=1}^{d}\Expect\!\left[\|g_t-\nabla f(\x_t)\|^2|\mathscr{F}_{t-1}\right]\Expect\!\left[|\Delta_{t,1,i}||\mathscr{F}_{t-1}\right]\notag\\
 &\quad+\sum_{i=1}^{d}(\nabla f(\x_t))_i^2\Expect\!\left[|\Delta_{t,1,i}||\mathscr{F}_{t-1}\right]\notag\\
 &\mathop{\le}^{(e)}\frac{1}{2}\sum_{i=1}^{d}\gamma_{t-1,i}(\nabla f(\x_t))_i^{2}+\sum_{i=1}^{d}\left((\nabla f(\x_t))_i^2+C\right)\Expect\!\left[|\Delta_{t,1,i}||\mathscr{F}_{t-1}\right]\notag\\
 &\mathop{\le}^{(f)}\frac{3}{4}\sum_{i=1}^{d}\gamma_{t-1,i}(\nabla f(\x_t))_i^{2}+\left(\sum_{i=1}^{d}\gamma_{t-1,i}(\nabla f(\x_t))_i^2-\sum_{i=1}^{d}\gamma_{t,i}(\nabla f(\x_{t+1}))_i^2\right)\notag\\
 &\quad+\frac{35dL^2\eta}{v}\|\gamma_{t}\odot m_t\|^2+\underbrace{\sum_{i=1}^{d}(\nabla f(\x_t))_{i}^{2} \left(\Expect\left[|\Delta_{t,1,i}||\mathscr{F}_{t-1}\right]-|\Delta_{t,1,i}|\right)}_{D_{t,2}}\notag\\&\quad+C\sum_{i=1}^{d}\Expect\left[|\Delta_{t,1,i}||\mathscr{F}_{t-1}\right],
\end{align}
where (c) follows from Cauchy–Schwarz inequality, (e) applies  Assumption~\ref{ass:abc} to the second-term in the preceding step resulting
$\Expect\!\left[\|g_t-\nabla f(\x_t)\|^2|\mathscr{F}_{t-1}\right]\le C$, and the other steps (a), (b), (d), and (f) are explained below. We first observe that $\{(D_{t,2}, \mathscr{F}_t)\}_{1\le t\le T}$ is a martingale difference sequence.

In step (a) of the above derivation, the first term follows from the non-negativity of \( \Delta_{t,1,i} \) as follows. As shown in \cref{v_10}, we can immediately obtain
\begin{align*}
T v_{t,i} = (T-1)v_{t-1,i} + g_{t,i}^2 \ge (T-1)v_{t-1,i}.
\end{align*}
Therefore, the absolute value can be removed, i.e.,
\begin{align}\label{v_12}
|\Delta_{t,1,i}|
&=\left|\frac{\eta }{\sqrt{T-1}} \frac{1}{\sqrt{v_{t-1,i}}+\epsilon}   
     -\frac{\eta }{\sqrt{T}}\frac{1}{\sqrt{v_{t,i}}+\epsilon}\right| \notag\\
&=\frac{\eta }{\sqrt{T-1}} \frac{1}{\sqrt{v_{t-1,i}}+\epsilon}   
   -\frac{\eta }{\sqrt{T}}\frac{1}{\sqrt{v_{t,i}}+\epsilon} \notag\\
&=\Delta_{t,1,i}.
\end{align}
Then, we can obtain the first term in step (a) as follows:
\begin{align*}
|\Delta_{t,1,i}| 
&= \Delta_{t,1,i} 
= \sqrt{\Delta_{t,1,i}}\sqrt{\Delta_{t,1,i}}\\
&\le \sqrt{\frac{\eta }{\sqrt{T-1}}\frac{1}{\sqrt{v_{t-1,i}}+\epsilon}}
   \sqrt{|\Delta_{t,1,i}|}
\le 2^{1/4}\sqrt{\gamma_{t-1,i}}\sqrt{|\Delta_{t,1,i}|}.
\end{align*}
It is evident that \( \gamma_{t-1, i} \) is \( \mathscr{F}_{t-1} \)-measurable. Hence, by the property of conditional expectation, it can be taken outside the conditional expectation operator.  

To obtain the second term in step (a), we also note that \( \Delta_{t,i,2} \) is \( \mathscr{F}_{t-1} \)-measurable, and therefore it can likewise be taken outside the conditional expectation.

To obtain step~(b) in \cref{sadam_0}, we first note that
\(\Delta_{t,2,i}\le0\). Hence, the second term in the previous expression can be
relaxed directly to zero. For the first term, we use the following weighted
AM--GM inequality:
\[
ab
\le
\frac{\rho}{2}a^2+\frac{1}{2\rho}b^2,
\qquad a,b\ge0,\quad \rho>0.
\]
Applying this inequality with
\[
a:=\sqrt{\gamma_{t-1,i}}\left|(\nabla f(\x_t))_i\right|,
\qquad
b:=
\Expect\!\left[
\sqrt{|\Delta_{t,1,i}|}\,|g_{t,i}|
\,\middle|\,
\mathscr F_{t-1}
\right],
\qquad
\rho:=2^{-1/4},
\]
we obtain
\begin{align*}
&\quad
\sum_{i=1}^{d}
\sqrt{\gamma_{t-1,i}}
|(\nabla f(\x_t))_i|
\Expect\!\left[
\sqrt{|\Delta_{t,1,i}|}\,|g_{t,i}|
\,\middle|\,
\mathscr F_{t-1}
\right]
\\
&\le
\frac{1}{2^{5/4}}
\sum_{i=1}^{d}
\gamma_{t-1,i}(\nabla f(\x_t))_i^2
+
\frac{1}{2^{3/4}}
\sum_{i=1}^{d}
\Expect^2\!\left[
\sqrt{|\Delta_{t,1,i}|}\,|g_{t,i}|
\,\middle|\,
\mathscr F_{t-1}
\right].
\end{align*}
To obtain step (d), we apply the following derivation to the second term from the preceding step:
\begin{align*}
 \frac{1}{2}\sum_{i=1}^{d}\Expect\!\left[|\Delta_{t,1,i}||\mathscr{F}_{t-1}\right]\Expect[g_{t,i}^2|\mathscr{F}_{t-1}]&\le  \frac{1}{2}\left(\sum_{i=1}^{d}\Expect\!\left[|\Delta_{t,1,i}||\mathscr{F}_{t-1}\right]\right)\sum_{i=1}^{d}\Expect\left[\left(g_{t,i}-\left(\nabla f(\x_t)\right)_i^2\right)^2\big|\mathscr{F}_{t-1}\right]\\&\quad+\frac{1}{2}\left(\sum_{i=1}^{d}\Expect\!\left[|\Delta_{t,1,i}||\mathscr{F}_{t-1}\right]\right)\left(\nabla f(\x_t)\right)_i^2\\&\le  \frac{1}{2}\left(\sum_{i=1}^{d}\Expect\!\left[|\Delta_{t,1,i}||\mathscr{F}_{t-1}\right]\right)\Expect[\|g_{t}-\nabla f(\x_t)\|^2|\mathscr{F}_{t-1}] \\&\quad+\frac{1}{2}\left(\sum_{i=1}^{d}\Expect\!\left[|\Delta_{t,1,i}||\mathscr{F}_{t-1}\right]\right)\left(\nabla f(\x_t)\right)_i^2.
\end{align*}
%To obtain step (e), we handle the second from the preceding step. For the second term, we apply Assumption~\ref{ass:abc} to relax it, and obtain
%\begin{align*}
%    \Expect\!\left[\|g_t-\nabla f(\x_t)\|^2|\mathscr{F}_{t-1}\right]\le C.
%\end{align*}
To obtain step (f), we derive as follows:
\begin{align}\label{adam_adam_0'}
&\sum_{i=1}^{d}\left((\nabla f(\x_t))_i^2+C\right)\Expect\!\left[|\Delta_{t,1,i}||\mathscr{F}_{t-1}\right]\notag\\
&=\sum_{i=1}^{d}(\nabla f(\x_t))_i^2|\Delta_{t,1,i}| +\sum_{i=1}^{d}(\nabla f(\x_t))_{i}^{2} \left(\Expect\!\left[|\Delta_{t,1,i}||\mathscr{F}_{t-1}\right]-|\Delta_{t,1,i}|\right) +C\sum_{i=1}^{d}\Expect\!\left[|\Delta_{t,1,i}||\mathscr{F}_{t-1}\right]\notag\\
&\mathop{=}^{(i)}\sum_{i=1}^{d}(\nabla f(\x_t))_i^2\Delta_{t,1,i} +\sum_{i=1}^{d}(\nabla f(\x_t))_{i}^{2} \left(\Expect\!\left[|\Delta_{t,1,i}||\mathscr{F}_{t-1}\right]-|\Delta_{t,1,i}|\right) \notag\\&\quad+C\sum_{i=1}^{d}\Expect\!\left[|\Delta_{t,1,i}||\mathscr{F}_{t-1}\right].
\end{align}
where (i) follows from \cref{v_12}.

For the first term on the right-hand side, we further obtain
\begin{align}\label{adam_adam_1'}
\sum_{i=1}^{d}(\nabla f(\x_t))_i^2\Delta_{t,1,i}&=  \eta\sum_{i=1}^{d}(\nabla f(\x_t))_i^2\left(\frac{1}{\sqrt{T-1}}\frac{1}{\sqrt{v_{t-1,i}}+\epsilon}-\frac{1}{\sqrt{T}}\frac{1}{\sqrt{v_{t,i}}+\epsilon}\right)\notag\\&= \eta\sum_{i=1}^{d}(\nabla f(\x_t))_i^2\left(\frac{1}{\sqrt{T}}\frac{1}{\sqrt{v_{t-1,i}}+\epsilon}-\frac{1}{\sqrt{T}}\frac{1}{\sqrt{v_{t,i}}+\epsilon}\right)\notag\\&\quad +\eta\sum_{i=1}^{d}(\nabla f(\x_t))_i^2\left(\frac{1}{\sqrt{T-1}}-\frac{1}{\sqrt{T}}\right)\frac{1}{\sqrt{v_{t-1,i}}+\epsilon}\notag\\
&\mathop{\le}^{(i)}\eta\sum_{i=1}^{d}\left(\frac{1}{\sqrt{T}}\frac{(\nabla f(\x_t))_i^2}{\sqrt{v_{t-1,i}}+\epsilon}-\frac{1}{\sqrt{T}}\frac{(\nabla f(\x_{t+1}))_i^2}{\sqrt{v_{t,i}}+\epsilon}\right)\notag\\&\quad+\frac{1}{8}\sum_{i=1}^{d}\gamma_{t-1,i}(\nabla f(\x_{t}))_i^2+\frac{140dL^2\eta}{\sqrt v}\|\x_{t+1}-\x_t\|^2\notag\\&\quad +\eta\sum_{i=1}^{d}(\nabla f(\x_t))_i^2\left(\frac{1}{\sqrt{T-1}}-\frac{1}{\sqrt{T}}\right)\frac{1}{\sqrt{v_{t-1,i}}+\epsilon}\notag\\&\mathop{\le}^{(ii)}\eta\sum_{i=1}^{d}\left(\frac{1}{\sqrt{T}}\frac{(\nabla f(\x_t))_i^2}{\sqrt{v_{t-1,i}}+\epsilon}-\frac{1}{\sqrt{T}}\frac{(\nabla f(\x_{t+1}))_i^2}{\sqrt{v_{t,i}}+\epsilon}\right)\notag\\&\quad+\frac{1}{8}\sum_{i=1}^{d}\gamma_{t-1,i}(\nabla f(\x_{t}))_i^2+\frac{140dL^2\eta}{\sqrt v}\|\x_{t+1}-\x_t\|^2\notag\\&\quad +\frac{1}{10}\sum_{i=1}^{d}\gamma_{t-1,i}(\nabla f(\x_t))_i^2  \notag\\&\le \sum_{i=1}^{d}\gamma_{t-1,i}(\nabla f(\x_t))_i^2-\sum_{i=1}^{d}\gamma_{t,i}(\nabla f(\x_{t+1}))_i^2+\frac{1}{4}\sum_{i=1}^{d}\gamma_{t-1,i}(\nabla f(\x_t))_i^2\notag\\&\quad+\frac{140dL^2\eta}{\sqrt v}\|\gamma_{t}\odot m_t\|^2.
\end{align}
To obtain step (i) in the above bound, using the \(L\)-smoothness of \(f\), for every coordinate \(i\), we have
\begin{align}\label{eq:gradstep1}
&-\frac{\eta}{\sqrt{T}}
\frac{(\nabla f(\x_t))_i^2}{\sqrt{v_{t,i}}+\epsilon} \nonumber\\
&=
-\frac{\eta}{\sqrt{T}}
\frac{
\left(
(\nabla f(\x_{t+1}))_i
+
(\nabla f(\x_t))_i
-
(\nabla f(\x_{t+1}))_i
\right)^2
}{\sqrt{v_{t,i}}+\epsilon}
\nonumber \\
&\le
-\frac{\eta}{\sqrt{T}}
\frac{(\nabla f(\x_{t+1}))_i^2}{\sqrt{v_{t,i}}+\epsilon}
+
\frac{2\eta}{\sqrt{T}}
\frac{
\left|(\nabla f(\x_{t+1}))_i\right|
\left|
(\nabla f(\x_t))_i-(\nabla f(\x_{t+1}))_i
\right|
}{\sqrt{v_{t,i}}+\epsilon}
\nonumber\\
&\quad
+
\frac{\eta}{\sqrt{T}}
\frac{
\left|
(\nabla f(\x_t))_i-(\nabla f(\x_{t+1}))_i
\right|^2
}{\sqrt{v_{t,i}}+\epsilon}
\nonumber\\
&\le
-\frac{\eta}{\sqrt{T}}
\frac{(\nabla f(\x_{t+1}))_i^2}{\sqrt{v_{t,i}}+\epsilon}
+
\frac{2L\eta}{\sqrt{T}}
\frac{
\left|(\nabla f(\x_t))_i\right|
\|\x_{t+1}-\x_t\|
}{\sqrt{v_{t,i}}+\epsilon}
+
\frac{3L^2\eta}{\sqrt{T}}
\frac{
\|\x_{t+1}-\x_t\|^2
}{\sqrt{v_{t,i}}+\epsilon}
\nonumber \\
&\le
-\frac{\eta}{\sqrt{T}}
\frac{(\nabla f(\x_{t+1}))_i^2}{\sqrt{v_{t,i}}+\epsilon}
+
\frac{\eta}{16\sqrt{T}}
\frac{(\nabla f(\x_t))_i^2}{\sqrt{v_{t,i}}+\epsilon}
+
\frac{70L^2\eta}{\sqrt{T}}
\frac{
\|\x_{t+1}-\x_t\|^2
}{\sqrt{v_{t,i}}+\epsilon} \nonumber \\
&\le
-\frac{\eta}{\sqrt{T}}
\frac{(\nabla f(\x_{t+1}))_i^2}{\sqrt{v_{t,i}}+\epsilon}
+
\frac18\gamma_{t-1,i}(\nabla f(\x_t))_i^2
+
70L^2\gamma_{t,i}
\|\x_{t+1}-\x_t\|^2
\end{align}
where in the last inequality, the second term can be derived by the recursion of \(v_t\),
\[
v_{t,i}
=
\left(1-\frac1T\right)v_{t-1,i}
+
\frac1T g_{t,i}^2
\ge
\left(1-\frac1T\right)v_{t-1,i}.
\]
which, due to \(T\ge2\), implies
\[
\sqrt{v_{t,i}}+\epsilon
\ge
\sqrt{1-\frac1T}\,\sqrt{v_{t-1,i}}+\epsilon
\ge
\frac12\bigl(\sqrt{v_{t-1,i}}+\epsilon\bigr),
\]
yielding
\[
\frac{1}{\sqrt{v_{t,i}}+\epsilon}
\le
\frac{2}{\sqrt{v_{t-1,i}}+\epsilon},
\]
and hence
\[
\frac{\eta}{16\sqrt T}
\frac{(\nabla f(\x_t))_i^2}{\sqrt{v_{t,i}}+\epsilon}
\le
\frac{\eta}{8\sqrt T}
\frac{(\nabla f(\x_t))_i^2}{\sqrt{v_{t-1,i}}+\epsilon} 
\le 
\frac{\eta}{8\sqrt{T-1}}
\frac{(\nabla f(\x_t))_i^2}{\sqrt{v_{t-1,i}}+\epsilon} 
=
\frac18\gamma_{t-1,i}(\nabla f(\x_t))_i^2.
\]
%where the last inequality uses
%\[\frac{\eta}{\sqrt T}\frac{1}{\sqrt{v_{t-1,i}}+\epsilon}\le\frac{\eta}{\sqrt{T-1}}\frac{1}{\sqrt{v_{t-1,i}}+\epsilon}=\gamma_{t-1,i}.\]
After summing over \(i\), the third term in the last inequality of \cref{eq:gradstep1} is bounded as follows. Since \[ v_{t,i}\ge (1-1/T)^t v\ge v/4\] for $t\le T$,
\[
\sum_{i=1}^{d}\gamma_{t,i}
\le
\frac{2d\eta}{\sqrt v},
\]
and using \(\x_{t+1}-\x_t=-\gamma_t\odot m_t\), the accumulated smoothness error is bounded by
\[
70L^2
\sum_{i=1}^{d}\gamma_{t,i}
\|\x_{t+1}-\x_t\|^2
\le
\frac{140dL^2\eta}{\sqrt v}
\|\gamma_t\odot m_t\|^2.
\]
For the step (ii) in \cref{adam_adam_1'}, since \(T\ge10\), we have
\[
\sqrt{T-1}
\left(
\frac{1}{\sqrt{T-1}}-\frac{1}{\sqrt T}
\right)
\le
\frac{1}{10},
\]
and therefore
\begin{align}\label{fdsafdsafasdf}
&
\eta
\sum_{i=1}^{d}
(\nabla f(\x_t))_i^2
\left(
\frac{1}{\sqrt{T-1}}-\frac{1}{\sqrt T}
\right)
\frac{1}{\sqrt{v_{t-1,i}}+\epsilon}
\notag\\
&\qquad \le
\frac{1}{10}
\sum_{i=1}^{d}
\gamma_{t-1,i}(\nabla f(\x_t))_i^2 .
\end{align}
 This completes the explanation of \cref{adam_adam_1'}.

Then substituting \cref{adam_adam_1'} back into \cref{adam_adam_0'}, and inserting the resulting expression into \cref{sadam_0}, completes step (f) in \cref{sadam_0}.

\noindent We then substitute the bound on \(W_{t,1}\) in \cref{sadam_0} back into \cref{W}, and obtain
\begin{align}\label{power_100}
W_t
&\le
-\frac14
\sum_{i=1}^d
\gamma_{t-1,i}(\nabla f(\x_t))_i^2
+
D_{t,1}
+
D_{t,2}
\notag\\
&\quad
+
\left(
\sum_{i=1}^{d}
\gamma_{t-1,i}(\nabla f(\x_t))_i^2
-
\sum_{i=1}^{d}
\gamma_{t,i}(\nabla f(\x_{t+1}))_i^2
\right)
\notag\\
&\quad
+
\frac{140dL^2\eta}{\sqrt v}
\|\gamma_t\odot m_t\|^2
+
C\sum_{i=1}^{d}
\Expect\!\left[
|\Delta_{t,1,i}|
\,\middle|\,
\mathscr F_{t-1}
\right].
\end{align}
\noindent We next decompose the term \(H_t\) in \cref{adam_1} into a martingale-difference term and a predictable drift term. By adding and subtracting its conditional expectation, we have
\begin{align}\label{H_decomposition}
H_t
=
\underbrace{
H_t-\mathbb E[H_t\mid \mathscr F_{t-1}]
}_{D_{t,3}}
+
\underbrace{
\mathbb E[H_t\mid \mathscr F_{t-1}]
}_{H_{t,1}} .
\end{align}
Equivalently,
\begin{align}\label{D_t3_def}
D_{t,3}
:=
\frac{\beta_1}{1-\beta_1}
\sum_{i=1}^{d}
\left[
(\gamma_{t-1,i}-\gamma_{t,i})
-
\mathbb E[
\gamma_{t-1,i}-\gamma_{t,i}
\mid \mathscr F_{t-1}
]
\right]
(\nabla f(\y_t))_i m_{t-1,i},
\end{align}
and
\begin{align}\label{H_t1_def}
H_{t,1}
:=
\frac{\beta_1}{1-\beta_1}
\sum_{i=1}^{d}
\mathbb E[
\gamma_{t-1,i}-\gamma_{t,i}
\mid \mathscr F_{t-1}
]
(\nabla f(\y_t))_i m_{t-1,i}.
\end{align}
Since \(\y_t\), \(m_{t-1}\), and \(\gamma_{t-1}\) are \(\mathscr F_{t-1}\)-measurable, it follows that
\[
\mathbb E[D_{t,3}\mid \mathscr F_{t-1}]=0.
\]
Thus \(\{(D_{t,3},\mathscr F_t)\}_{t\ge 2}\) is a martingale difference sequence, while \(H_{t,1}\) is the corresponding predictable drift term.

\noindent
It remains to control \(H_{t,1}\). To this end, we use the decomposition of the consecutive stepsize difference in \cref{W_1}, namely
\[
\gamma_{t-1,i}-\gamma_{t,i}
=
\Delta_{t,1,i}+\Delta_{t,2,i}.
\]
Here \(\Delta_{t,1,i}\ge 0\), while \(\Delta_{t,2,i}\) is \(\mathscr F_{t-1}\)-measurable. Therefore,
\begin{align*}
H_{t,1}
&=
\frac{\beta_1}{1-\beta_1}
\sum_{i=1}^{d}
\mathbb E[
\Delta_{t,1,i}+\Delta_{t,2,i}
\mid \mathscr F_{t-1}
]
(\nabla f(\y_t))_i m_{t-1,i} \\
&\le
\frac{\beta_1}{1-\beta_1}
\sum_{i=1}^{d}
\mathbb E[
\Delta_{t,1,i}
\mid \mathscr F_{t-1}
]
\left((\nabla f(\y_t))_i m_{t-1,i}\right)_+ \\
&\quad
+
\frac{\beta_1}{1-\beta_1}
\sum_{i=1}^{d}
\mathbb E[
|\Delta_{t,2,i}|
\mid \mathscr F_{t-1}
]
\left|(\nabla f(\y_t))_i m_{t-1,i}\right|,
\end{align*}
where the first term of the inequality uses the nonnegativity of \(\Delta_{t,1,i}\), and the second term is controlled by \(|\Delta_{t,2,i}|\), because the coefficient
\((\nabla f(\y_t))_i m_{t-1,i}\) has no fixed sign.

\noindent We first bound the term involving \(\Delta_{t,1,i}\). We claim that
\begin{align}\label{Delta_t1_cond_bound_refined}
\mathbb E[\Delta_{t,1,i}\mid \mathscr F_{t-1}]
\le
\frac{T}{T-1}\,
\frac{|(\nabla f(\x_t))_i|+\sqrt C+\frac{\epsilon}{\sqrt T+\sqrt{T-1}}}{\eta}
\gamma_{t-1,i}^2 .
\end{align}
Indeed, by the definition of \(\Delta_{t,1,i}\) in \cref{W_1},
\begin{align*}
\Delta_{t,1,i}
&=
\frac{\eta}{\sqrt{T-1}}
\frac{1}{\sqrt{v_{t-1,i}}+\epsilon}
-
\frac{\eta}{\sqrt T}
\frac{1}{\sqrt{v_{t,i}}+\epsilon} \\
&=
\eta
\frac{
\sqrt T(\sqrt{v_{t,i}}+\epsilon)
-
\sqrt{T-1}(\sqrt{v_{t-1,i}}+\epsilon)
}{
\sqrt T\sqrt{T-1}
(\sqrt{v_{t-1,i}}+\epsilon)(\sqrt{v_{t,i}}+\epsilon)
}.
\end{align*}
Since
\[
v_{t,i}
=
\left(1-\frac1T\right)v_{t-1,i}
+
\frac1Tg_{t,i}^2,
\]
we have
\[
\sqrt T\sqrt{v_{t,i}}
=
\sqrt{(T-1)v_{t-1,i}+g_{t,i}^2}
\le
\sqrt{T-1}\sqrt{v_{t-1,i}}+|g_{t,i}|.
\]
Moreover,
\[
\sqrt T\epsilon-\sqrt{T-1}\epsilon
=
\frac{\epsilon}{\sqrt T+\sqrt{T-1}}.
\]
Thus,
\[
\sqrt T(\sqrt{v_{t,i}}+\epsilon)
-
\sqrt{T-1}(\sqrt{v_{t-1,i}}+\epsilon)
\le
|g_{t,i}|
+
\frac{\epsilon}{\sqrt T+\sqrt{T-1}}.
\]
On the other hand, since \(v_{t,i}\ge (1-\frac1T)v_{t-1,i}\), we have
\[
\sqrt{v_{t,i}}+\epsilon
\ge
\sqrt{1-\frac1T}\,(\sqrt{v_{t-1,i}}+\epsilon).
\]
Therefore,
\begin{align*}
\Delta_{t,1,i}
&\le
\eta
\frac{
|g_{t,i}|
+
\frac{\epsilon}{\sqrt T+\sqrt{T-1}}
}{
(T-1)(\sqrt{v_{t-1,i}}+\epsilon)^2
} \\
&=
\frac{T}{T-1}
\frac{
|g_{t,i}|
+
\frac{\epsilon}{\sqrt T+\sqrt{T-1}}
}{\eta}
\gamma_{t-1,i}^2 .
\end{align*}
Taking conditional expectation and using
\[
\mathbb E[|g_{t,i}|\mid \mathscr F_{t-1}]
\le
|(\nabla f(\x_t))_i|+\sqrt C
\]
gives \cref{Delta_t1_cond_bound_refined}.

\noindent Using the nonnegativity of 
\(\mathbb E[\Delta_{t,1,i}\mid \mathscr F_{t-1}]\), we apply Young's inequality as
\begin{align*}
&\frac{\beta_1}{1-\beta_1}
\mathbb E[\Delta_{t,1,i}\mid \mathscr F_{t-1}]
\left((\nabla f(\y_t))_i m_{t-1,i}\right)_+ \\
&\le
\frac{\beta_1}{1-\beta_1}
\mathbb E[\Delta_{t,1,i}\mid \mathscr F_{t-1}]
\left(
a(\nabla f(\y_t))_i^2
+
\frac{1}{4a}m_{t-1,i}^2
\right),
\end{align*}
where we choose the smoothed coefficient
\[
a:=\frac{1-\beta_1}{32(1+\beta_1)}.
\]
We now bound the two terms separately. First, since
\[
0\le \Delta_{t,1,i}
\le
\frac{\eta}{\sqrt{T-1}}
\frac{1}{\sqrt{v_{t-1,i}}+\epsilon}
=
\sqrt{\frac{T}{T-1}}\gamma_{t-1,i}
\le
2\gamma_{t-1,i},
\]
we have
\[
\mathbb E[\Delta_{t,1,i}\mid \mathscr F_{t-1}]
\le
2\gamma_{t-1,i}.
\]
Therefore, with the above choice of \(a\),
\begin{align*}
&\frac{\beta_1}{1-\beta_1}
a\,
\mathbb E[\Delta_{t,1,i}\mid \mathscr F_{t-1}]
(\nabla f(\y_t))_i^2 \\
&\le
\frac{\beta_1}{1-\beta_1}
\frac{1-\beta_1}{32(1+\beta_1)}
\cdot
2\gamma_{t-1,i}
(\nabla f(\y_t))_i^2 \\
&=
\frac{\beta_1}{16(1+\beta_1)}
\gamma_{t-1,i}(\nabla f(\y_t))_i^2 \\
&\le
\frac{1}{16}
\gamma_{t-1,i}(\nabla f(\y_t))_i^2.
\end{align*}

For the second term, we use the sharper estimate
\cref{Delta_t1_cond_bound_refined}. Since
\[
\frac{1}{4a}
=
8\frac{1+\beta_1}{1-\beta_1},
\]
we obtain
\begin{align*}
&\frac{\beta_1}{1-\beta_1}
\frac{1}{4a}
\mathbb E[\Delta_{t,1,i}\mid \mathscr F_{t-1}]
m_{t-1,i}^2 \\
&=
8\frac{\beta_1(1+\beta_1)}{(1-\beta_1)^2}
\mathbb E[\Delta_{t,1,i}\mid \mathscr F_{t-1}]
m_{t-1,i}^2 \\
&\le
8\frac{\beta_1(1+\beta_1)}{(1-\beta_1)^2}
\frac{T}{T-1}
\frac{
|(\nabla f(\x_t))_i|+\sqrt C+\frac{\epsilon}{\sqrt T+\sqrt{T-1}}
}{\eta}
\gamma_{t-1,i}^2
m_{t-1,i}^2.
\end{align*}
Consequently,
\begin{align}\label{Delta_t1_part_bound}
&\frac{\beta_1}{1-\beta_1}
\sum_{i=1}^{d}
\mathbb E[\Delta_{t,1,i}\mid \mathscr F_{t-1}]
\left((\nabla f(\y_t))_i m_{t-1,i}\right)_+
\notag\\
&\le
\frac{1}{16}
\sum_{i=1}^{d}
\gamma_{t-1,i}(\nabla f(\y_t))_i^2
\notag\\
&\quad
+
8\frac{\beta_1(1+\beta_1)}{(1-\beta_1)^2}
\frac{T}{T-1}
\sum_{i=1}^{d}
\frac{
|(\nabla f(\x_t))_i|+\sqrt C+\frac{\epsilon}{\sqrt T+\sqrt{T-1}}
}{\eta}
\gamma_{t-1,i}^2
m_{t-1,i}^2 .
\end{align}

\noindent Next, we control the term involving \(\Delta_{t,2,i}\). Since \(\Delta_{t,2,i}\) is \(\mathscr F_{t-1}\)-measurable, we have
\[
\mathbb E[|\Delta_{t,2,i}|\mid \mathscr F_{t-1}]
=
|\Delta_{t,2,i}|.
\]
By the definition of \(\Delta_{t,2,i}\) in \cref{W_1},
\begin{align}\label{Delta_t2_abs_bound_refined}
|\Delta_{t,2,i}|
&=
\eta\left(
\frac{1}{\sqrt{T-1}}-\frac{1}{\sqrt T}
\right)
\frac{1}{\sqrt{v_{t-1,i}}+\epsilon}
\notag\\
&=
\left(
\sqrt{\frac{T}{T-1}}-1
\right)
\gamma_{t-1,i}.
\end{align}
Moreover, since
\[
\sqrt{\frac{T}{T-1}}-1
=
\frac{1}{\sqrt{T-1}(\sqrt T+\sqrt{T-1})}
\le
\frac{1}{T},
\]
we further have
\begin{align}\label{Delta_t2_gamma_bound}
|\Delta_{t,2,i}|
\le
\frac{1}{T}\gamma_{t-1,i}.
\end{align}
Substituting \cref{Delta_t1_part_bound} and \cref{Delta_t2_gamma_bound} into the previous bound on \(H_{t,1}\), we obtain
\begin{align}\label{H_t1_bound_final}
H_{t,1}
&\le
\frac{1}{16}
\sum_{i=1}^{d}
\gamma_{t-1,i}(\nabla f(\y_t))_i^2
\notag\\
&\quad
+
8\frac{\beta_1(1+\beta_1)}{(1-\beta_1)^2}
\frac{T}{T-1}
\sum_{i=1}^{d}
\frac{
|(\nabla f(\x_t))_i|+\sqrt C+\frac{\epsilon}{\sqrt T+\sqrt{T-1}}
}{\eta}
\gamma_{t-1,i}^2
m_{t-1,i}^2
\notag\\
&\quad
+
\frac{\beta_1}{1-\beta_1}
\frac{1}{T}
\sum_{i=1}^{d}
\gamma_{t-1,i}
\left|(\nabla f(\y_t))_i m_{t-1,i}\right|.
\end{align}
Combining \cref{adam_1}, \cref{power_100}, \cref{H_decomposition}, and \cref{H_t1_bound_final}, and defining
\begin{align}\label{D_t_combined_def}
D_t:=D_{t,1}+D_{t,2}+D_{t,3},
\end{align}
we obtain
\begin{align}\label{adam_1_final}
&f(\y_{t+1})-f(\y_t)
\notag\\
&\le
-\frac{1}{4}\sum_{i=1}^d\gamma_{t-1,i}(\nabla f(\x_t))_i^2
+
\frac{1}{16}
\sum_{i=1}^{d}
\gamma_{t-1,i}(\nabla f(\y_t))_i^2
+
D_t
\notag\\
&\quad
+
\left(
\sum_{i=1}^{d}\gamma_{t-1,i}(\nabla f(\x_t))_i^2
-
\sum_{i=1}^{d}\gamma_{t,i}(\nabla f(\x_{t+1}))_i^2
\right)
\notag\\
&\quad
+
\frac{\beta_1^2L}{2(1-\beta_1)^2}
\|\gamma_{t-1}\odot m_{t-1}\|^2
+
\frac{3L}{2}\sum_{i=1}^{d}\gamma_{t,i}^2g_{t,i}^2
+
\frac{\beta_1^2L}{(1-\beta_1)^2}
\sum_{i=1}^{d}
(\gamma_{t-1,i}-\gamma_{t,i})^2m_{t-1,i}^2
\notag\\
&\quad
+
\frac{140dL^2\eta}{\sqrt v}\|\gamma_t\odot m_t\|^2
+
C\sum_{i=1}^{d}\Expect\left[|\Delta_{t,1,i}|\mid\mathscr F_{t-1}\right]
\notag\\
&\quad
+
8\frac{\beta_1(1+\beta_1)}{(1-\beta_1)^2}
\frac{T}{T-1}
\sum_{i=1}^{d}
\frac{
|(\nabla f(\x_t))_i|+\sqrt C+\frac{\epsilon}{\sqrt T+\sqrt{T-1}}
}{\eta}
\gamma_{t-1,i}^2
m_{t-1,i}^2
\notag\\
&\quad
+
\frac{\beta_1}{1-\beta_1}
\frac{1}{T}
\sum_{i=1}^{d}
\gamma_{t-1,i}
\left|(\nabla f(\y_t))_i m_{t-1,i}\right|.
\end{align}
Since \(D_{t,1}\), \(D_{t,2}\), and \(D_{t,3}\) are martingale difference terms with respect to \(\mathscr F_t\), their sum \(D_t\) also satisfies
\[
\Expect[D_t\mid \mathscr F_{t-1}]=0.
\]
Thus \(\{(D_t,\mathscr F_t)\}_{t\ge 2}\) is a martingale difference sequence.
We further convert the term involving \(\nabla f(\y_t)\) back to \(\nabla f(\x_t)\). 
By \((a+b)^2\le 2a^2+2b^2\), we have
\begin{align*}
\frac{1}{16}
\sum_{i=1}^{d}
\gamma_{t-1,i}(\nabla f(\y_t))_i^2
&\le
\frac{1}{8}
\sum_{i=1}^{d}
\gamma_{t-1,i}(\nabla f(\x_t))_i^2
\\
&\quad
+
\frac{1}{8}
\sum_{i=1}^{d}
\gamma_{t-1,i}
\left((\nabla f(\y_t))_i-(\nabla f(\x_t))_i\right)^2 .
\end{align*}
For the second term, using the coordinatewise bound
\[
\gamma_{t-1,i}
=
\frac{\eta}{\sqrt T}
\frac{1}{\sqrt{v_{t-1,i}}+\epsilon}
\le
\frac{2\eta}{\sqrt v},
\]
we obtain
\begin{align*}
\frac{1}{8}
\sum_{i=1}^{d}
\gamma_{t-1,i}
\left((\nabla f(\y_t))_i-(\nabla f(\x_t))_i\right)^2
&\le
\frac{\eta}{4\sqrt v}
\|\nabla f(\y_t)-\nabla f(\x_t)\|^2 \\
&\le
\frac{\eta L^2}{4\sqrt v}
\|\y_t-\x_t\|^2.
\end{align*}
Moreover, by the definition of \(\y_t\),
\[
\y_t-\x_t
=
-\frac{\beta_1}{1-\beta_1}
\gamma_{t-1}\odot m_{t-1}.
\]
Therefore,
\begin{align}\label{yt_to_xt_grad_bound}
\frac{1}{16}
\sum_{i=1}^{d}
\gamma_{t-1,i}(\nabla f(\y_t))_i^2
&\le
\frac{1}{8}
\sum_{i=1}^{d}
\gamma_{t-1,i}(\nabla f(\x_t))_i^2
+
\frac{\beta_1^2L^2\eta}{4(1-\beta_1)^2\sqrt v}
\|\gamma_{t-1}\odot m_{t-1}\|^2 .
\end{align}
Substituting \cref{yt_to_xt_grad_bound} into \cref{adam_1_final}, we obtain
\begin{align}\label{adam_1_final_refined}
f(\y_{t+1})-f(\y_t)
&\le
-\frac{1}{8}\sum_{i=1}^d\gamma_{t-1,i}(\nabla f(\x_t))_i^2
+
D_t
\notag\\
&\quad
+
\left(
\sum_{i=1}^{d}\gamma_{t-1,i}(\nabla f(\x_t))_i^2
-
\sum_{i=1}^{d}\gamma_{t,i}(\nabla f(\x_{t+1}))_i^2
\right)
\notag\\
&\quad
+
\left(
\frac{\beta_1^2L}{2(1-\beta_1)^2}
+
\frac{\beta_1^2L^2\eta}{4(1-\beta_1)^2\sqrt v}
\right)
\|\gamma_{t-1}\odot m_{t-1}\|^2
\notag\\
&\quad
+
\frac{3L}{2}\sum_{i=1}^{d}\gamma_{t,i}^2g_{t,i}^2
+
\frac{\beta_1^2L}{(1-\beta_1)^2}
\sum_{i=1}^{d}
(\gamma_{t-1,i}-\gamma_{t,i})^2m_{t-1,i}^2
\notag\\
&\quad
+
\frac{140dL^2\eta}{\sqrt v}\|\gamma_t\odot m_t\|^2
+
C\sum_{i=1}^{d}\Expect\left[|\Delta_{t,1,i}|\mid\mathscr F_{t-1}\right]
\notag\\
&\quad
+
8\frac{\beta_1(1+\beta_1)}{(1-\beta_1)^2}
\frac{T}{T-1}
\sum_{i=1}^{d}
\frac{
|(\nabla f(\x_t))_i|+\sqrt C+\frac{\epsilon}{\sqrt T+\sqrt{T-1}}
}{\eta}
\gamma_{t-1,i}^2
m_{t-1,i}^2
\notag\\
&\quad
+
\frac{\beta_1}{1-\beta_1}
\frac{1}{T}
\sum_{i=1}^{d}
\gamma_{t-1,i}
\left|(\nabla f(\y_t))_i m_{t-1,i}\right|.
\end{align}
Moving the telescoping term to the left-hand side of \cref{adam_1_final_refined}, we obtain
\begin{align}\label{adam_1_final_potential}
&\quad
f(\y_{t+1})
+
\sum_{i=1}^{d}
\gamma_{t,i}(\nabla f(\x_{t+1}))_i^2
-
\left(
f(\y_t)
+
\sum_{i=1}^{d}
\gamma_{t-1,i}(\nabla f(\x_t))_i^2
\right)
\notag\\
&\le
-\frac{1}{8}
\sum_{i=1}^{d}
\gamma_{t-1,i}(\nabla f(\x_t))_i^2
+
D_t
+
P_t,
\end{align}
where \(D_t\) is the martingale-difference term defined in \cref{D_t_combined_def}, and \(P_t\) collects the remaining second-order error terms:
\begin{align}
%\label{P_t_def}
P_t
:=
&
\left(
\frac{\beta_1^2L}{2(1-\beta_1)^2}
+
\frac{\beta_1^2L^2\eta}{4(1-\beta_1)^2\sqrt v}
\right)
\|\gamma_{t-1}\odot m_{t-1}\|^2
\notag\\
&+
\frac{3L}{2}
\sum_{i=1}^{d}
\gamma_{t,i}^2g_{t,i}^2
+
\frac{\beta_1^2L}{(1-\beta_1)^2}
\sum_{i=1}^{d}
(\gamma_{t-1,i}-\gamma_{t,i})^2m_{t-1,i}^2
\notag\\
&+
\frac{140dL^2\eta}{\sqrt v}\|\gamma_t\odot m_t\|^2
+
C\sum_{i=1}^{d}\Expect\left[|\Delta_{t,1,i}|\mid\mathscr F_{t-1}\right]
\notag\\
&+
8\frac{\beta_1(1+\beta_1)}{(1-\beta_1)^2}
\frac{T}{T-1}
\sum_{i=1}^{d}
\frac{
|(\nabla f(\x_t))_i|+\sqrt C+\frac{\epsilon}{\sqrt T+\sqrt{T-1}}
}{\eta}
\gamma_{t-1,i}^2
m_{t-1,i}^2
\notag\\
&+
\frac{\beta_1}{1-\beta_1}
\frac{1}{T}
\sum_{i=1}^{d}
\gamma_{t-1,i}
\left|(\nabla f(\y_t))_i m_{t-1,i}\right|.
\end{align}

\noindent With this, we complete the proof.
\end{proof}
\subsection{The Proof of Lemma \ref{lemma_sum_two} (Formal Version of Lemma~\ref{lemma_sum_two'})}\label{p_lemma_sum_two}
\begin{proof}
By a straightforward computation, we obtain
\begin{align}\label{super_adam_10}
\sum_{t=1}^{\mu\wedge T}\sum_{i=1}^{d}\gamma_{t,i}^2g_{t,i}^2=\sum_{t=1}^{\mu \wedge T}\sum_{i=1}^d\frac{\eta ^2}{T}\frac{g_{t,i}^2}{\left(\sqrt{v_{t,i}}+\epsilon\right)^2}   \le \eta ^2\sum_{t=1}^{\mu\wedge T}\sum_{i=1}^{d}\frac{1}{T}\frac{g_{t,i}^2}{v_{t,i}} .
\end{align}   
Recalling \cref{v_11}, we obtain
\begin{align*}
v_{t,i}&=\left(1-\frac{1}{T}\right)^t v_{0,i}+\frac{1}{T}\sum_{s=1}^{t}\left(1-\frac{1}{T}\right)^{t-s}g_{s,i}^2\\&\mathop{\ge}^{(a)} \frac{v}{4}+\frac{1}{4T}\sum_{s=1}^{t}g_{s,i}^2=\frac{1}{4T}\left(vT+\sum_{s=1}^{t}g_{s,i}^2\right)\\
&\overset{(b)}{=}\frac{1}{4T}\left(S_{t,i}+(T-1)v\right),
\end{align*}
where (a) applies the following inequality, valid for all integers $T \ge 2$ and $t \in [1, T]$:
\[
\left(1-\frac{1}{T}\right)^t 
\ge \left(1-\frac{1}{T}\right)^T 
\ge \frac{1}{4},
\]
and the quantity $S_{t,i}$ in (b) is defined in \cref{S_v}.

Substituting the above lower bound of $v_{t,i}$ into \cref{super_adam_10}, we obtain
\begin{align*}
\sum_{t=1}^{\mu\wedge T}\sum_{i=1}^{d}\gamma_{t,i}^2g_{t,i}^2&\le 4\eta ^2\sum_{i=1}^d\sum_{t=1}^{\mu\wedge T}\frac{g_{t,i}^2}{S_{t,i}+(T-1)v}\le 4\eta^2\sum_{i=1}^d\int_{Tv}^{S_{\mu\wedge T,i}+(T-1)v}\frac{1}{x}\text{d}x\\&\le 4\eta ^2\sum_{i=1}^{d}\log\left(1+\frac{S_{\mu \wedge T,i}}{Tv}\right)\le4\eta ^2d\log\left(1+\frac{S_{\mu\wedge T}}{vdT}\right),
\end{align*}
where $S_T:=\sum_{i=1}^dS_{T,i}$, and the last inequality follows from Jensen's inequality.

Thus, the first inequality is proved. 
For the second inequality, we can immediately obtain the result from Lemma~\ref{property_3.5}.
\noindent With this, we complete the proof.
\end{proof}
\subsection{The Proof of Lemma \ref{lem_important_-1}}\label{p_lem_important_-1}
\begin{proof}
We first consider the case \(s=1\). Since
\[
Y_{\mu\wedge n}
=
\sum_{k=1}^{n}\mathbf 1_{\{k\le \mu\wedge n\}}\mathbb E[Z_k\mid \mathcal F_{k-1}],
\]
and \(\mathbf 1_{\{k\le \mu\wedge n\}}\in\mathcal F_{k-1}\), the tower property yields
\begin{align*}
\mathbb E[Y_{\mu\wedge n}]
&=
\sum_{k=1}^{n}
\mathbb E\!\left[
\mathbf 1_{\{k\le \mu\wedge n\}}\mathbb E[Z_k\mid \mathcal F_{k-1}]
\right] \\
&=
\sum_{k=1}^{n}
\mathbb E\!\left[
\mathbf 1_{\{k\le \mu\wedge n\}} Z_k
\right]
=
\mathbb E[X_{\mu\wedge n}].
\end{align*}
Hence
\[
\|Y_{\mu\wedge n}\|_{L^1}\le \|X_{\mu\wedge n}\|_{L^1}.
\]
Now assume \(1<s<\infty\), and let \(r=\frac{s}{s-1}\) be the conjugate exponent of \(s\). By the standard \(L^s\)–\(L^r\) duality for nonnegative random variables,
\[
\|Y_{\mu\wedge n}\|_{L^s}
=
\sup\bigl\{\mathbb E[U\,Y_{\mu\wedge n}]: U\ge 0,\ \|U\|_{L^r}=1\bigr\}.
\]
Fix any \(U\ge 0\) with \(\|U\|_{L^r}=1\), define the martingale
\[
M_k:=\mathbb E[U\mid \mathcal F_k],\qquad k\ge 0.
\]
Then
\begin{align*}
\mathbb E[U\,Y_{\mu\wedge n}]
&=
\sum_{k=1}^{n}
\mathbb E\!\left[
U\,\mathbf 1_{\{k\le \mu\wedge n\}}\,\mathbb E[Z_k\mid \mathcal F_{k-1}]
\right] \\
&=
\sum_{k=1}^{n}
\mathbb E\!\left[
M_{k-1}\,\mathbf 1_{\{k\le \mu\wedge n\}}\, Z_k
\right].
\end{align*}
Since \(Z_k\ge 0\), we have
\[
\mathbb E[U\,Y_{\mu\wedge n}]
\le
\mathbb E\!\left[
M_n^*\,X_{\mu\wedge n}
\right],
\quad \text{where}\;\;
M_n^*:=\sup_{0\le j\le n} M_j.
\]
Applying Hölder’s inequality gives
\[
\mathbb E[M_n^*\,X_{\mu\wedge n}]
\le
\|M_n^*\|_{L^r}\,\|X_{\mu\wedge n}\|_{L^s}.
\]
By Doob’s maximal inequality for \(L^r\)-martingales, we have
\[
\|M_n^*\|_{L^r}
\le
\frac{r}{r-1}\|M_n\|_{L^r}.
\]
Since conditional expectation is an \(L^r\)-contraction,
\[
\|M_n\|_{L^r}
=
\|\mathbb E[U\mid \mathcal F_n]\|_{L^r}
\le
\|U\|_{L^r}
=
1.
\]
Therefore
\[
\mathbb E[U\,Y_{\mu\wedge n}]
\le
\frac{r}{r-1}\|X_{\mu\wedge n}\|_{L^s}
=
s\,\|X_{\mu\wedge n}\|_{L^s}.
\]
Taking the supremum over all such \(U\) yields
\[
\|Y_{\mu\wedge n}\|_{L^s}\le s\,\|X_{\mu\wedge n}\|_{L^s}.
\]
This completes the proof.
\end{proof}
\section{Discussion of General \texorpdfstring{$\beta_2$}{beta2} Choices}
\label{sec:general_beta_discussion}

We explain why the self-normalization mechanism used in the main proof is not
intrinsically tied to the particular calibration \(\beta_2=1-1/T\), and why we
nevertheless state the main theorem in this calibrated form. As in the
expository discussion in \Cref{subsec:self_normalization}, we focus on an
RMSProp-style reduction \((\beta_1=0)\) in order to isolate the role of the
second-moment accumulator. The main point is that the pathwise logarithmic
self-normalization created by \(v_t\) remains valid for every fixed
\(\beta_2\in(0,1)\). What changes is the scale of the deterministic term in the
pathwise bound. If \(\beta_2\) and the base stepsize \(\gamma\) are both fixed,
then the descent residual leaves a non-vanishing neighborhood term. The
calibration \(\beta_2=1-1/T\) and \(\gamma=\eta/\sqrt T\) makes this term
vanish and yields the clean rate in the main theorem.

\subsection{Pathwise self-normalization for arbitrary \texorpdfstring{$\beta_2$}{beta2} and \texorpdfstring{$\gamma$}{beta2}}

Consider the RMSProp-style update
\begin{align}\label{eq:general_beta_rmsprop_update}
\x_{t+1}
=
\x_t-\gamma_t\odot g_t,
\qquad
\gamma_t
:=
\gamma(\sqrt{v_t}+\epsilon)^{-1},
\qquad
v_t
=
\beta_2 v_{t-1}
+
(1-\beta_2)(g_t\odot g_t),
\end{align}
with \(v_0=v\mathbf 1\), \(v>0\), \(\gamma>0\), and
\(\beta_2\in(0,1)\). All operations are componentwise. The following estimate
is deterministic and does not use unbiasedness, variance bounds, or tail
assumptions.

\begin{lem}[Pathwise self-normalization for fixed \(\beta_2\)]
\label{lem:general_beta_pathwise_log}
For every deterministic sequence \(g_1,\ldots,g_T\in\mathbb R^d\), the update
\cref{eq:general_beta_rmsprop_update} satisfies
\begin{align}
\sum_{t=1}^{T}
\|\gamma_t\odot g_t\|^2
&\le
\frac{\gamma^2}{1-\beta_2}
\sum_{i=1}^{d}
\log\!\left(
1+
\frac{1-\beta_2}{v}
\sum_{t=1}^{T}
\beta_2^{-t}g_{t,i}^2
\right)
\notag\\
&\le
\frac{\gamma^2 dT\log(1/\beta_2)}{1-\beta_2}
+
\frac{\gamma^2 d}{1-\beta_2}
\log\!\left(
1+
\frac{1-\beta_2}{vd}
\sum_{t=1}^{T}\|g_t\|^2
\right).
\label{eq:general_beta_pathwise_log}
\end{align}
\end{lem}

\begin{proof}
Since
\[
(\sqrt{v_{t,i}}+\epsilon)^2\ge v_{t,i},
\]
it suffices to control
\[
\sum_{t=1}^{T}\sum_{i=1}^{d}
\gamma^2\frac{g_{t,i}^2}{v_{t,i}}.
\]
Fix a coordinate \(i\) and define
\[
C_{t,i}
:=
\beta_2^{-t}v_{t,i}
=
v
+
(1-\beta_2)
\sum_{s=1}^{t}
\beta_2^{-s}g_{s,i}^2 .
\]
Then \(C_{t,i}\) is nondecreasing and
\[
C_{t,i}-C_{t-1,i}
=
(1-\beta_2)\beta_2^{-t}g_{t,i}^2 .
\]
Therefore,
\[
\frac{g_{t,i}^2}{v_{t,i}}
=
\frac{\beta_2^{-t}g_{t,i}^2}{C_{t,i}}
=
\frac{C_{t,i}-C_{t-1,i}}{(1-\beta_2)C_{t,i}}
\le
\frac{1}{1-\beta_2}
\log\frac{C_{t,i}}{C_{t-1,i}},
\]
where the last inequality uses \(1-u^{-1}\le \log u\) for \(u\ge1\). Summing
over \(t\) and then over \(i\) gives the first inequality in
\cref{eq:general_beta_pathwise_log}. For the second inequality, we use
\[
1+
\frac{1-\beta_2}{v}
\sum_{t=1}^{T}
\beta_2^{-t}g_{t,i}^2
\le
\beta_2^{-T}
\left(
1+
\frac{1-\beta_2}{v}
\sum_{t=1}^{T}g_{t,i}^2
\right).
\]
Taking logarithms and summing over \(i\), we obtain
\[
\sum_{i=1}^{d}
\log\!\left(
1+
\frac{1-\beta_2}{v}
\sum_{t=1}^{T}
\beta_2^{-t}g_{t,i}^2
\right)
\le
dT\log(1/\beta_2)
+
\sum_{i=1}^{d}
\log\!\left(
1+
\frac{1-\beta_2}{v}
\sum_{t=1}^{T}g_{t,i}^2
\right).
\]
Finally, Jensen's inequality gives
\[
\sum_{i=1}^{d}
\log\!\left(
1+
\frac{1-\beta_2}{v}
\sum_{t=1}^{T}g_{t,i}^2
\right)
\le
d\log\!\left(
1+
\frac{1-\beta_2}{vd}
\sum_{t=1}^{T}\|g_t\|^2
\right).
\]
This proves the claim.
\end{proof}

The endpoint \(\beta_2=0\) can be handled separately. In that case,
\(v_t=g_t\odot g_t\), and hence
\[
\|\gamma_t\odot g_t\|^2
=
\sum_{i=1}^{d}
\gamma^2
\frac{g_{t,i}^2}{(|g_{t,i}|+\epsilon)^2}
\le
\gamma^2d.
\]
Therefore,
\[
\sum_{t=1}^{T}\|\gamma_t\odot g_t\|^2
\le
\gamma^2dT.
\]
Thus the same qualitative conclusion holds: with a fixed base stepsize, the
smoothness residual creates a fixed-size term after averaging.

\subsection{Proof sketch for the descent consequence}

We now explain, at a proof-sketch level, how the preceding pathwise estimate
enters the descent mechanism. By \(L\)-smoothness,
\begin{align}
f(\x_{t+1})
&\le
f(\x_t)
-
\langle\nabla f(\x_t),\gamma_t\odot g_t\rangle
+
\frac{L}{2}\|\gamma_t\odot g_t\|^2 .
\label{eq:general_beta_rmsprop_descent}
\end{align}
The second-order term is exactly the adaptive stochastic-gradient energy
controlled in Lemma~\ref{lem:general_beta_pathwise_log}. Hence the same
logarithmic self-normalization mechanism used in the calibrated proof remains
available for every fixed \(\beta_2\in(0,1)\).

The only additional issue is the first-order coupling: \(\gamma_t\) depends on
the same stochastic gradient \(g_t\) that appears in
\(\langle\nabla f(\x_t),\gamma_t\odot g_t\rangle\). To isolate this dependence,
define the predictable proxy
\[
\bar v_t:=\beta_2 v_{t-1},
\qquad
\bar\gamma_t:=\gamma(\sqrt{\bar v_t}+\epsilon)^{-1}.
\]
Then \(\bar\gamma_t\) is \(\mathscr F_{t-1}\)-measurable and
\[
v_t=\bar v_t+(1-\beta_2)(g_t\odot g_t),
\qquad
0<\gamma_t\le \bar\gamma_t
\]
componentwise.

Let
\[
A_t
:=
\sum_{i=1}^{d}
\gamma_{t,i}(\nabla f(\x_t))_i g_{t,i}.
\]
Adding and subtracting its conditional expectation gives
\[
-A_t
=
-\mathbb E[A_t\mid\mathscr F_{t-1}]
+
D_t,
\qquad
D_t
:=
\mathbb E[A_t\mid\mathscr F_{t-1}]-A_t,
\]
where \(\{(D_t,\mathscr F_t)\}_{t\ge1}\) is a martingale difference sequence.
Using
\[
\gamma_{t,i}
=
\bar\gamma_{t,i}
-
(\bar\gamma_{t,i}-\gamma_{t,i})
\]
and the conditional unbiasedness of \(g_t\), we obtain
\begin{align}
-\langle\nabla f(\x_t),\gamma_t\odot g_t\rangle
&=
-
\sum_{i=1}^{d}
\bar\gamma_{t,i}(\nabla f(\x_t))_i^2
+
D_t
+
H_t,
\label{eq:general_beta_first_order_split_sketch}
\end{align}
where
\[
H_t
:=
\mathbb E\!\left[
\sum_{i=1}^{d}
(\bar\gamma_{t,i}-\gamma_{t,i})
(\nabla f(\x_t))_i g_{t,i}
\,\middle|\,
\mathscr F_{t-1}
\right].
\]
The term \(H_t\) is the predictable coupling residual. Since
\(\bar\gamma_{t,i}-\gamma_{t,i}\ge0\), Young's inequality and conditional
Cauchy--Schwarz yield the schematic bound
\[
H_t
\le
\frac14
\sum_{i=1}^{d}
\bar\gamma_{t,i}(\nabla f(\x_t))_i^2
+
R_t^{\mathrm{coup}},
\]
where
\[
R_t^{\mathrm{coup}}
:=
\sum_{i=1}^{d}
\mathbb E[
(\bar\gamma_{t,i}-\gamma_{t,i})
\mid
\mathscr F_{t-1}]
\,
\mathbb E[
g_{t,i}^2
\mid
\mathscr F_{t-1}].
\]
Substituting this into \cref{eq:general_beta_rmsprop_descent}, we arrive at the
one-step template
\begin{align}
f(\x_{t+1})
&\le
f(\x_t)
-
\frac34
\sum_{i=1}^{d}
\bar\gamma_{t,i}(\nabla f(\x_t))_i^2
+
D_t
+
R_t^{\mathrm{coup}}
+
\frac{L}{2}\|\gamma_t\odot g_t\|^2 .
\label{eq:general_beta_one_step_template}
\end{align}
This is the fixed-\(\beta_2\) analogue of the descent template used in the main
proof. The martingale term \(D_t\) can be treated by the same martingale
arguments as before; the smoothness residual is controlled by
Lemma~\ref{lem:general_beta_pathwise_log}; and the coupling residual
\(R_t^{\mathrm{coup}}\) contributes additional \(\beta_2\)-dependent constants
to the same stopped-process bookkeeping.

To make the stopped argument explicit, for \(G>1\) define
\[
\bar f(\x):=f(\x)-f^\star+1,
\qquad
\tau_G:=\inf\{t\ge1:\bar f(\x_t)>G\}.
\]
On the event \(\{\tau_G>T\}\), the trajectory remains inside the localized
region \(\bar f(\x_t)\le G\) for \(1\le t\le T\). Summing
\cref{eq:general_beta_one_step_template} up to \(\tau_G\wedge T\), and then
following the same stopped-moment argument as in the calibrated Adam proof, one
first bounds the \(p/6\)-th moments of the stopped martingale and residual
terms. Specifically, BDG reduces the martingale accumulation to \(p/12\)-th
moments of its quadratic-variation terms, while
Lemma~\ref{lem:general_beta_pathwise_log} controls the adaptive
stochastic-gradient energy. The coupling residual \(R_t^{\mathrm{coup}}\) is
handled in the same stopped region using the Young--Cauchy--Schwarz bound above,
at the cost of additional \(\beta_2\)-dependent constants. A Taylor expansion of
the same stretched-exponential transform then yields the corresponding
high-probability bound
\begin{align}
\sum_{t=1}^{T}
\sum_{i=1}^{d}
\bar\gamma_{t,i}
(\nabla f(\x_t))_i^2
&\le
C_{\beta_2,\gamma}
\polylog(\delta^{-1})
+
C_{\beta_2}
\frac{\gamma^2 dT\log(1/\beta_2)}{1-\beta_2}
\label{eq:general_beta_precond_schematic}
\end{align}
with probability at least \(1-\delta\). Here \(C_{\beta_2,\gamma}\) and
\(C_{\beta_2}\) hide fixed constants depending on the algorithmic parameters
\((\beta_2,\gamma,v,\epsilon)\) and problem parameters, but not on \(T\) or
\(\delta\). The first term is the stopped concentration contribution, analogous
to the preconditioned-energy bound in the main theorem. The second term is the
deterministic \(T\)-linear contribution already visible in
Lemma~\ref{lem:general_beta_pathwise_log}.

After normalizing by the first-order scale \(\gamma T\), the natural
RMSProp-geometry stationarity measure satisfies
\begin{align}
\frac{1}{T}
\sum_{t=1}^{T}
\sum_{i=1}^{d}
\frac{(\nabla f(\x_t))_i^2}{\sqrt{\bar v_{t,i}}+\epsilon}
&\le
\frac{C_{\beta_2,\gamma}}{\gamma T}
\polylog(\delta^{-1})
+
C_{\beta_2}
\gamma
\frac{\log(1/\beta_2)}{1-\beta_2}.
\label{eq:general_beta_preconditioned_fixed_bias}
\end{align}
Thus, for fixed \(\beta_2\in(0,1)\) and fixed base stepsize \(\gamma>0\), the
same self-normalization mechanism gives convergence only to a fixed
preconditioned-gradient neighborhood of order
\[
\gamma\frac{\log(1/\beta_2)}{1-\beta_2}.
\]

Finally, to connect this preconditioned statement to the usual unweighted
stationarity criterion, one needs a de-preconditioning step. Let \(G_\delta\) be
the localization level used in the stopped argument. On the localized event,
\[
\sup_{1\le t\le T}\bar f(\x_t)\le G_\delta.
\]
Using \(L\)-smoothness and bounded variance, a crude Markov--union-bound
argument gives the denominator scale
\[
R_{\delta,T}
:=
\left(
\frac{2dT(v+4LG_\delta+2C)}{\delta}
\right)^{1/2}
=
\mathcal O\!\left(
\sqrt{
\frac{dT}{\delta}
\bigl(v+LG_\delta+C\bigr)
}
\right)
\]
such that, with high probability on the localized event,
\[
\max_{1\le t\le T,\,i\in[d]}\sqrt{\bar v_{t,i}}\le R_{\delta,T}.
\]
Consequently,
\[
\bar\gamma_{t,i}
=
\frac{\gamma}{\sqrt{\bar v_{t,i}}+\epsilon}
\ge
\frac{\gamma}{R_{\delta,T}+\epsilon}.
\]
Combining this with \cref{eq:general_beta_precond_schematic} yields the
schematic unweighted bound
\begin{align}
\frac{1}{T}
\sum_{t=1}^{T}
\|\nabla f(\x_t)\|^2
&\le
(R_{\delta,T}+\epsilon)
\left[
\frac{C_{\beta_2,\gamma}}{\gamma T}
\polylog(\delta^{-1})
+
C_{\beta_2}
\gamma
\frac{\log(1/\beta_2)}{1-\beta_2}
\right].
\label{eq:general_beta_unweighted_fixed_bias}
\end{align}
A sharper denominator estimate would replace \(R_{\delta,T}\) by the
corresponding smaller quantity. The important point is that fixed
\(\beta_2\) and fixed \(\gamma\) leave a fixed preconditioned neighborhood, and
the unweighted stationarity bound additionally inherits the cost of controlling
the adaptive denominator.

Let us finally compare this with the calibrated choice in the main theorem. If
\[
\beta_2=1-\frac1T,
\qquad
\gamma=\frac{\eta}{\sqrt T},
\]
then
\[
\frac{\log(1/\beta_2)}{1-\beta_2}
=
\frac{\log(1/(1-1/T))}{1/T}
=
O(1).
\]
Therefore the fixed-neighborhood term in
\cref{eq:general_beta_preconditioned_fixed_bias} becomes \(O(\eta/\sqrt T)\) at
the preconditioned level. Moreover, the calibrated proof uses the generalized
monotonicity of \(v_t\) to control the maximum denominator through the terminal
accumulator, avoiding the extra union-over-time factor in the crude
\(R_{\delta,T}\) above. This is why the main theorem is stated under
\(\beta_2=1-1/T\): the same proof mechanism extends to other \(\beta_2\) choices
by tracking parameter dependence, but the calibrated choice removes the fixed
bias and gives a clean vanishing stationarity rate.
\section{Proof of Lower Bounds for SGD}\label{Teacher_1}
\subsection{Formal Statement and Proof of Theorem~\ref{prop:sgd_lower_informal}}\label{Teacher_1'}

\paragraph{The hard noise family.}
Throughout this subsection, we use the one-dimensional quadratic objective
\begin{align}\label{ffd}
f(x)=\frac12 x^2,
\qquad
f'(x)=x.
\end{align}
For any \(A\ge1\), let \(\xi^{(A)}\) follow the symmetric three-point distribution
\[
\xi^{(A)}
=
\begin{cases}
A, & \text{with probability } \dfrac{1}{2A^2},\\[6pt]
-A, & \text{with probability } \dfrac{1}{2A^2},\\[6pt]
0, & \text{with probability } 1-\dfrac{1}{A^2}.
\end{cases}
\]
Equivalently,
\[
\mathbb P(\xi^{(A)}\ne0)=\frac1{A^2},
\qquad
\mathbb P(\xi^{(A)}=\sigma A)=\frac1{2A^2},
\quad \sigma\in\{-1,1\}.
\]
Moreover,
\[
\mathbb E[\xi^{(A)}]=0,
\qquad
\mathbb E[(\xi^{(A)})^2]=1.
\]
Thus the stochastic oracle
\[
g(x;\xi)=x+\xi
\]
is unbiased and satisfies the bounded-variance condition with variance constant \(1\). 
The SGD recursion considered below is
\[
x_{t+1}=x_t-\gamma(x_t+\xi_t),
\qquad t=1,\ldots,T-1,
\]
where we apply $f'(x_t)=x_t$ and \(\{\xi_t\}_{t=1}^{T}\) are i.i.d. copies of the chosen three-point noise.

\begin{thm}[Formal Statement of Theorem~\ref{prop:sgd_lower_informal}]
\label{prop:grad-lower-bound}
For the function \(f\) defined in \cref{ffd}, fix any deterministic initial
point \(x_1=x_{\mathrm{init}}\in\mathbb R\), any deterministic constant
stepsize \(\gamma>0\), and any horizon \(T\ge 10\).
If \(0<\gamma<1\), define the one-shock response factor
\[
    R_{\gamma,T}
    :=
    \gamma^2
    \sum_{r=0}^{T-\lfloor T/2\rfloor-1}
    (1-\gamma)^{2r}.
\]
Define
\[
\delta_{\gamma,x_{\mathrm{init}},T}
:=
\begin{cases}
\min\left\{
\frac1{64},
\exp\!\left(-\frac{1}{32R_{\gamma,T}\sqrt T}\right),
\exp\!\left(-\frac{1}{\sqrt{32\gamma T R_{\gamma,T}}}\right)
\right\},
& 0<\gamma<1,\\[18pt]
\frac1{64},
& \gamma\ge 1.
\end{cases}
\]
Then, for every
\[
    0<\delta<\delta_{\gamma,x_{\mathrm{init}},T},
\]
one can choose the i.i.d. noise sequence
\[
\bigl(\xi_t^{(A_{\delta,\gamma})}\bigr)_{t=1}^{T}
\]
from the three-point family above, with the noise law allowed to depend on \((\delta,\gamma,T)\), such that the corresponding SGD iterates satisfy
\[
\mathbb{P}\!\left(
\frac{1}{T}\sum_{t=1}^{T}|f'(x_t)|^2
\ge
\frac{1}{512\,\delta\sqrt{T}\,\log(1/\delta)}
\right)
>
\delta.
\]
\end{thm}

\begin{proof}
We split the proof into two cases.

\paragraph{Case 1: \(0<\gamma<1\).}
Set
\[
    m:=\left\lfloor\frac{T}{2}\right\rfloor
\]
and choose
\[
    A_{\delta,\gamma}^2:=\frac{T}{16\delta}.
\]
Then
\[
    p:=\mathbb P\bigl(\xi_t^{(A_{\delta,\gamma})}\ne0\bigr)
    =
    \frac{1}{A_{\delta,\gamma}^2}
    =
    \frac{16\delta}{T}.
\]

For each \(j\in\{1,\ldots,m\}\), suppose that all noises before time \(j\)
are zero. Then
\[
    x_j=(1-\gamma)^{j-1}x_{\mathrm{init}}.
\]
Let
\[
    a_j:=(1-\gamma)x_j.
\]
Choose a deterministic sign \(\sigma_j\in\{-1,1\}\) such that
\[
    |a_j-\gamma\sigma_j A_{\delta,\gamma}|
    \ge
    \gamma A_{\delta,\gamma}.
\]
Such a sign always exists, since for any \(a\in\mathbb R\) and \(b\ge0\),
\[
    \max\{|a-b|,|a+b|\}\ge b.
\]

Define
\[
E_j
:=
\left\{
\xi_j^{(A_{\delta,\gamma})}=\sigma_j A_{\delta,\gamma},
\quad
\xi_t^{(A_{\delta,\gamma})}=0
\text{ for all }
t\in\{1,\ldots,T-1\}\setminus\{j\}
\right\}.
\]
The events \(E_1,\ldots,E_m\) are disjoint. Moreover, since
\[
\mathbb P(\xi_j^{(A_{\delta,\gamma})}=\sigma_j A_{\delta,\gamma})
=
\frac{1}{2A_{\delta,\gamma}^2},
\]
we have
\[
\mathbb P(E_j)
=
\frac{1}{2A_{\delta,\gamma}^2}
(1-p)^{T-2}
=
\frac{8\delta}{T}
\left(1-\frac{16\delta}{T}\right)^{T-2}.
\]
Let
\[
    E:=\bigcup_{j=1}^{m}E_j.
\]
Since \(T\ge10\), one has \(m/T\ge5/11\). Since
\(\delta<1/64\), Bernoulli's inequality gives
\[
\left(1-\frac{16\delta}{T}\right)^{T-2}
\ge
1-\frac{16\delta(T-2)}{T}
>
1-16\delta
\ge
\frac34.
\]
Therefore
\[
\mathbb P(E)
=
\sum_{j=1}^{m}\mathbb P(E_j)
\ge
\frac{5T}{11}\cdot
\frac{8\delta}{T}\cdot
\frac34
=
\frac{30}{11}\delta
>
\delta.
\]

We now lower bound the trajectory energy on \(E\). Fix \(j\in\{1,\ldots,m\}\)
and work on \(E_j\). By construction,
\[
    x_{j+1}
    =
    (1-\gamma)x_j-\gamma\xi_j^{(A_{\delta,\gamma})}
    =
    a_j-\gamma\sigma_j A_{\delta,\gamma},
\]
so
\[
    |x_{j+1}|\ge \gamma A_{\delta,\gamma}.
\]
For all \(t=j+1,\ldots,T\), no further noise occurs on \(E_j\), and hence
\[
    x_t
    =
    (1-\gamma)^{t-j-1}x_{j+1}.
\]
Therefore
\[
\begin{aligned}
    \sum_{t=1}^{T}x_t^2
    &\ge
    \sum_{t=j+1}^{T}x_t^2                                 
    =x_{j+1}^2
    \sum_{r=0}^{T-j-1}(1-\gamma)^{2r}                                      \ge
    \gamma^2A_{\delta,\gamma}^2
    \sum_{r=0}^{T-j-1}(1-\gamma)^{2r}.
\end{aligned}
\]
Since \(j\le m=\lfloor T/2\rfloor\),
\[
    T-j-1\ge T-\lfloor T/2\rfloor-1.
\]
Thus, by the definition of \(R_{\gamma,T}\),
\[
    \sum_{t=1}^{T}x_t^2
    \ge
    A_{\delta,\gamma}^2R_{\gamma,T}
    =
    \frac{T R_{\gamma,T}}{16\delta}.
\]
Equivalently,
\[
    \frac1T\sum_{t=1}^{T}x_t^2
    \ge
    \frac{R_{\gamma,T}}{16\delta}.
\]

By the definition of \(\delta_{\gamma,x_{\mathrm{init}},T}\), for
\(0<\delta<\delta_{\gamma,x_{\mathrm{init}},T}\),
\[
    \log(1/\delta)
    >
    \frac{1}{32R_{\gamma,T}\sqrt T}.
\]
Hence
\begin{equation}\label{eq:sgd-small-delta-comparison}
    \frac{R_{\gamma,T}}{16\delta}
    >
    \frac{1}{512\,\delta\sqrt T\log(1/\delta)}.
\end{equation}
Since \(f'(x_t)=x_t\), the event \(E\) implies
\[
    \frac1T\sum_{t=1}^{T}|f'(x_t)|^2
    \ge
    \frac{1}{512\,\delta\sqrt T\log(1/\delta)}.
\]
Together with \(\mathbb P(E)>\delta\), this proves the claim for
\(0<\gamma<1\).

\paragraph{Case 2: \(\gamma\ge1\).}
Choose
\[
    A_{\delta,\gamma}^2:=\frac{T}{16\delta}.
\]
Then
\[
    p:=\mathbb P\bigl(\xi_t^{(A_{\delta,\gamma})}\ne0\bigr)
    =
    \mathbb P\bigl(|\xi_t^{(A_{\delta,\gamma})}|=A_{\delta,\gamma}\bigr)
    =
    \frac{1}{A_{\delta,\gamma}^2}
    =
    \frac{16\delta}{T}.
\]
Let
\[
    m:=\left\lfloor\frac{T}{2}\right\rfloor.
\]
Define
\begin{align*}
E_{\mathrm{one}}
:=
\Bigl\{
&\text{there exists exactly one index }
j\in\{1,\ldots,m\}
\text{ such that }
\left|\xi_j^{(A_{\delta,\gamma})}\right|
=
A_{\delta,\gamma},
\\
&\text{and }
\xi_t^{(A_{\delta,\gamma})}=0
\quad
\text{for all }
t\in\{1,\ldots,T-1\}\setminus\{j\}
\Bigr\}.
\end{align*}
Since the event is defined through \(|\xi_j|=A_{\delta,\gamma}\), the one-jump probability is \(p=1/A_{\delta,\gamma}^2\).
Hence, we have
\[
\mathbb P(E_{\mathrm{one}})
=
m\frac{16\delta}{T}
\left(1-\frac{16\delta}{T}\right)^{T-2}\ge
\frac{5T}{11}\cdot\frac{16\delta}{T}\cdot\frac34
=
\frac{60}{11}\delta
>
\delta,
\]
where the inequality follows because (i) \(m/T\ge5/11\) due to \(T\ge10\), and (ii) the fact that \(\delta<1/64\) implies
\[
\left(1-\frac{16\delta}{T}\right)^{T-2}
\ge
1-\frac{16\delta(T-2)}{T}
>
1-16\delta
\ge
\frac34.
\]
%Therefore 
%\[
%\mathbb P(E_{\mathrm{one}})\ge\frac{5T}{11}\cdot\frac{16\delta}{T}\cdot\frac34 =\frac{60}{11}\delta > \delta.
%\]

On \(E_{\mathrm{one}}\), let \(j\in\{1,\ldots,m\}\) be the unique jump time.
Then
\[
    x_{j+1}
    =
    (1-\gamma)x_j-\gamma\xi_j^{(A_{\delta,\gamma})}.
\]
For any \(u,v\in\mathbb R\),
\[
    u^2+\bigl((1-\gamma)u-v\bigr)^2
    \ge
    \frac{v^2}{1+(1-\gamma)^2}.
\]
Taking \(u=x_j\) and \(v=\gamma\xi_j^{(A_{\delta,\gamma})}\), we have
\[
    x_j^2+x_{j+1}^2
    \ge
    \frac{\gamma^2(\xi_j^{(A_{\delta,\gamma})})^2}
    {1+(1-\gamma)^2}.
\]
For \(\gamma\ge1\),
\[
    \frac{\gamma^2}{1+(1-\gamma)^2}\ge\frac12.
\]
Thus, on \(E_{\mathrm{one}}\),
\[
    \sum_{t=1}^{T}x_t^2
    \ge
    x_j^2+x_{j+1}^2
    \ge
    \frac12A_{\delta,\gamma}^2
    =
    \frac{T}{32\delta}.
\]
Consequently,
\[
    \frac1T\sum_{t=1}^{T}|f'(x_t)|^2
    =
    \frac1T\sum_{t=1}^{T}x_t^2
    \ge
    \frac{1}{32\delta} > \frac{1}{512\,\delta\sqrt T\log(1/\delta)}
\]
where the last inequality follows because \(T\ge10\) and \(\delta<1/64\). %we have
%\[
%    \frac{1}{32\delta}> \frac{1}{512\,\delta\sqrt T\log(1/\delta)}.
%\]

Together with \(\mathbb P(E_{\mathrm{one}})>\delta\), this completes the proof.
\end{proof}

\subsection{Formal Statement and Proof of Proposition~\ref{prop:sgd_peak_lower_informal}}\label{Teacher_2}

\begin{proo}[Formal Statement of Proposition~\ref{prop:sgd_peak_lower_informal}]
\label{prop:hp-lower-bound-same-cuts}
For the function \(f\) defined in \cref{ffd}, fix any deterministic initial
point \(x_1=x_{\mathrm{init}}\in\mathbb R\), any deterministic constant
stepsize \(\gamma>0\), and any horizon \(T\ge10\).
Let \(\delta_{\gamma,x_{\mathrm{init}},T}\) be as in
Theorem~\ref{prop:grad-lower-bound}. Then, for every
\[
    0<\delta<\delta_{\gamma,x_{\mathrm{init}},T},
\]
one can choose an i.i.d. noise sequence from the same three-point family as
in Theorem~\ref{prop:grad-lower-bound}, with the same \((\delta,\gamma,T)\)-dependent choice of the noise law, such that the corresponding SGD iterates satisfy
\[
\mathbb P\!\left(
\sum_{t=1}^{T}\gamma |\nabla f(x_t)|^2
\ge
\frac{1}{512\,\delta\log^2(1/\delta)}
\right)
>
\delta.
\]
\end{proo}
\begin{proof}
We use the same hard instances and the same one-shock events as in the proof of
Theorem~\ref{prop:grad-lower-bound}. Note that $\nabla f(x)=x$.

\paragraph{Case 1: \(0<\gamma<1\).}
Let \(E\) be the event constructed in the proof of
Theorem~\ref{prop:grad-lower-bound}. We have already shown that
\[
    \mathbb P(E)>\delta.
\]
On \(E\),
\[
    \sum_{t=1}^{T}x_t^2
    \ge
    A_{\delta,\gamma}^2R_{\gamma,T}
    =
    \frac{T R_{\gamma,T}}{16\delta}.
\]
Therefore
\[
    \sum_{t=1}^{T}\gamma x_t^2
    \ge
    \frac{\gamma T R_{\gamma,T}}{16\delta} > \frac{1}{512\,\delta\log^2(1/\delta)}
\]
where the last inequality follows because the definition of \(\delta_{\gamma,x_{\mathrm{init}},T}\) implies that for
\(0<\delta<\delta_{\gamma,x_{\mathrm{init}},T}\),
\[
    \log(1/\delta)
    >
    \frac{1}{\sqrt{32\gamma T R_{\gamma,T}}}, 
\]
which is equivalent to
\[
    \frac{\gamma T R_{\gamma,T}}{16\delta}
    >
    \frac{1}{512\,\delta\log^2(1/\delta)}.
\]
%Thus, on \(E\),
%\[
%    \sum_{t=1}^{T}\gamma x_t^2    \ge   \frac{1}{512\,\delta\log^2(1/\delta)}.
%\]
Since \(\mathbb P(E)>\delta\), the desired bound follows for \(0<\gamma<1\).

\paragraph{Case 2: \(\gamma\ge1\).}
Let \(E_{\mathrm{one}}\) be the event constructed in the proof of
Theorem~\ref{prop:grad-lower-bound}. We have
\[
    \mathbb P(E_{\mathrm{one}})>\delta.
\]
On \(E_{\mathrm{one}}\),
\[
    \sum_{t=1}^{T}x_t^2
    \ge
    \frac{T}{32\delta}.
\]
Hence, since \(\gamma\ge1\),
\[
    \sum_{t=1}^{T}\gamma x_t^2
    \ge
    \frac{\gamma T}{32\delta}
    \ge
    \frac{T}{32\delta} > \frac{1}{512\,\delta\log^2(1/\delta)},
\]
where the last inequality follows because \(T\ge10\) and \(\delta<1/64\).
%\[
%    \frac{T}{32\delta}   >   \frac{1}{512\,\delta\log^2(1/\delta)}.
%\]
Thus
\[
\mathbb P\!\left(
\sum_{t=1}^{T}\gamma x_t^2
\ge
\frac{1}{512\,\delta\log^2(1/\delta)}
\right)
>
\delta.
\]
The proof is complete.
\end{proof}
\subsection{Results for SGD with Time-Varying Step Sizes}
\label{sec:sgd_time_varying}

In this subsection, we provide two companion lower bounds for SGD with deterministic time-varying step sizes. These results are not needed for the proof of the main Adam--SGD separation, but they strengthen the message that the SGD lower-bound mechanism is not an artifact of using a constant step size; it holds even under deterministic time-varying step sizes. 

Consider the one-dimensional quadratic objective
\begin{align}\label{fffff}
f(x)=\frac{1}{2}x^2,
\qquad
\nabla f(x)=x.
\end{align}
Let \(\{\eta_t\}_{t=1}^{T}\) be an arbitrary deterministic nonnegative stepsize
sequence. The time-varying SGD recursion is
\begin{equation}\label{eq:tvs_sgd_update}
x_{t+1}
=
x_t-\eta_t(x_t+\xi_t),
\qquad
x_1=0,
\end{equation}
where we apply $f'(x_t)=x_t$.

For an internal parameter \(0<\bar\delta<1\), define
\[
A_{\bar\delta}^2:=\frac{4T}{\bar\delta},
\]
and let each \(\xi_t\) follow the symmetric three-point distribution
\begin{equation}\label{eq:tvs_sgd_noise}
\xi_t
=
\begin{cases}
A_{\bar\delta}, & \text{with probability } \dfrac{1}{2A_{\bar\delta}^2},\\[6pt]
-A_{\bar\delta}, & \text{with probability } \dfrac{1}{2A_{\bar\delta}^2},\\[6pt]
0, & \text{with probability } 1-\dfrac{1}{A_{\bar\delta}^2}.
\end{cases}
\end{equation}
Then
\[
\mathbb E[\xi_t]=0,
\qquad
\mathbb E[\xi_t^2]=1.
\]
Hence the stochastic oracle \(g(x;\xi)=x+\xi\) is unbiased and satisfies the
bounded-variance condition with constant \(1\).

For \(s<t\), define
\[
\Pi_{s,t}:=\prod_{r=s+1}^{t-1}(1-\eta_r),
\]
with the convention \(\Pi_{s,s+1}=1\). Let
\[
I_T:=\left\{1,2,\ldots,\left\lfloor\frac{T}{2}\right\rfloor\right\}.
\]
Define the deterministic response quantities
\begin{align}
\mathcal R_T(\eta)
&:=
\min_{s\in I_T}
\eta_s^2
\sum_{t=s+1}^{T}
\Pi_{s,t}^2,
\label{eq:R_T_eta_def}
\\
\mathcal Q_T(\eta)
&:=
\min_{s\in I_T}
\eta_s^2
\sum_{t=s+1}^{T}
\eta_t\Pi_{s,t}^2.
\label{eq:Q_T_eta_def}
\end{align}

\begin{proo}\label{prop:tvs_sgd_unweighted_lower}
For the function \(f\) defined in \cref{fffff}, for any deterministic
nonnegative stepsize sequence \(\{\eta_t\}_{t=1}^{T}\) and any internal
parameter \(0<\bar\delta<1\), the time-varying SGD iterates generated by
\cref{eq:tvs_sgd_update} with the oracle \cref{eq:tvs_sgd_noise} satisfy
\[
\mathbb P\!\left(
\frac{1}{T}
\sum_{t=1}^{T}
|\nabla f(x_t)|^2
\ge
\frac{4}{\bar\delta}\mathcal R_T(\eta)
\right)
>
\frac{\bar\delta}{16}.
\]
\end{proo}

\begin{proof}
Let
\[
p_{\bar\delta}:=\mathbb P(\xi_t\ne0)
=
\frac{1}{A_{\bar\delta}^2}
=
\frac{\bar\delta}{4T}.
\]
\[
\mathcal E
:=
\left\{
\text{there exists exactly one }s\in I_T\text{ such that }\xi_s\ne0,
\quad
\xi_t=0\text{ for all }t\ne s
\right\}.
\]
Since \(|I_T|/T\ge5/11\) for \(T>10\), and since \(0<\bar\delta<1\),
\[
(1-p_{\bar\delta})^{T-1}
\ge
1-(T-1)p_{\bar\delta}
>
1-\frac{\bar\delta}{4}
>
\frac34.
\]
Therefore
\[
\mathbb P(\mathcal E)
=
|I_T|p_{\bar\delta}(1-p_{\bar\delta})^{T-1}
>
\frac{5T}{11}\cdot
\frac{\bar\delta}{4T}\cdot
\frac34
=
\frac{15}{176}\bar\delta
>
\frac{\bar\delta}{16}.
\]

On \(\mathcal E\), let \(s\in I_T\) be the unique nonzero-noise time. Since
\(x_1=0\) and all previous noises are zero, \(x_s=0\). Hence
\[
x_{s+1}=-\eta_s\xi_s,
\qquad
x_{s+1}^2=\eta_s^2A_{\bar\delta}^2.
\]
For \(t=s+1,\ldots,T\), all subsequent noises vanish on \(\mathcal E\), so
\[
x_t
=
-\eta_s\xi_s
\prod_{r=s+1}^{t-1}(1-\eta_r).
\]
Thus
\[
x_t^2
=
A_{\bar\delta}^2\eta_s^2\Pi_{s,t}^2,
\qquad
t=s+1,\ldots,T.
\]
Consequently, on \(\mathcal E\),
\[
\frac{1}{T}
\sum_{t=1}^{T}
|\nabla f(x_t)|^2
=
\frac1T\sum_{t=1}^{T}x_t^2
\ge
\frac{A_{\bar\delta}^2}{T}
\eta_s^2
\sum_{t=s+1}^{T}
\Pi_{s,t}^2
\ge
\frac{A_{\bar\delta}^2}{T}\mathcal R_T(\eta)
=
\frac{4}{\bar\delta}\mathcal R_T(\eta).
\]
Combining this event together with
\(\mathbb P(\mathcal E)>\bar\delta/16\) proves the claim.
\end{proof}

\begin{thm}\label{thm:tvs_sgd_weighted_lower}
For the function \(f\) defined in \cref{fffff}, for any deterministic
nonnegative stepsize sequence \(\{\eta_t\}_{t=1}^{T}\) and any internal
parameter \(0<\bar\delta<1\), the time-varying SGD iterates generated by
\cref{eq:tvs_sgd_update} with the oracle \cref{eq:tvs_sgd_noise} satisfy
\[
\mathbb P\!\left(
\sum_{t=1}^{T}
\eta_t|\nabla f(x_t)|^2
\ge
\frac{4T}{\bar\delta}\mathcal Q_T(\eta)
\right)
>
\frac{\bar\delta}{16}.
\]
\end{thm}

\begin{proof}
Use the same event \(\mathcal E\) as in the proof of
Proposition~\ref{prop:tvs_sgd_unweighted_lower}, where we have shown that
\[
    \mathbb P(\mathcal E)>\frac{\bar\delta}{16}.
\]
On \(\mathcal E\), if \(s\in I_T\) is the unique nonzero-noise time, then
for every \(t=s+1,\ldots,T\),
\[
x_t^2
=
A_{\bar\delta}^2\eta_s^2\Pi_{s,t}^2.
\]
Therefore, on \(\mathcal E\),
\[
\sum_{t=1}^{T}
\eta_t|\nabla f(x_t)|^2
=
\sum_{t=1}^{T}\eta_tx_t^2
\ge
A_{\bar\delta}^2
\eta_s^2
\sum_{t=s+1}^{T}
\eta_t\Pi_{s,t}^2
\ge
A_{\bar\delta}^2\mathcal Q_T(\eta)
=
\frac{4T}{\bar\delta}\mathcal Q_T(\eta).
\]
Combining this event together with
\(\mathbb P(\mathcal E)>\bar\delta/16\) proves the theorem.
\end{proof}

\paragraph{Consequences for sufficiently small confidence levels.}
The preceding two lower bounds are stated in terms of the internal parameter
\(\bar\delta\). We now restate them using an external confidence parameter
\(\delta\), so that the final probability lower bound is exactly of order
\(\delta\).

\begin{cor}[Small-confidence unweighted lower bound]
\label{cor:tvs_sgd_unweighted_small_delta}
For the function \(f\) defined in \cref{fffff}, suppose that
\(\mathcal R_T(\eta)>0\). If
\[
0<\delta
<
\frac{1}{16}
\exp\left\{
-\frac{1}{4\sqrt T\,\mathcal R_T(\eta)}
\right\},
\]
then, by choosing the noise level with internal parameter
\[
\bar\delta:=16\delta,
\qquad
A_{\bar\delta}^{2}:=\frac{4T}{\bar\delta},
\]
the time-varying SGD iterates generated by \cref{eq:tvs_sgd_update} and
\cref{eq:tvs_sgd_noise} satisfy
\begin{equation}\label{eq:tvs_sgd_unweighted_small_delta_result}
\mathbb P\!\left(
\frac{1}{T}
\sum_{t=1}^{T}
|\nabla f(x_t)|^2
\ge
\frac{1}{32\,\delta\sqrt T\log(1/\delta)}
\right)
>
\delta.
\end{equation}
\end{cor}

\begin{proof}
Set
\[
    \bar\delta:=16\delta.
\]
Then \(0<\bar\delta<1\), and
\[
    \bar\delta
    <
    \exp\left\{
    -\frac{1}{4\sqrt T\,\mathcal R_T(\eta)}
    \right\}.
\]
Applying Proposition~\ref{prop:tvs_sgd_unweighted_lower} with internal parameter
\(\bar\delta\), we obtain
\[
\mathbb P\!\left(
\frac{1}{T}
\sum_{t=1}^{T}
|\nabla f(x_t)|^2
\ge
\frac{4}{\bar\delta}\mathcal R_T(\eta)
\right)
>
\frac{\bar\delta}{16}
=
\delta.
\]
Moreover,
\[
    \log(1/\bar\delta)
    >
    \frac{1}{4\sqrt T\,\mathcal R_T(\eta)}.
\]
Thus
\[
\frac{4}{\bar\delta}\mathcal R_T(\eta)
>
\frac{1}{\bar\delta\sqrt T\log(1/\bar\delta)}
=
\frac{1}{16\delta\sqrt T\log(1/(16\delta))}.
\]
Since \(0<\delta<1/16\),
\[
    \log(1/(16\delta))\le \log(1/\delta).
\]
Hence
\[
\frac{1}{16\delta\sqrt T\log(1/(16\delta))}
\ge
\frac{1}{16\delta\sqrt T\log(1/\delta)}
\ge
\frac{1}{32\delta\sqrt T\log(1/\delta)}.
\]
This proves \cref{eq:tvs_sgd_unweighted_small_delta_result}.
\end{proof}

\begin{cor}[Small-confidence stepsize-weighted lower bound]
\label{cor:tvs_sgd_weighted_small_delta}
For the function \(f\) defined in \cref{fffff}, suppose that
\(\mathcal Q_T(\eta)>0\). If
\[
0<\delta
<
\frac{1}{16}
\exp\left\{
-\frac{1}{4T\,\mathcal Q_T(\eta)}
\right\},
\]
then, by choosing the noise level with internal parameter
\[
\bar\delta:=16\delta,
\qquad
A_{\bar\delta}^{2}:=\frac{4T}{\bar\delta},
\]
the time-varying SGD iterates generated by \cref{eq:tvs_sgd_update} and
\cref{eq:tvs_sgd_noise} satisfy
\begin{equation}\label{eq:tvs_sgd_weighted_small_delta_result}
\mathbb P\!\left(
\sum_{t=1}^{T}
\eta_t|\nabla f(x_t)|^2
\ge
\frac{1}{32\,\delta\log(1/\delta)}
\right)
>
\delta.
\end{equation}
\end{cor}

\begin{proof}
Set
\[
    \bar\delta:=16\delta.
\]
Then \(0<\bar\delta<1\), and
\[
    \bar\delta
    <
    \exp\left\{
    -\frac{1}{4T\,\mathcal Q_T(\eta)}
    \right\}.
\]
Applying Theorem~\ref{thm:tvs_sgd_weighted_lower} with internal parameter
\(\bar\delta\), we obtain
\[
\mathbb P\!\left(
\sum_{t=1}^{T}
\eta_t|\nabla f(x_t)|^2
\ge
\frac{4T}{\bar\delta}\mathcal Q_T(\eta)
\right)
>
\frac{\bar\delta}{16}
=
\delta.
\]
Moreover,
\[
    \log(1/\bar\delta)
    >
    \frac{1}{4T\,\mathcal Q_T(\eta)}.
\]
Thus
\[
\frac{4T}{\bar\delta}\mathcal Q_T(\eta)
>
\frac{1}{\bar\delta\log(1/\bar\delta)}
=
\frac{1}{16\delta\log(1/(16\delta))}.
\]
Since \(0<\delta<1/16\),
\[
    \log(1/(16\delta))\le\log(1/\delta).
\]
Hence
\[
\frac{1}{16\delta\log(1/(16\delta))}
\ge
\frac{1}{16\delta\log(1/\delta)}
\ge
\frac{1}{32\delta\log(1/\delta)}.
\]
This proves \cref{eq:tvs_sgd_weighted_small_delta_result}.
\end{proof}

\paragraph{Interpretation.}
The quantities \(\mathcal R_T(\eta)\) and \(\mathcal Q_T(\eta)\) encode the
deterministic response of the time-varying SGD trajectory to a single rare
stochastic-gradient shock. Corollaries~\ref{cor:tvs_sgd_unweighted_small_delta}
and \ref{cor:tvs_sgd_weighted_small_delta} show that, after the harmless
rescaling \(\bar\delta=16\delta\), the lower bounds can be stated with final
probability strictly larger than the external confidence parameter \(\delta\).

The response-dependent statements above apply to arbitrary deterministic
nonnegative learning-rate schedules. In particular, whenever the deterministic
schedule satisfies
\[
\mathcal R_T(\eta)\gtrsim T^{-1/2},
\]
the unweighted lower bound takes the form
\[
\frac{1}{T}
\sum_{t=1}^{T}
|\nabla f(x_t)|^2
\ge
\widetilde{\Omega}\!\left(
\frac{1}{\delta\sqrt T}
\right)
\]
with probability strictly larger than \(\delta\), for the corresponding range of
confidence levels. Similarly, whenever
\[
\mathcal Q_T(\eta)\gtrsim T^{-1},
\]
the stepsize-weighted lower bound takes the form
\[
\sum_{t=1}^{T}
\eta_t|\nabla f(x_t)|^2
\ge
\widetilde{\Omega}\!\left(
\frac{1}{\delta}
\right)
\]
with probability strictly larger than \(\delta\).

Thus the message is not tied to the constancy of the learning rate. Rather, the
lower bounds show that any deterministic schedule whose response to a rare shock
has the above natural scale still exhibits the same \(\delta^{-1}\)-type
obstruction. A deterministic learning-rate schedule is fixed before observing the
stochastic gradients, and therefore cannot selectively shrink the update exactly
when a large rare shock occurs. To suppress such shocks through the update scale
without making the entire schedule uniformly negligible, the scaling must depend
on the observed stochastic gradients themselves. This is precisely the mechanism
implemented by Adam's second-moment accumulator: after a large coordinatewise
stochastic gradient is observed, the denominator
\[
(\sqrt{v_t}+\epsilon)^{-1}
\]
immediately normalizes the corresponding coordinatewise update. The absence of
such data-dependent second-moment normalization is the structural limitation of
SGD highlighted by these lower bounds.

\end{document}